\documentclass[letterpaper]{article} 
\usepackage{aaai23}  
\usepackage{times}  
\usepackage{helvet}  
\usepackage{courier}  
\usepackage[hyphens]{url}  
\usepackage{graphicx} 
\usepackage{natbib}  
\usepackage{caption} 
\usepackage{algorithm}
\usepackage{algorithmic}

\usepackage[utf8]{inputenc} 
\usepackage[T1]{fontenc}    
\usepackage{url}            
\usepackage{booktabs}       
\usepackage{nicefrac}       
\usepackage{microtype}      
\usepackage{xcolor}         
\usepackage{multicol}
\usepackage{graphicx}
\usepackage{subcaption}
\usepackage{amsmath}
\usepackage{mathrsfs}
\usepackage{amssymb}
\usepackage{color}
\usepackage{multirow}
\usepackage{paralist}
\usepackage{verbatim}
\usepackage{galois}
\usepackage{boxedminipage}
\usepackage{accents}
\usepackage{stmaryrd}
\usepackage[bottom]{footmisc}
\definecolor{light-gray}{gray}{0.85}
\usepackage{colortbl}
\usepackage{makecell}
\usepackage{hhline}
\usepackage{bbm}
\usepackage{mathtools}
\usepackage{xcolor}         
\usepackage{MnSymbol, graphicx}
\usepackage{xr}

\usepackage{newfloat}
\usepackage{listings}
\DeclareCaptionStyle{ruled}{labelfont=normalfont,labelsep=colon,strut=off} 
\floatstyle{ruled}
\newfloat{listing}{tb}{lst}{}
\floatname{listing}{Listing}
\usepackage[utf8]{inputenc} 
\usepackage[T1]{fontenc}    
\usepackage{hyperref}       
\usepackage{url}            
\usepackage{booktabs}       
\usepackage{amsfonts}       
\usepackage{nicefrac}       
\usepackage{microtype}      
\usepackage{xcolor}         
\usepackage{multicol}
\usepackage{microtype}
\usepackage{graphicx}
\usepackage{caption}
\usepackage{subcaption}
\usepackage{booktabs} 
\usepackage{amsmath}
\usepackage{mathrsfs}
\usepackage{amssymb}
\usepackage{color}
\usepackage{xcolor}
\usepackage{multirow}
\usepackage{paralist}
\usepackage{verbatim}
\usepackage{algorithm}
\usepackage{algorithmic}
\usepackage{galois}
\usepackage{boxedminipage}
\usepackage{accents}
\usepackage{stmaryrd}
\usepackage[bottom]{footmisc}
\definecolor{light-gray}{gray}{0.85}
\usepackage{colortbl}
\usepackage{makecell}
\usepackage{hhline}
\usepackage{bbm}
\usepackage{mathtools}
\usepackage{xcolor}         
\usepackage{bbold}
\usepackage[utf8]{inputenc} 
\usepackage[T1]{fontenc}    
\usepackage{hyperref}       
\usepackage{url}            
\usepackage{booktabs}       
\usepackage{amsfonts}       
\usepackage{nicefrac}       
\usepackage{microtype}      
\usepackage{MnSymbol, graphicx}
\usepackage{xr}
\usepackage{enumitem}

\usepackage{pifont}
\newcommand{\cmark}{\ding{51}}%
\newcommand{\xmark}{\ding{55}}%

\newtheorem{theorem}{Theorem}[section]
\newtheorem{lemma}[theorem]{Lemma}
\newtheorem{assumption}[theorem]{Assumption}

\newcommand{\EE}{\mathbb{E}}

\newcommand{\argmin}{\mathop{\rm argmin}}
\newcommand{\argmax}{\mathop{\rm argmax}}

\newcommand{\ceil}[1]{\left\lceil {#1} \right\rceil}

\newcommand{\ba}{\begin{array}}
\newcommand{\ea}{\end{array}}

\usepackage{tikz}
\newcommand*\circled[1]{\tikz[baseline=(char.base)]{
            \node[shape=circle,draw,inner sep=0.5pt] (char) {#1};}}

\title{Non-stationary Risk-sensitive Reinforcement Learning: Near-optimal Dynamic Regret, Adaptive Detection, and Separation Design}
\author {
    Yuhao Ding\textsuperscript{\rm 1},
    Ming Jin\textsuperscript{\rm 2}, 
    Javad Lavaei\textsuperscript{\rm 1}
}
\affiliations {
    \textsuperscript{\rm 1} University of California, Berkeley,
    \textsuperscript{\rm 2} Virginia Tech\\
    \{ yuhao\_ding, lavaei\}@berkeley.edu, \ jinming@vt.edu 
}

\begin{document}

\maketitle

\begin{abstract}
We study risk-sensitive reinforcement learning (RL) based on an entropic risk measure in episodic non-stationary Markov decision processes (MDPs). Both the reward functions and the state transition kernels are unknown and allowed to vary arbitrarily over time with a budget on their cumulative variations. When this variation budget is known a prior, we propose two restart-based algorithms, namely Restart-RSMB and Restart-RSQ, and establish their dynamic regrets. Based on these results, we further present a meta-algorithm that does not require any prior knowledge of the variation budget and can adaptively detect the non-stationarity on the exponential value functions. A dynamic regret lower bound is then established for  non-stationary risk-sensitive RL to certify the near-optimality of the proposed algorithms. Our results also show that the risk control and the handling of the non-stationarity can be separately designed in the algorithm if the variation budget is known a prior, while the non-stationary detection mechanism in the adaptive algorithm depends on the risk parameter. This work offers the first non-asymptotic theoretical analyses for the non-stationary risk-sensitive RL in the literature. 
\end{abstract}

\section{Introduction}

Risk-sensitive RL considers problems in which the objective takes into account risks that arise during the learning process, in contrast to the typical expected accumulated reward objective. Effective management of the variability of the return in RL is essential in various applications in finance \cite{markowitz1968portfolio}, autonomous driving \cite{garcia2015comprehensive} and  human behavior modeling \cite{niv2012neural}.

While classical risk-sensitive RL assumes that an agent interacts with a time-invariant (stationary) environment, both the reward functions and the transition kernels can be time-varying for many risk-sensitive applications. For example, in finance \cite{markowitz1968portfolio}, the federal reserve adjusts the interest rate or the balance sheet in a non-stationary way and the market participants should adjust their trading policies accordingly.
In the medical treatments \cite{man2014uva}, the patient’s health condition and the sensitivity of the patient’s internal body organs to the medicine vary over time. This non-stationarity should be accounted for to minimize the risk of any potential side effects of the treatment.

Despite the importance and ubiquity of non-stationary risk-sensitive RL problems, the literature lacks provably efficient algorithms and theoretical results. In this work, we study risk-sensitive RL with an entropic risk measure \cite{howard1972risk} under episodic Markov decision processes with unknown and time-varying reward functions and state transition kernels. 

{The non-stationary RL problem with an entropic risk measure has the following technical challenges. (1) Due to the non-stationarity of the model, any estimation error of the expectation operator may be tremendously amplified in the value function when the risk parameter $\beta$ is small. (2)
In addition, the exponential Bellman equation (see Equation \eqref{eq: exp bellman equation}) used in our risk-sensitive analysis associates the instantaneous reward and value function of the next step in a multiplicative way \cite{fei2021exponential}. However, this multiplicative feature of the exponential Bellman equation will also involve the policy evaluation errors due to the non-stationary drifting as multiplicative terms,  which makes it more difficult to gauge the bounds than the risk-neural non-stationary setting in which all policy evaluation errors are in an additive way.
(3) Furthermore, the non-linearity of the objective function  (see Equation \eqref{eq: risk-sensitive objective}) makes it difficult to obtain an unbiased estimation of the value function, which is needed in the design of a non-stationary detection mechanism in risk-neutral non-stationary RL \cite{wei2021non}. 
(4) It is unclear whether the  risk control and the handling of the non-stationarity can be separately designed when achieving the optimal dynamic regret.
To address these difficulties, we develop a novel analysis to carefully quantify the effect of the non-stationarity in risk-sensitive RL.}
Our main theoretical contributions, summarized in Table \ref{table: comparison}, are as follows

\begin{itemize}[leftmargin=*]
 \item 
    When the variation budget is known a prior, we propose two provably efficient restart algorithms, namely Restart-RSMB and Restart-RSQ, and establish their dynamic regrets. {The stationary version of the model-based method Restart-RSMB is also the first model-based risk-sensitive algorithm in the stationary setting in the literature}.
    \item When the variation budget is unknown (parameter-free), we propose a meta-algorithm that adaptively detects the non-stationarity of the exponential value functions. 
    The proposed adaptive algorithms, namely Adaptive-RSMB and Adaptive-RSQ, can achieve the (almost) same dynamic regret as the algorithms requiring the knowledge of the variation budget. 
    \item We establish a lower bound result for non-stationary RL with entropic risk measure that certifies the near-optimality of our upper bounds. 
    \item Our results also show that the risk control and the handling of the non-stationarity can be separately designed if the variation budget is known a prior, while the non-stationary detection mechanism in the adaptive algorithms depends on the risk parameter.
\end{itemize}

\begin{table*}[ht]
\centering 
{
 \begin{tabular}{c|c|c|c|c} 
 \hline
 \textbf{Algorithm} & \textbf{D-Regret} & \textbf{Parameter-free} & \textbf{Model-free} & \textbf{Separation} \\\hline\hline
   Restart-RSMB  & $\widetilde{\mathcal{O}} \left(e^{|\beta| H}  |\mathcal{S}|^{\frac{2}{3}} |\mathcal{A}|^{\frac{1}{3}} H^2 M^{\frac{2}{3}} B^{\frac{1}{3}}\right)$  &  \xmark  & \xmark & \cmark \\\hline
    Restart-RSQ  & $\widetilde{\mathcal{O}} \left(e^{|\beta| H}  |\mathcal{S}|^{\frac{1}{3}} |\mathcal{A}|^{\frac{1}{3}} H^{\frac{9}{4}} M^{\frac{2}{3}} B^{\frac{1}{3}}\right)$  & \xmark   &   \cmark & \cmark\\ \hline\hline
  Adaptive-RSMB  & $\widetilde{\mathcal{O}} \left( e^{|\beta| H} |\mathcal{S}|^{\frac{2}{3}} |\mathcal{A}|^{\frac{1}{3}} H^{2} M^{\frac{2}{3}} B^{\frac{1}{3}} \right) $  & \cmark &  \xmark  & \xmark  \\\hline
   Adaptive-RSQ & $\widetilde{\mathcal{O}} \left( e^{|\beta| H} |\mathcal{S}|^{\frac{1}{3}} |\mathcal{A}|^{\frac{1}{3}} H^{\frac{5}{3}} M^{\frac{2}{3}} B^{\frac{1}{3}} \right) $ &  \cmark &    \cmark  & \xmark   \\\hline \hline
   Lower bound   & $\Omega\left( \frac{e^{\frac{2|\beta| H}{3}}-1}{|\beta|}  |\mathcal{S}|^{\frac{1}{3}} |\mathcal{A}|^{\frac{1}{3}}  M^{\frac{2}{3}} B^{\frac{1}{3}}\right)$ &  & \\\hline 
 \end{tabular}
 \caption{We summarize the dynamic regrets and lower bound obtained in this paper.
Here, $\beta$ is the risk parameter, $H$ is the horizon of each episode, $M$ is the total number of episodes, $B$ is the total variation measurement, and $|\mathcal{S}|$ and $|\mathcal{A}|$ are the cardinalities of the state and action spaces. 
}
\label{table: comparison}
 }
\end{table*}

\subsection{Related work}

\textbf{Non-stationary RL.} 
Non-stationary RL has been mostly studied in the risk-neutral setting. When the variation budget is known a prior, a common strategy for adapting to the non-stationarity is to follow the forgetting principle, such as the restart strategy \cite{mao2020model,zhou2020nonstationary,zhao2020simple,ding2022provably}, exponential decayed weights \cite{touati2020efficient}, or sliding window \cite{cheung2020reinforcement,zhong2021optimistic}. In this work, we focus on the restart method mainly due to its advantage of the simplicity of the the memory efficiency \cite{zhao2020simple} and generalize it to the risk-sensitive RL setting.
However, the prior knowledge of the variation budget is often unavailable in practice. The work
\cite{cheung2020reinforcement} develop a Bandit-over-Reinforcement-Learning framework to relax this assumption, but it leads to the suboptimal regret.
To achieve a nearly-optimal regret without the prior knowledge of the variation budget, \cite{auer2019adaptively} and \cite{chen2019new} maintain a distribution over bandit arms with properly controlled variance for all reward estimators. For RL problems,  the seminar work \cite{wei2021non} proposes a black-box reduction approach that turns a certain RL algorithm with optimal regret in a (near-)stationary environment into another algorithm with optimal dynamic regret in a non-stationary environment. 
However, the above works only consider risk-neutral RL and may not apply to the more general risk-sensitive RL problems.

\textbf{Risk-sensitive RL.} Many risk-sensitive objectives have been investigated in the literature and applied to RL, such as  the entropic risk measure, Markowitz mean-variance model, Value-at-Risk (VaR), and Conditional Value at Risk (CVaR) \cite{moody2001learning,chow2014algorithms,delage2010percentile,la2013actor,di2012policy,tamar2015optimizing,tamar2015policy, howard1972risk}. 
Our work is closely related to the entropic risk measure. Following the seminal paper \cite{howard1972risk}, this line of work includes \cite{bauerle2014more,borkar2001sensitivity,borkar2002risk,borkar2002q,cavazos2000vanishing, coraluppi1999risk, di1999risk, fernandez1997risk,fleming1995risk,hernandez1996risk, osogami2012robustness, fleming1992risk, shen2013risk, fei2020risk,fei2021risk, fei2021exponential}. In particular, when transitions are
unknown and simulators of the environment are unavailable,  the first non-asymptotic
regret guarantees are established under the tabular setting in \cite{fei2020risk} and the function approximation setting in \cite{fei2021risk}. Then, a simple transformation of the risk-sensitive Bellman equations is proposed in \cite{fei2021exponential}, which leads to improved regret upper bounds. However, the above papers all assume that the environment is stationary, and therefore their results may quickly collapse in a non-stationary environment.

\section{Problem formulation}

\subsection{Notations}\label{appe: notation}
For a positive integer $n$, let $[n]:=\{1,2, \ldots, n\}$. 
Given a variable $x$, the notation $a=\mathcal{O}(b(x))$ means that $a \leq C \cdot b(x)$ for some constant $C>0$ that is independent of $x$.
Similarly, $a=\widetilde{\mathcal{O}}(b(x))$ indicates that the previous inequality may also depend on the function $\log(x)$, where $C>0$ is again independent of $x$. In addition, the notation $a=\Omega(b(x))$ means that $a \geq C \cdot b(x)$ for some constant $C>0$ that is independent of $x$.

\subsection{Episodic MDP and risk-sensitive objective}

In this paper, we study risk-sensitive RL in non-stationary
environments via episodic MDPs with adversarial bandit-information reward feedback and unknown adversarial transition dynamics. At each episode $m$, an episodic MDP is defined by the finite state space $\mathcal{S}$, the finite action space $\mathcal{A}$,  a collection of transition probability measure $\{\mathcal{P}_h^m\}_{h=1}^H$  specifying the transition probability $\mathcal{P}_h^m(s^\prime\mid s, a)$ from state $s$ to the next state $s^\prime$ under action $a\in\mathcal{A}$, a collection of reward functions $\{r_h^m\}_{h=1}^H$ where $r_h^m: \mathcal{S}\times \mathcal{A} \rightarrow [0,1]$ , and $H>0$ as the length of episodes. In this paper, we focus on a bandit setting where the agent only observes the values of reward functions, i.e., $r_h^m(s_h^m, a_h^m)$ at the visited state-action pair $(s_h^m, a_h^m)$.  We also assume that reward functions are deterministic to streamline the presentation, while our analysis readily generalizes to the setting where reward functions are random.

For simplicity, we assume the initial state $s_1^m$ to be fixed as $s_1$ in different episodes. 
We use the convention that the episode terminates when a state $s_{H+1}$ at step $H+1$ is reached, at which the agent does not take any further action and receives no reward.

A policy $\pi^m=\left\{\pi_{h}^m\right\}_{h \in[H]}$ of an agent is a sequence of functions $\pi_{h}^m: \mathcal{S} \rightarrow \mathcal{A}$, where $\pi_{h}^m(s)$ is the action that the agent takes in state $s$ at step $h$ at episode $m$. For each $h\in[H]$ and $m\in[M]$, we define the value function $V_h^{\pi, m}: \mathcal{S} \rightarrow \mathbb{R}$ of a policy $\pi$ as the expected value of the cumulative rewards the agent receives under a risk measure of exponential utility by executing $\pi$ starting from an arbitrary state at step $h$. Specifically, we have
\begin{align}\label{eq: risk-sensitive objective}
V_h^{\pi,m}(s):=\frac{1}{\beta} \log \left\{\mathbb{E}_{\pi, \mathcal{P}^m}\left[\exp \left(\beta \sum_{i=h}^{H}  r^m_i\left(s_i,a_i\right)\right) \mid s_{h}=s\right]\right\}
\end{align}
where the expectation $\mathbb{E}_{\pi, \mathcal{P}^{m}}$ is taken over the random state-action sequence $\left\{\left(x_{i}^{m}, a_{i}^{m}\right)\right\}_{i=h}^{H}$, the action $a_{i}^{m}$ follows the policy $\pi_{i}^{m}\left(\cdot \mid x_{i}^{m}\right)$, and the next state $x_{i+1}$ follows the transition dynamics $\mathcal{P}_{i}^{m}\left(\cdot \mid x_{i}^{m}, a_{i}^{m}\right)$.  Here $\beta \neq 0$ is the risk parameter of the exponential utility: $\beta>0$ corresponds to a risk-seeking value function, $\beta<0$ corresponds to a risk-averse value function, and as $\beta \rightarrow 0$ the agent tends to be risk-neutral and we recover the classical value function $V^{\pi,m}_h(s)=$ $\mathbb{E}_{\pi, \mathcal{P}^m}\left[\sum_{t=1}^{H}  r_h^m\left(s_{t}, a_{t}\right) \mid s_0=s\right]$ in standard RL. 

We further define the action-value function $Q_{h}^{\pi, m}: \mathcal{S} \times \mathcal{A} \rightarrow \mathbb{R}$, for each $h\in[H]$ and $m\in[M]$, which gives the expected value of the risk measured by the exponential utility when the agent starts from an arbitrary state-action pair and follows the policy $\pi$ afterwards; that is,
\begin{align*}
&Q_{h}^{\pi, m}\\
\coloneqq &\frac{1}{\beta} \log \left\{\exp \left(\beta \cdot r^m_h(s, a)\right) \mathbb{E}\left[\exp \left(\beta \sum_{i=h}^{H} r_i^m\left(s_t, a_t\right)\right) \right.\right.\\
&\hspace{5cm}\Big\mid s_{h}=s, a_{h}=a\Bigg]\Bigg\}\\
=&r_h^m(s, a) + \frac{1}{\beta} \log \left\{\mathbb{E}\left[\exp \left(\beta \sum_{i=h+1}^{H} r_i^m\left(s_t, a_t\right)\right) \right.\right.\\
&\hspace{5cm}\Big\mid s_{h}=s, a_{h}=a\Bigg]\Bigg\}
\end{align*}
for all $(s, a) \in \mathcal{S} \times \mathcal{A}$.
Under some mild regularity conditions \cite{bauerle2014more}, for each episode $m$, there always exists an optimal policy, denoted as $\pi^{*,m}$, that yields the optimal value $V_{h}^{\pi^{*,m},m}(s):=\sup _{\pi} V_{h}^{\pi,m}(s)$ for all $(h, s) \in[H] \times \mathcal{S}$. For convenience, we denote $V_{h}^{\pi^{*,m},m}(s)$ as $V_{h}^{*,m}(s)$ when it is clear from the context. 


\subsection{Exponential Bellman equation}

For all $(s, a, h, m) \in \mathcal{S} \times \mathcal{A} \times [H] \times [M]$, the Bellman equation associated with $\pi$ is given by
\begin{subequations} \label{eq: bellman equation}
\begin{align}
Q_{h}^{\pi,m}(s, a) &=r_{h}^m(s, a)+\frac{1}{\beta} \log \left\{\mathbb{E}_{s^{\prime} \sim P_{h}^m(\cdot \mid s, a)}\left[e^{\beta \cdot V_{h+1}^{\pi,m}\left(s^{\prime}\right)}\right]\right\}, \\
V_{h}^{\pi,m}(s) &=Q_{h}^{\pi,m}(s, \pi(s)), \quad V_{H+1}^{\pi,m}(s)=0.\label{eq: Vh pi m}
\end{align}
\end{subequations}
 In Equation \eqref{eq: bellman equation}, it can be seen that the action value $Q_{h}^{\pi,m}$ of step $h$ is a non-linear function of the value function $V_{h+1}^{\pi,m}$ of the later step. 
Based on Equation \eqref{eq: bellman equation}, for  $h \in[H]$ and $m \in[M]$, the Bellman optimality equation is given by
$$
\begin{aligned}
Q_{h}^{*,m}(s, a) &=r_{h}^m(s, a)+\frac{1}{\beta} \log \left\{\mathbb{E}_{s^{\prime} \sim P_{h}^m(\cdot \mid s, a)}\left[e^{\beta \cdot V_{h+1}^{*,m}\left(s^{\prime}\right)}\right]\right\}, \\
V_{h}^{*,m}(s) &=\max _{a \in \mathcal{A}} Q_{h}^{*,m}(s, a), \quad V_{H+1}^{*,m}(s)=0 .
\end{aligned}
$$
It has been recently shown in \cite{fei2021exponential} that under the risk-sensitive measurement, it is easier to analyze a simple transformation of the Bellman equation (by taking exponential on both sides of \eqref{eq: bellman equation}), which is  called \textit{exponential Bellman equation}: for every policy $\pi$ and tuple $(s, a, h, m)$, we have
\begin{align}\label{eq: exp bellman equation}
e^{\beta \cdot Q_{h}^{\pi,m}(s, a)}=\mathbb{E}_{s^{\prime} \sim P^m_{h}(\cdot \mid s, a)}\left[e^{\beta\left(r^m_{h}(s, a)+V_{h+1}^{\pi,m}\left(s^{\prime}\right)\right)}\right] .
\end{align}

When $\pi=\pi^{*,m}$, we obtain the corresponding optimality equation
\begin{align}\label{eq: optimal exp bellman equation}
e^{\beta \cdot Q_{h}^{*,m}(s, a)}=\mathbb{E}_{s^{\prime} \sim P_{h}^m(\cdot \mid s, a)}\left[e^{\beta\left(r^m_{h}(s, a)+V_{h+1}^{*,m}\left(s^{\prime}\right)\right)}\right] .
\end{align}
Note that Equation \eqref{eq: exp bellman equation} associates the current and future cumulative utilities $\left(Q_{h}^{\pi,m}\right.$ and $\left.V_{h+1}^{\pi,m}\right)$ in a multiplicative way, rather than in an additive way as in the standard Bellman equations \eqref{eq: bellman equation}. 


\subsection{Non-stationarity and variation budget}

In this work, we focus on a non-stationary environment where the  transition function $P_{h}^{m}$ and reward functions $r_h^m$ can vary over the episodes.
We measure the non-stationarity of the MDP over an interval $\mathcal{I}$ in terms of its variation in the reward functions and transition kernels:
$$
\begin{aligned}
&B_{r,\mathcal{I}}\coloneqq \sum_{m\in \mathcal{I}} \sum_{h=1}^{H} \sup _{s, a}\left|r_{h}^{m}(s, a)-r_{h}^{m+1}(s, a)\right|, \\
&B_{\mathcal{P},\mathcal{I}} \coloneqq \sum_{m\in \mathcal{I}} \sum_{h=1}^{H} \sup _{s, a}\left\|\mathcal{P}_{h}^{m}(\cdot \mid s, a)-\mathcal{P}_{h}^{m+1}(\cdot \mid s, a)\right\|_{1}.
\end{aligned}
$$
Note that our definition of variation only imposes restrictions on the summation of non-stationarity across different episodes, and does not put any restriction on the difference between two steps in the same episode. We further let $B_r\coloneqq B_{r,[1,M]}$, $B_p\coloneqq B_{p,[1,M]}$, and  $B\coloneqq B_{r}+B_{p}$, and assume $B>0$.

\subsection{Performance metrics}

Since both the reward and the transition dynamics vary over the episodes and are revealed only after a policy is decided, the agent aims to ensure the long-term optimality guarantee over some given period of episodes $M$. Suppose that the agent executes policy $\pi^{m}$ in episode $m$. We now define the dynamic regret as the difference between the total reward value of policy $\left\{\pi^{\star, m}\right\}_{m=1}^{M}$ and that of the agent's policy $\pi^{m}$ over $M$ episodes:
$$
\operatorname{D-Regret}(M):=\sum_{m=1}^{M}\left(V_{1}^{*, m}-V_{ 1}^{\pi^{m}, m}\right).
$$

\section{Restart algorithms with the knowledge of variation budget}

\subsection{Periodically restarted risk-sensitive model-based method}

We first present the Periodically Restarted Risk-sensitive Model-based method (Restart-RSMB) in Algorithm \ref{alg:algoirthm 1}. It consists of two main stages: estimation of value function (line \ref{line: begin of value estimation in alg 1}-\ref{line: end of value estimation in alg 1}) with the periodical restart (line \ref{eq: alg1 reset}) and the policy execution (line \ref{line: policy execuation in alg 1}). 

To estimate the value function under the unknown non-stationarity,  we take the optimistic value evaluation to properly handle the exploration-exploitation trade-off and apply the restart strategy to adapt to the unknown non-stationarity. In particular, we reset the visitation counters $N_{h}^m\left(s, a, s^{\prime}\right)$ and $N_{h}^m(x, a)$ to zero every $W$ episodes (line \ref{eq: alg1 reset}). Then, the reward and transition dynamics are estimated using only the data from the episode  $\ell^m=(\ceil{\frac{m}{W}}-1)W+1$ to the episode $m$ by
\begin{subequations} \label{eq: hat P and diamond}
\begin{eqnarray}
&\hspace{-1cm}\widehat{\mathcal{P}}_{h}^m\left(s^{\prime} \mid s, a\right)=\frac{N_{h}^m\left(s, a, s^{\prime}\right) + \frac{\lambda}{|\mathcal{S}|}}{N_{h}^m(s, a)+\lambda}, \label{eq: hat P} \\
\nonumber& \hspace{3cm}\text { for all }\left(s, a, s^{\prime}\right) \in \mathcal{S} \times \mathcal{A} \times \mathcal{S},   \\
&\widehat{r}_{h}^m(s, a)=\frac{ \sum_{\tau=\ell^m}^{m-1} 1\left\{(s, a)=\left(s_{h}^{\tau}, a_{h}^{\tau}\right)\right\} r^\tau_{h}\left(s_{h}^{\tau}, a_{h}^{\tau}\right)}{N_{h}^m(s, a)+\lambda}, \label{eq: hat diamond}  \\
\nonumber& \hspace{2.5cm}\text { for all }(s, a) \in \mathcal{S} \times \mathcal{A}, 
\end{eqnarray}
\end{subequations}
which are used to compute the estimated cumulative rewards at step $h$ (line \ref{eq: alg1 w h}). 
To encourage a sufficient exploration in the uncertain environment, 
Algorithm \ref{alg:algoirthm 1} applies the counter-based Upper Confidence Bound (UCB). Under the entropic risk measure, this bonus term takes the form 
\begin{align} \label{eq: bonus term thm 1}
\begin{cases}
&\hspace{-0.35cm}C_1 \left(\left(e^{\beta (H-h+1)}-1\right) + e^{\beta (H-h+1)} \beta  \right) \sqrt{\frac{\left|\mathcal{S}\right| \log \left(6W H \left|\mathcal{S}\right| \left| \mathcal{A}\right|/p\right)}{N_{h}^{m}(s, a)+1}}, \\
&\hspace{6cm}  \text{ if } \beta>0,\\
&\hspace{-0.35cm}C_1 \left(\left(1-e^{\beta (H-h+1)}\right) -\beta \right) \sqrt{\frac{\left|\mathcal{S}\right| \log \left(6W H \left|\mathcal{S}\right| \left| \mathcal{A}\right|/p\right)}{N_{h}^{m}(s, a)+1}}, \\
 & \hspace{6cm}\text{ if } \beta<0,
\end{cases}
\end{align}
for some constant $C_1>1$.  Bonus terms of the form \eqref{eq: bonus term thm 1} are called ``doubly decaying bonus''  since they shrink deterministically and exponentially across the horizon steps due to the term $e^{\beta (H-h+1)}$, apart from decreasing in the visit count.  We refer the reader to \cite{fei2021risk} for more discussion.
 
\begin{algorithm}[tb]
   \caption{Periodically Restarted Risk-sensitive Model-based RL (Restart-RSMB)}
   \label{alg:algoirthm 1}
\begin{algorithmic}[1]
   \STATE {\bfseries Inputs:} Time horizon $M$, restart period $W$;
   \FOR{$m=1, \ldots, M$}
    \STATE Set the initial state $x_1^m=x_1$ and $\ell^m=(\ceil{\frac{m}{W}}-1)W+1$;  
    \IF{$m=\ell^m$}
    \STATE $Q_h^m(s,a), V_h^m(s) \leftarrow H-h+1$ if $\beta>0$, $Q_h^m(s,a), V_h^m(s) \leftarrow 0$ if $\beta<0$, \\
    $N_h^m(s,a)\leftarrow 0, N_h^m(s,a,s^\prime) \leftarrow 0$  for all $(s,a,s^\prime, h)\in\mathcal{S}\times\mathcal{A}\times \mathcal{S}\times [H]$ \label{eq: alg1 reset};
    \ENDIF
    \FOR{ $h=H, \ldots, 1$} \label{line: begin of value estimation in alg 1}
    \FOR{ $(s,a)\in\mathcal{S}\times \mathcal{A}$}
    \STATE $w_{h}^m(s, a)=$\\$\sum_{s^\prime} \hat{\mathcal{P}}_{h}^{m}(s^\prime \mid s, a)\left[e^{\beta\left[\hat{r}_{h}^{m}(s, a)+V_{h+1}^m\left(s^\prime\right)\right]}\right]$ where $\hat{\mathcal{P}}_{h}^{m}$, $\hat{r}_{h}^{m}$ are defined in \eqref{eq: hat P and diamond}; \label{eq: alg1 w h}
    \STATE $G^m_{h}(s, a) \leftarrow$\\
    $ \begin{cases}  \min \left\{e^{\beta(H-h+1)}, w_{h}^m(s, a)+\Gamma_{h}^m(s, a)\right\},  \text {if } \beta>0; \\ 
    \max \left\{e^{\beta(H-h+1)}, w_{h}^m(s, a)-\Gamma_{h}^m(s, a)\right\},  \text {if } \beta<0; \end{cases}$ where $\Gamma_{h}^m$ is defined in \eqref{eq: bonus term thm 1};
    \STATE $V_{h}^m(s) \leftarrow \max _{a^{\prime} \in \mathcal{A}} \frac{1}{\beta} \log G_{h}^m\left(s, a^{\prime}\right)$;
    \ENDFOR 
    \ENDFOR \label{line: end of value estimation in alg 1}
    \FOR{ $h=1,2,\ldots,H$}
    \STATE Take an action $a_h^m \leftarrow \argmax _{a^{\prime} \in \mathcal{A}} \frac{1}{\beta} \log \{G_{h}^m\left(s_h^m, a^{\prime}\right)\}$, and observe $r_h(s_h^m,a_h^m)$ and $s_{h+1}^m$; \label{line: policy execuation in alg 1}
    \STATE $N_h^m(s_h^m, a_h^m) \leftarrow N_h^m(s_h^m, a_h^m)+1$; $N_h^m(s_h^m, a_h^m,s_{h+1}^m) \leftarrow N_h^m(s_h^m, a_h^m,s_{h+1}^m,)+1$;
    \ENDFOR
   \ENDFOR
\end{algorithmic}
\end{algorithm}

\subsection{Periodically restarted risk-sensitive Q-learning}

Next, we introduce  Periodically Restarted Risk-sensitive Q-learning (Restart-RSQ) in Algorithm \ref{alg:algoirthm 2}, which is model-free and inspired by RSQ2 in \cite{fei2021exponential}.
Similar to Algorithm \ref{alg:algoirthm 1},  we use the optimistic value evaluation to handle the exploration-exploitation trade-off and apply the restart strategy to adapt to the unknown non-stationarity. In particular, we re-initialize the value functions $Q_h^m(s,a), V_h^m(s)$ and
reset the visitation counter  $N_{h}^m(x, a)$ to zero every $W$ episodes (line \ref{eq: alg2 restart}). The algorithm then updates the exponential Q values using the Q-learning style update (line \ref{eq: alg2 OPE start}-\ref{eq: alg2 OPE end}) for the state action pair that just visited (line \ref{line: policy execuation in alg 2}).
The learning rate $\alpha_t$ is defined as $\frac{H+1}{H+t}$, which is motivated by \cite{jin2018q} and ensures that only the last $\mathcal{O}(\frac{1}{H})$ fraction of samples in each epoch is given non-negligible weights when used to estimate the optimistic Q-values under the non-stationarity.  Algorithm \ref{alg:algoirthm 2} also applies the UCB by incorporating a ``doubly decaying bonus'' term that takes the form 
\begin{align}\label{eq: bonus term thm 2}
\Gamma_{h,t}^m(s_h^m, a_h^m) \leftarrow C_2 \left|e^{\beta (H-h+1)}-1\right| \sqrt{\frac{|\mathcal{S}| \log (MH |\mathcal{S}| |\mathcal{A}| / \delta)}{t}} 
\end{align}
for some constant $C_2>1$.

\begin{algorithm}[tb]
   \caption{Periodically Restarted Risk-sensitive Q-learning (Restart-RSQ)}
   \label{alg:algoirthm 2}
\begin{algorithmic}[1]
   \STATE {\bfseries Inputs:} Time horizon $M$, restart period $W$;
   \FOR{$m=1, \ldots, M$}
    \STATE Set the initial state $x_1^m=x_1$ and $\ell^m=(\ceil{\frac{m}{W}}-1)W+1$;  \label{line: begin of value estimation in alg 2}
    \IF{$m=\ell^m$}
    \STATE $Q_h^m(s,a), V_h^m(s) \leftarrow H-h+1$ if $\beta>0$, $Q_h^m(s,a), V_h^m(s) \leftarrow 0$ if $\beta<0$,  $N_h^m(s,a)\leftarrow 0$  for all $(s,a,h)\in\mathcal{S}\times\mathcal{A}\times [H]$ \label{eq: alg2 restart}; 
    \ENDIF
    \FOR{ $h=1,2,\ldots,H$}
    \STATE Take an action $a_h^m \leftarrow$ \\$ \argmax _{a^{\prime} \in \mathcal{A}} \frac{1}{\beta} \log \{G_{h}^m\left(s_h^m, a^{\prime}\right)\}$, and observe $r_h^m(s_h^m,a_h^m)$ and $s_{h+1}^m$; \label{line: policy execuation in alg 2}
    \STATE $N_h^m(s_h^m, a_h^m) \leftarrow N_h^m(s_h^m, a_h^m)+1$; \\
    $t\leftarrow N_h^m(s_h^m, a_h^m)$;
    \STATE Set $\alpha_t=\frac{H+1}{H+t}$ and define $\Gamma_{h,t}^m(s_h^m, a_h^m)$ as in \eqref{eq: bonus term thm 2};
    \STATE $w_{h}^m(s_h^m, a_h^m)= (1-\alpha_t)\cdot G_h(s_h^m, a_h^m)+ \alpha_t\cdot \left[e^{\beta\left[{r}_{h}^{m}(s_h^m, a_h^m)+V_{h+1}^m\left(s^\prime\right)\right]}\right]$ \label{eq: alg2 OPE start};
    \STATE $G^m_{h}(s_h^m, a_h^m) \leftarrow$ \\
    $ \begin{cases}  
    &\hspace{-0.3cm} \min \left\{e^{\beta(H-h+1)}, w_{h}^m(s_h^m, a_h^m)+\alpha_t \Gamma_{h,t}^m(s_h^m, a_h^m)\right\}, \\
    &\hspace{5cm} \text { if } \beta>0;\label{eq: alg2 OPE end} \\ 
    &\hspace{-0.3cm} \max \left\{e^{\beta(H-h+1)}, w_{h}^m(s_h^m, a_h^m)-\alpha_t\Gamma_{h,t}^m(s_h^m, a_h^m)\right\}, \\
    &\hspace{5cm}  \text { if } \beta<0; \end{cases}$
    \STATE $V_{h}^m(s_h^m) \leftarrow \max _{a^{\prime} \in \mathcal{A}} \frac{1}{\beta} \log G_{h}^m\left(s_h^m, a^{\prime}\right)$;
    \ENDFOR \label{line: end of value estimation in alg 2}
   \ENDFOR
\end{algorithmic}
\end{algorithm}

\subsection{Theoretical results and discussions}

We now present our main theoretical results for Algorithms \ref{alg:algoirthm 1} and \ref{alg:algoirthm 2}.
\begin{theorem} \label{thm: for alg 1}
For every $\delta\in(0,1]$, with probability at least $1-\delta$ there exists a universal constant $c_1>0$ (used in Algorithm \ref{alg:algoirthm 1}) such that the dynamic regret of Algorithm \ref{alg:algoirthm 1} with $W=M^{\frac{2}{3}} B^{-\frac{2}{3}}|\mathcal{S}|^{\frac{2}{3}} |\mathcal{A}|^{\frac{1}{3}}$ is bounded by
\begin{align*}
\operatorname{D-Regret(M)}\leq & \widetilde{\mathcal{O}} \left(e^{|\beta| H}  |\mathcal{S}|^{\frac{2}{3}} |\mathcal{A}|^{\frac{1}{3}} H^{2} M^{\frac{2}{3}} B^{\frac{1}{3}}\right).
\end{align*}
\end{theorem}

\begin{theorem} \label{thm: for alg 2}
For every $\delta\in(0,1]$, with probability at least $1-\delta$ there exists a universal constant $c_2>0$ (used in Algorithm \ref{alg:algoirthm 2}) such that the dynamic regret of Algorithm \ref{alg:algoirthm 2} with $W=M^{\frac{2}{3}} H^{-\frac{3}{4}} B^{-\frac{2}{3}}|\mathcal{S}|^{\frac{2}{3}} |\mathcal{A}|^{\frac{1}{3}}$ is bounded by
\begin{align*}
\operatorname{D-Regret(M)}\leq & \widetilde{\mathcal{O}} \left(e^{|\beta| H}  |\mathcal{S}|^{\frac{1}{3}} |\mathcal{A}|^{\frac{1}{3}} H^{\frac{9}{4}} M^{\frac{2}{3}} B^{\frac{1}{3}}\right).
\end{align*}
\end{theorem}

The proofs of the two theorems are provided in Appendices \ref{app: proof of thm 1} and \ref{app: proof of thm 2}, respectively. Note that the above results generalize those in the literature of risk-neutral non-stationary RL. In particular, when $\beta \rightarrow 0$, we recover the regret bounds with the same dependence on $M$ and $B$ for the restart model-based RL \cite{domingues2021kernel} and restart Q-learning \cite{mao2020model}.


\section{Adaptive algorithm without the knowledge of variation budget}

In Theorems \ref{thm: for alg 1} and \ref{thm: for alg 2}, 
we need to set the restart period to $W=\mathcal{O}(B^{-\frac{2}{3}} M^{\frac{2}{3}})$, which clearly requires the variation budget $B$ in advance. To overcome this limitation, we propose a meta-algorithm that adaptively detects the non-stationarity without the knowledge of $B$, while still achieving the similar dynamic regret as in Theorems \ref{thm: for alg 1} and \ref{thm: for alg 2}. In particular, we generalize the black-box approach \cite{wei2021non}
to the risk-sensitive RL setting and design a non-stationarity detection based on the exponential Bellman equations \eqref{eq: exp bellman equation}. 

\begin{algorithm}[tb]
   \caption{Risk-sensitive MALG with Stationary Tests and Restarts (Adaptive-ALG)}
   \label{alg:algoirthm 3}
\begin{algorithmic}[1]
   \STATE {\bfseries Inputs:}  ALG and its associated $\rho(\cdot)$,  $\hat{n}=\log_2 M +1$, $\hat{\rho}(m)=6 \hat{n} \log(\frac{M}{\delta}) \rho(m)$;
   \FOR{$n=0, 1,\ldots, $} \label{eq: alg3 restart line}
\STATE Set $m_n \leftarrow m$ and run \text{MALG-Initialization} (Algorithm \ref{alg:algoirthm 4}) for the block $[m_n, m_n+2^n-1]$; 
    \WHILE{$m < m_n +2^n$}
    \STATE Identify the unique active instance covering the episode $m$ and denote it as $alg$; \label{eq: alg3 interaction start}
    \STATE Construct the optimistic estimator $g_m$ for the active instance $alg$;
   \STATE Follow $alg$'s decision $\pi_m$, receive estimated value $R_m=e^{\beta \sum_{h=1}^H r_h^m}$,  and update $alg$; \label{eq: alg3 interaction end}
    \STATE Set 
    $U_m=
    \begin{cases}
    \min_{\tau \in [m_n,m]} {g}_\tau, \text{ if } \beta>0,\\
    \max_{\tau \in [m_n,m]} {g}_\tau, \text{ if } \beta<0;
    \end{cases}$  
    \STATE Perform \textbf{Test1} and \textbf{Test2}; Increment $t \leftarrow t+1$;
    \STATE \textbf{If}  either test returns \textit{fail}, \textbf{then} restart from Line \ref{eq: alg3 restart line}.\label{eq: alg3 test end}
    \ENDWHILE
   \ENDFOR
   \STATE \textbf{Test1}:  Return  \textit{fail} if $m=alg.e$ for some order-k $alg$ and\\
   $
   \begin{cases}
    \frac{1}{2^k} \sum_{\tau=alg.s}^{alg.e} R_{\tau} - U_t \geq 9\hat{\rho}(2^k), \text{ if } \beta>0,\\
    U_t-\frac{1}{2^k} \sum_{\tau=alg.s}^{alg.e} R_{\tau}  \geq  9\hat{\rho}(2^k), \text{ if } \beta<0;
   \end{cases}
   $
   \STATE \textbf{Test2}: Return \textit{fail} if 
   $
   \begin{cases}
   \frac{1}{m-m_n+1} \sum_{\tau=m_n}^m ({g}_{\tau}-R_{\tau})\geq 3 \hat{\rho}(m-m_n+1), \text{ if } \beta>0,\\
   \frac{1}{m-m_n+1} \sum_{\tau=m_n}^m (R_{\tau} -{g}_{\tau})\geq 3 \hat{\rho}(m-m_n+1), \text{ if } \beta<0,
   \end{cases}
   $
\end{algorithmic}
\end{algorithm}

\subsection{Risk-sensitive non-stationary detection}

We first sketch the high-level idea of the black-box reduction approach for risk-sensitive non-stationary RL with $\beta>0$. Note that the dynamic regret can be bounded and decomposed as follows:
\begin{align} \label{eq: test decomposition}
\nonumber&\operatorname{D-Regret}(M) \\ \leq
&\frac{1}{\beta} \underbrace{\sum_{m=1}^{M}\left( e^{\beta V_{1}^{*, m}} -e^{\beta V_{ 1}^{ m}}\right)}_{\textbf{R1}} +\frac{1}{\beta} \underbrace{\sum_{m=1}^{M}\left( e^{\beta V_{1}^{m}} -e^{\beta V_{ 1}^{\pi^{m}, m}}\right)}_{\textbf{R2}}
\end{align}
where $V_{1}^{m}$ is an UCB-based optimistic estimator of the value function as constructed in Algorithms \ref{alg:algoirthm 1} and \ref{alg:algoirthm 2}. In a stationary environment with $\beta>0$, the base algorithms, such as Algorithms \ref{alg:algoirthm 1} and \ref{alg:algoirthm 2} without the restart mechanism (that is, $W=M$), ensure that $\textbf{R1}$ is simply non-positive and $\textbf{R2}$ is bounded by $\widetilde{\mathcal{O}}(M^{\frac{1}{2}})$. However, in a non-stationary environment, both terms can be substantially larger. Thus, if we can detect the event that either of the two terms is abnormally larger than the promised bound for a stationary environment, we learn that the environment has changed substantially and should restart the base algorithm. 
This detection can be easily performed for \textbf{R2} since both $e^{\beta V_{1}^{m}}$ and $e^{\beta V_{ 1}^{\pi^{m}, m}}$ are observable \footnote{More precisely, $\sum_{m=1}^M e^{\beta V_{ 1}^{\pi^{m}, m}}$ can be estimated from $\sum_{m=1}^M e^{\beta \sum_{h=1}^H r_h^m}$ using the Azuma's inequality.}, but not for $\textbf{R1}$ since $V_{1}^{*, m}$ is unknown. To address this issue, we fully utilize the fact that $e^{\beta V_{ 1}^{m}}$ is a UCB-based optimistic estimator to facilitate non-stationary detection. 

\begin{figure}[ht]
\begin{subfigure}{.5\textwidth}
  \centering
  \includegraphics[width=1\linewidth]{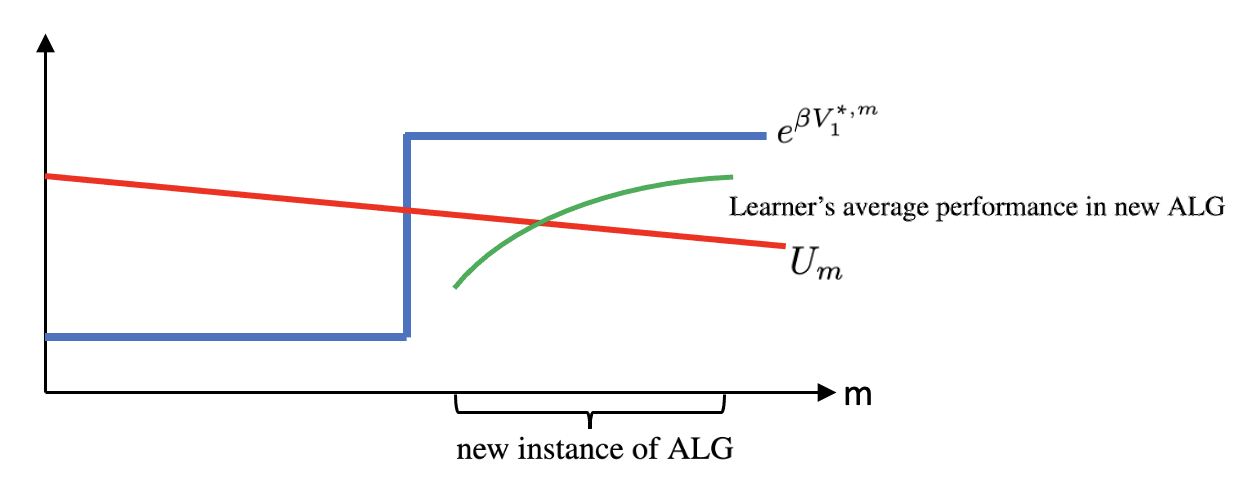}  
  \caption{$\beta>0$.}
  \label{fig:sub-first}
\end{subfigure}
\begin{subfigure}{.5\textwidth}
  \centering
  \includegraphics[width=1\linewidth]{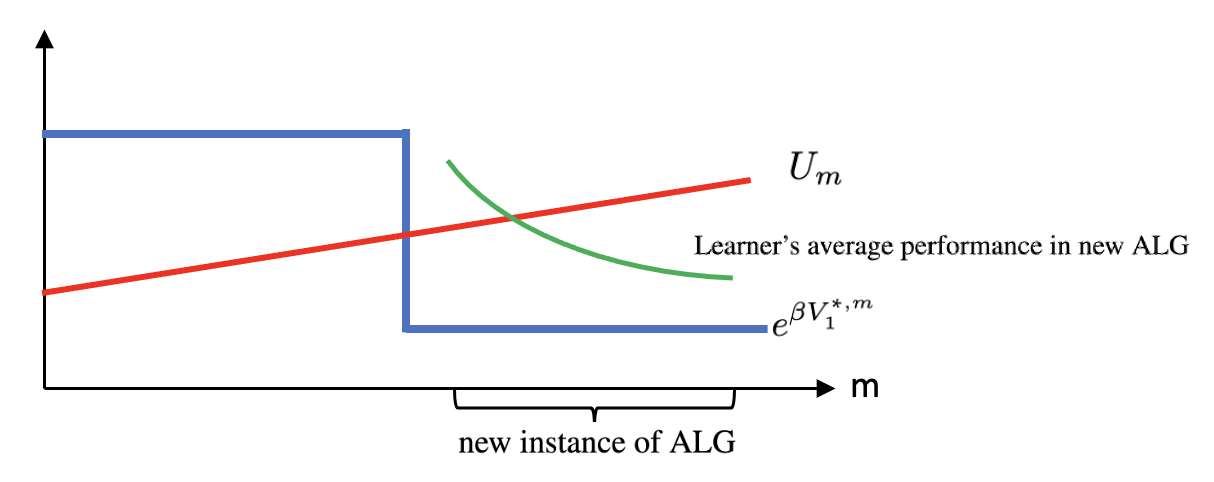}  
  \caption{$\beta<0$.}
  \label{fig:sub-second}
\end{subfigure}
\caption{An illustration of the risk-sensitive non-stationarity detection. Since both $U_m$ and learner's average performance depend on the risk-sensitive parameter $\beta$ in a non-linear way. The non-stationarity detection relies on the choice of $\beta$ and thus the risk control and the handling
of the non-stationarity can not be separately designed.}
\label{fig: risk-sensitive NS detection}
\end{figure}

We illustrate the idea of non-stationary detection for risk-sensitive RL in Figure \ref{fig: risk-sensitive NS detection}.
Here, the value of $V_1^{*,m}$ drastically increases which results to an increase in $e^{\beta V_1^{*,m}}$ for $\beta>0$ and an decrease in $e^{\beta V_1^{*,m}}$ for $\beta<0$. 
If we start running another instance of base algorithm after this environment change, then its performance will gradually approach  due to its regret guarantee in a stationary environment. 
Since the optimistic estimators should always be an upper bound of the learner's average performance in a stationary environment for $\beta>0$ or a lower bound of the learner's average performance in a stationary environment for $\beta<0$, if, at some point, we find that the new instance of the base algorithm significantly outperformances/underperformances (depending on the value of $\beta$) this quantity, we can infer that the environment has changed.

\subsection{Multi-scale ALG (MALG) and Non-stationarity Tests}

To detect the non-stationarity at different scales, we schedule and run instances of the base algorithm ALG in a randomized and multi-scale manner.
In particular, Adaptive-ALG runs MALG in a sequence of blocks with doubling lengths. Within each block, Adaptive-ALG first initializes a MALG schedule (Algorithm \ref{alg:algoirthm 4} in Appendix \ref{sec: MALG}), and then interacts the unique active instance at each episode with the environment (lines \ref{eq: alg3 interaction start}-\ref{eq: alg3 interaction end} in Algorithm \ref{alg:algoirthm 3}). At the end of each episode, Adaptive-ALG performs two non-stationarity tests (line \ref{eq: alg3 test end} in Algorithm \ref{alg:algoirthm 3}), and if either of them returns \textit{fail}, the restart is triggered. We now describe these three parts in detail below.

\textbf{MALG-initialization.}
MALG is run for an interval of length $2^n$ (unless it is terminated by the non-stationarity detection), which is called a \textit{block}. During the initialization,
MALG partions the block equally into $2^{n-k}$ sub-intervals of length $2^k$ for $k=0,1,\ldots,n$, and an instance of based algorithm (denoted by ALG) is scheduled for each of these sub-intervals with probability $\frac{\rho(2^n)}{\rho(2^k)}$, where $\rho$ is a non-increasing function associated with the bound on \textbf{R2} for ALG in a stationary environment (see Appendix \ref{app: adaptive prelim}).  We refer to these instances of length $2^k$ as order-$k$ instances. 

\textbf{MALG-interaction.}
After the initialization, MALG starts interacting with the environment as follows. In each episode $m$, the unique instance $alg$ that covers this episode with the shortest length is considered as active, while all others are regarded as inactive. MALG follows the decision of the active instance $alg$ and updates it after receiving the feedback from the environment. All inactive instances do not make any decisions or updates, that is, they are paused but may be resumed at some future episode. We refer the read to Appendix \ref{appe: MALG example} for an illustrative example for MALG procedure. 

\textbf{Non-stationarity detection}
For $\beta>0$, two non-stationarity tests are performed for the two terms in the decomposition \eqref{eq: test decomposition}. In particular, \textbf{Test1} prevents $\textbf{R1}$ from growing too large by testing if there is some order$-k$ instance's interval during which the learner's average performance $\frac{1}{2^k} \sum_{\tau=alg.s}^{alg.e} R_{\tau}$ is larger than the promised optimistic estimator $U_m=\min_{\tau \in [m_n,m]} {g}_\tau$ (for a stationary environment) by a certain amount. On the other hand, \textbf{Test2} prevents $\textbf{R2}$ from growing too large by directly testing if its average is large than the promised regret bound. The two non-stationarity tests for $\beta<0$ are similar but with $\frac{1}{2^k} \sum_{\tau=alg.s}^{alg.e} R_{\tau}$ and $U_m$ exchanged in \textbf{TEST1}, as well as with $g_{\tau}$ and $R_\tau$ exchanged in \textbf{TEST2}.


\subsection{Theoretical results and discussions}

For simplicity, we denote the revised Algorithms \ref{alg:algoirthm 1} and \ref{alg:algoirthm 2} without the restart mechanism (that is, $W=M$) as RSMB and RSQ, respectively.
We now present  our main theoretical result for Algorithm \ref{alg:algoirthm 3} when the base algorithms are RSMB and RSQ, respectively.
\begin{theorem} \label{thm: for alg 3}
For every $\delta\in(0,1]$, with probability at least $1-\delta$
it holds for Algorithm \ref{alg:algoirthm 3} that
\begin{align*}
\operatorname{D-Regret(M)}\leq 
\begin{cases}
&\widetilde{\mathcal{O}} \left(e^{|\beta| H}  |S|^{\frac{2}{3}} |A|^{\frac{1}{3}} H^{2} M^{\frac{2}{3}} B^{\frac{1}{3}}\right), \\
& \hspace{3cm} \text{if ALG is RSMB},\\
&\widetilde{\mathcal{O}} \left(e^{|\beta| H}  |S|^{\frac{2}{3}} |A|^{\frac{1}{3}} H^{\frac{5}{3}} M^{\frac{2}{3}} B^{\frac{1}{3}}\right), \\
&\hspace{3cm} \text{if ALG is RSQ}.
\end{cases}
\end{align*}
\end{theorem}
The above results show that the dynamic regret bound of the adaptive Algorithm \ref{alg:algoirthm 3} (almost) matches that of the restart Algorithms \ref{alg:algoirthm 1}-\ref{alg:algoirthm 2} that require the knowledge of the variation budget. The proof of Theorem \ref{thm: for alg 3} relies on the results in Theorems \ref{thm: for alg 1}-\ref{alg:algoirthm 2} and
is provided in Appendix \ref{app: proof of thm3}.

\section{Lower bound}

We now present a lower bound on the dynamic regret which complements the upper bounds in Theorems \ref{thm: for alg 1}, \ref{thm: for alg 2} and \ref{thm: for alg 3}.
\begin{theorem} \label{thm: low bound}
For sufficiently large $M$, there exists an instance of non-stationary MDP with $H$ horizons, state space $\mathcal{S}$, action space $\mathcal{A}$ and variation budget $B$ such that 
\begin{align*}
\operatorname{D-Regret(M)}\geq & {{\Omega}} \left(\frac{e^{\frac{2|\beta| H}{3}}-1}{|\beta|}  |\mathcal{S}|^{\frac{1}{3}} |\mathcal{A}|^{\frac{1}{3}}  M^{\frac{2}{3}} B^{\frac{1}{3}}\right).
\end{align*}
\end{theorem}
Theorem \ref{thm: low bound} shows that the exponential dependence on $|\beta|$ and $H$ in Theorems \ref{thm: for alg 1}, \ref{thm: for alg 2} and \ref{thm: for alg 3} is essentially indispensable and that the results in Theorems \ref{thm: for alg 1}, \ref{thm: for alg 2} and \ref{thm: for alg 3} are nearly optimal in their dependence on $|\mathcal{A}|, M$ and $B$. When $\beta \rightarrow 0$, we recover the existing lower bound for the non-stationary risk-neutral episodic MDP problems {\cite{mao2020model}}. 

The proof is given in Appendix \ref{app: proof of thm3}.  In the proof, the hard instance we construct is  a non-stationary MDP with piecewise constant dynamics on each segment of the horizon, and its dynamics experience an abrupt change at the beginning of each new segment. In each segment, we construct a $|\mathcal{S}| |\mathcal{A}|$-arm bandit model with Bernoulli reward for each arm. This bandit model can be seen as a special case of our episodic MDP problem, and then we show the expected regret, in terms of the logarithmic-exponential objective, that any bandit algorithm has to incur.

\section{Risk Control Under the Non-stationarity}
Risk control in non-stationary RL is more challenging since the rewards and dynamics are time-varying and unknown.
In this section, we discuss some  key ideas behind our methods and proofs.

\textbf{Normalized dynamics estimation in model-based algorithm.}
In model-based algorithms for non-stationary risk-neutral RL, the un-normalized dynamics estimation \cite{domingues2021kernel, ding2022provably} is sufficient for achieving a near-optimal regret because the effect of the model estimation error due to the ``unnormalization'' on the dynamic regret is little.
However, it is critical to use the normalized dynamics estimation \eqref{eq: hat P} in Algorithm \ref{alg:algoirthm 1}. This is because that a small model estimation error due to the ``unnormalization'' may be amplified when $\beta \rightarrow 0$. {We note that the stationary version of our Algorithm \ref{alg:algoirthm 1} is also the first model-based algorithm with a theoretical guarantee for \textit{stationary} risk-sensitive RL problems in the literature.}

\textbf{Multiplicative feature of the exponential Bellman equation.} 
The multiplicative feature of the exponential Bellman equation will involve the policy evaluation error as multiplicative terms. These terms are easy to bound in a stationary environment in light of the optimistic estimator of the exponential value function. However, 
{due to the non-stationary drifting of the environment, the estimator $V_h^m$ may no longer be an optimistic estimator and 
the errors of the optimistic estimator are all in the form of a multiplicative way due to the nature of the exponential Bellman equation. We need to introduce additional terms to guarantee each multiplicative terms are non-negative as in the proof of Theorem \ref{thm: for alg 2}.}

\textbf{Non-stationarity detection on the exponential value functions.} Different from non-stationarity detection for risk-neutral RL \cite{wei2021non}, we design non-stationarity detection mechanism for the exponential value functions \eqref{eq: exp bellman equation} instead of the value functions \eqref{eq: risk-sensitive objective} in Algorithm \ref{alg:algoirthm 3}. This is because the non-linearity of the risk-sensitive value function makes it  difficult to obtain its unbiased estimation, which is needed in the design of non-stationary detection mechanism.

\textbf{Separation design of the risk-control and the non-stationarity.} When the variation budget is known, the risk-control and the handling of the non-stationarity can be separately designed in the algorithm, that is, the restart frequency in Algorithms \ref{alg:algoirthm 1} and \ref{alg:algoirthm 2} does not depend on the risk parameter $\beta$ and only depends on the non-stationarity of the environment $B$. If we know the environment's variation budget 
in advance, then we can schedule the restart frequency ahead no matter the risk-sensitivity. 
On the other hand, without such knowledge of the variation budget, the adaptive non-stationary detection needs to take into account the risk parameter $\beta$ because the promised regret bound, the optimistic estimator, and the unbiased sample of the exponential value functions all depend on  $\beta$. 

\section{Conclusion and future work}
In this paper, we provide strong theoretical analyses for the non-stationary risk-sensitive RL problem, which is motivated by various risk-sensitive applications. We propose two restart-based algorithms that require the knowledge of the variation budget, as well as a black-box approach to turn a certain risk-sensitive RL algorithm in a (near-)stationary environment into another algorithm in a non-stationary environment without requiring the knowledge of the variation budge.
The dynamic regret bounds of these algorithms are obtained and a lower bound is established to verify the near-optimality of the proposed upper bounds. Our results also reveal the condition under which the risk control and the handling of the non-stationarity can be separately designed in the algorithm.

One important future direction lies in extending our results to other notions of risk, such as the general coherent risk measures \cite{artzner1999coherent}. Furthermore, it is useful to study how to adjust the risk sensitivity parameter adaptively in a non-stationary environment.

\bibliography{ref}

\newpage
\onecolumn
\appendix
\section{Proof of Theorem \ref{thm: for alg 1}} \label{app: proof of thm 1}
\subsection{Preliminaries}
First, we set some notations and definitions. Define $\iota:=\log (6 H |\mathcal{S}| |\mathcal{A}| W / p)$ for a given $p \in(0,1]$. We adopt the shorthand notations $\mathbb{1}_{h}^{m}(s, a):=\mathbb{1}\left\{\left(s_{h}^{m}, a_{h}^{m}\right)=(s, a)\right\}$ and $r_{h}^{m}:=r_{h}\left(s_{h}^{m}, a_{h}^{m}\right)$ for $(m, h) \in[M] \times[H]$. The epoch is defined as an interval that starts at the first episode after a restart and ends at the first
time when the restart is triggered. In Algorithm \ref{alg:algoirthm 1}, the restart mechanism divides $M$ episodes into $\ceil{\frac{M}{W}}$ epochs.

For every $(m, h) \in [M]\times [H]$, and $\left(s, a, s^{\prime}\right) \in \mathcal{S} \times \mathcal{A} \times \mathcal{S}$, we define two visitation counters $N_{h}^m\left(s, a, s^{\prime}\right)$ and $N_{h}^m(x, a)$ at step $h$ in episode $m$ as follows:
\begin{subequations}
\begin{eqnarray}\label{eq: n h m}
\begin{aligned}
N_{h}^m\left(s, a, s^{\prime}\right)=&\sum_{\tau=\ell^m}^{m-1} 1\left\{\left(s, a, s^{\prime}\right)=\left(s_{h}^{\tau}, a_{h}^{\tau}, s_{h+1}^{\tau}\right)\right\}, \\
N_{h}^m(s, a)=&\sum_{\tau=\ell^m}^{m-1} 1\left\{(s, a)=\left(s_{h}^{\tau}, a_{h}^{\tau}\right)\right\} . 
\end{aligned}
\end{eqnarray}
\end{subequations}
This allows us to estimate the transition kernel $\mathcal{P}_{h}^m$ and reward function $r^m$ for episode $m$ using only the data from the episode  $\ell^m=(\ceil{\frac{m}{W}}-1)W+1$ to the episode $m$ by
\begin{subequations}
\begin{eqnarray}
&\widehat{\mathcal{P}}_{h}^m\left(s^{\prime} \mid s, a\right)=\frac{N_{h}^m\left(s, a, s^{\prime}\right) + \frac{\lambda}{|\mathcal{S}|}}{N_{h}^m(s, a)+\lambda}, \text { for all }\left(s, a, s^{\prime}\right) \in \mathcal{S} \times \mathcal{A} \times \mathcal{S}  \label{eq: hat P appe}\\
&\widehat{r}_{h}^m(s, a)=\frac{1}{N_{h}^m(s, a)+\lambda} \sum_{\tau=\ell^m}^{m-1} 1\left\{(s, a)=\left(s_{h}^{\tau}, a_{h}^{\tau}\right)\right\} r^\tau_{h}\left(s_{h}^{\tau}, a_{h}^{\tau}\right), \text { for all }(s, a) \in \mathcal{S} \times \mathcal{A},  \label{eq: hat diamond appe} 
\end{eqnarray}
\end{subequations}
where $\lambda>0$ is the regularization parameter. We denote by $V_{h}^{m}$, $G_{h}^{m}, \Gamma_{h}^{m}$ the values of $V_{h}, G_{h}, \Gamma_{h}$ after the updates in step $h$ of episode $m$, respectively. We also set $Q_{h}^{m}=\frac{1}{\beta} \log \left\{G_{h}^{m}\right\}.$

Let us fix a pair $(s, a) \in \mathcal{S} \times \mathcal{A}$. Recall from Algorithm \ref{alg:algoirthm 1} that
$$
w_{h}^m(s, a)=\sum_{s^\prime} \hat{\mathcal{P}}_{h}^{m}(s^\prime \mid s, a)\left[e^{\beta\left[\hat{r}_{h}^{m}(s, a)+V_{h+1}^m\left(s^\prime\right)\right]}\right] .
$$
We define
$$
\begin{aligned}
q_{h, 1}^{m,+}(s, a) &:= \begin{cases}w_{h}^m(s, a)+\Gamma_h^m(s, a), & \text { if } \beta>0 \\
w_{h}^m(s, a)-\Gamma_h^m(s, a), & \text { if } \beta<0\end{cases} \\
q_{h, 1}^m(s, a) &:= \begin{cases}\min \left\{q_{h, 1}^{m,+}(s, a), e^{\beta(H-h+1)}\right\}, & \text { if } \beta>0 \\
\max \left\{q_{h, 1}^{m,+}(s, a), e^{\beta(H-h+1)}\right\}, & \text { if } \beta<0\end{cases}
\end{aligned}
$$
and
\begin{align}\label{eq: qh2}
q_{h, 2}^m(s, a):= \mathbb{E}_{s^{\prime} \sim \mathcal{P}_h^m \left(\cdot \mid s, a\right)} \left[ e^{\beta\left[r_{h}^{m}(s, a)+V_{h+1}^m\left(s^{\prime}\right)\right]}\right],
\end{align}
as well as the following for a policy $\pi$,
\begin{align}\label{eq: qh3}
q_{h, 3}^{m, \pi}(s, a):= \mathbb{E}_{s^{\prime} \sim \mathcal{P}_h^m\left(\cdot \mid s, a\right)}\left[ e^{\beta\left[r_{h}^{m}(s,a)+V_{h+1}^{\pi,m}\left(s^{\prime}\right)\right]}\right]
\end{align}

\subsection{Model prediction errors}
\begin{lemma} \label{lemma: VI estimation difference}
Define $\overline{\mathcal{V}}_{h+1}:=\left\{\bar{V}_{h+1}: \mathcal{S} \rightarrow \mathbb{R} \mid \forall s \in \mathcal{S}, \bar{V}_{h+1}(s) \in[0, H-h]\right\}$. For any $p \in(0,1]$, with probability $1-p/2$, we have
\begin{align*}
&\left|\sum_{s^{\prime} \in \mathcal{S}}\left(\widehat{\mathcal{P}}_{h}^{m}\left(s^{\prime} \mid s, a\right) e^{\beta\left[r_{h}^{m}(s, a)+\bar{V}\left(s^{\prime}\right)\right]}-\mathcal{P}_{h}^{m}\left(s^{\prime} \mid s, a\right) e^{\beta\left[r_{h}^{m}(s, a)+\bar{V}\left(s^{\prime}\right)\right]}\right)\right|\\
\leq& \Gamma_{h}^{m}+ \left|e^{\beta (H-h+1)}-1\right| B_{\mathcal{P},\mathcal{E}} 
\end{align*}
for every $(s,a, m, h)\in \mathcal{S}\times \mathcal{A} \times [M] \times [H]$ and $\bar{V} \in \overline{\mathcal{V}}_{h+1}$, where $\Gamma_{h}^{m}$ is defined in \eqref{eq: bonus term thm 1}.
\end{lemma}

\begin{proof}
For the ease of notation, we denote $\sum_{s^{\prime} \in \mathcal{S}} \mathcal{P}_{h}^{m}\left(s^{\prime} \mid s, a\right) e^{\beta\left[r_{h}^{m}(s, a)+\bar{V}\left(s^{\prime}\right)\right]}$  as $\left(\mathcal{P}_{h}^{m} e^{\beta\left[r_{h}^{m}+\bar{V}\right]}\right)(s,a)$. Then, 
for every $\bar{V} \in \mathcal{V}_{h+1}$, we consider the difference between $\sum_{s^{\prime} \in \mathcal{S}} \widehat{\mathcal{P}}_{h}^{m}\left(s^{\prime} \mid \cdot, \cdot\right) e^{\beta\left[r_{h}^{m}(s, a)+\bar{V}\left(s^{\prime}\right)\right]}$ and $\sum_{s^{\prime} \in \mathcal{S}} \mathcal{P}_{h}^{m}\left(s^{\prime} \mid \cdot, \cdot\right) e^{\beta\left[r_{h}^{m}(s, a)+\bar{V}\left(s^{\prime}\right)\right]}$ as follows:
\begin{align} 
 &\left(N_{h}^{m}(s, a)+\lambda\right)\left|\sum_{s^{\prime} \in \mathcal{S}}\left(\widehat{\mathcal{P}}_{h}^{m}\left(s^{\prime} \mid s, a\right) e^{\beta\left[r_{h}^{m}(s, a)+\bar{V}\left(s^{\prime}\right)\right]}-\mathcal{P}_{h}^{m}\left(s^{\prime} \mid s, a\right) e^{\beta\left[r_{h}^{m}(s, a)+\bar{V}\left(s^{\prime}\right)\right]}\right)\right|  \label{eq: hat P V -P V tabular}\\
\nonumber =&\left|\sum_{s^{\prime} \in \mathcal{S}} \left(N_{h}^{m}\left(s, a, s^{\prime}\right)+ \frac{\lambda}{|\mathcal{S}|}  \right) e^{\beta\left[r_{h}^{m}(s, a)+\bar{V}\left(s^{\prime}\right)\right]}-\left(N_{h}^{m}(s, a)+\lambda\right)\left(\mathcal{P}_{h}^{m} e^{\beta\left[r_{h}^{m}+\bar{V}\right]}\right)(s,a)\right| \\
\nonumber \leq &\left|\sum_{s^{\prime} \in \mathcal{S}} N_{h}^{m}\left(s, a, s^{\prime}\right) e^{\beta\left[r_{h}^{m}(s, a)+\bar{V}\left(s^{\prime}\right)\right]}-N_{h}^{m}(s,a)\left(\mathcal{P}_{h}^{m} e^{\beta\left[r_{h}^{m}+\bar{V}\right]}\right)(s,a)\right| \\
\nonumber &+\lambda\left| \frac{1}{\left|\mathcal{S} \right|}\sum_{s^\prime \in \mathcal{S}}e^{\beta\left[r_{h}^{m}(s, a)+\bar{V}\left(s^{\prime}\right)\right]}-\mathcal{P}_{h}^{m} e^{\beta\left[r_{h}^{m}+\bar{V}\right]}(s,a)\right| \\
\nonumber =&\left|\sum_{\tau=\ell^m_Q}^{m-1} \mathbb{1}\left\{(s, a)=\left(s_{h}^{\tau}, a_{h}^{\tau}\right)\right\}\left(e^{\beta\left[r_{h}^{m}(s_h^\tau, a_h^\tau)+\bar{V}\left(s_{h+1}^{\tau}\right)\right]}-\left(\mathcal{P}_{h}^{m} e^{\beta\left[r_{h}^{m}+\bar{V}\right]}\right)(s, a)\right)\right| \\
\nonumber &+\lambda\left|\frac{1}{\left|\mathcal{S} \right|}\sum_{s^\prime \in \mathcal{S}} e^{\beta\left[r_{h}^{m}(s, a)+\bar{V}\left(s^{\prime}\right)\right]}-\mathcal{P}_{h}^{m} e^{\beta\left[r_{h}^{m}+\bar{V}\right]}(s,a)\right| \\
\nonumber =&\left|\sum_{\tau=\ell^m_Q}^{m-1} \mathbb{1}\left\{(s, a)=\left(s_{h}^{\tau}, a_{h}^{\tau}\right)\right\} e^{\beta r_h^m(s_h^\tau, a_h^\tau)}\left(e^{\beta\bar{V}\left(s_{h+1}^{\tau}\right)}-\left(\mathcal{P}_{h}^{m} e^{\beta\bar{V}}\right)(s, a)\right)\right| \\
\nonumber &+\lambda\left|\frac{1}{\left|\mathcal{S} \right|}\sum_{s^\prime \in \mathcal{S}}e^{\beta\left[r_{h}^{m}(s, a)+\bar{V}\left(s^{\prime}\right)\right]}-\mathcal{P}_{h}^{m} e^{\beta\left[r_{h}^{m}+\bar{V}\right]}(s,a)\right| \\
\leq&\left|\sum_{\tau=\ell^m_Q}^{m-1} \mathbb{1}\left\{(s, a)=\left(s_{h}^{\tau}, a_{h}^{\tau}\right)\right\}e^{\beta r_h^m(s_h^\tau, a_h^\tau)}\left(e^{\beta\bar{V}\left(s_{h+1}^{\tau}\right)}-\left(\mathcal{P}_{h}^{\tau} e^{\beta\bar{V}}\right)(s, a)\right)\right| \label{eq: hat P V -P V tabular 1}\\
&+\left|\sum_{\tau=\ell^m_Q}^{m-1} \mathbb{1}\left\{(s, a)=\left(s_{h}^{\tau}, a_{h}^{\tau}\right)\right\}e^{\beta r_h^m(s_h^\tau, a_h^\tau)}\left(\left(\mathcal{P}_{h}^{\tau} e^{\beta\bar{V}}\right)(s, a)-\left(\mathcal{P}_{h}^{m} e^{\beta\bar{V}}\right)(s, a)\right)\right| \label{eq: hat P V -P V tabular 2}\\
&+\lambda\left|\frac{1}{\left|\mathcal{S} \right|}\sum_{s^\prime \in \mathcal{S}} e^{\beta\left[r_{h}^{m}(s, a)+\bar{V}\left(s^{\prime}\right)\right]}-\mathcal{P}_{h}^{m} e^{\beta\left[r_{h}^{m}+\bar{V}\right]}(s,a)\right|  \label{eq: hat P V -P V tabular 3}
\end{align}
for every $(m, h) \in[M] \times[H]$ and $(s, a) \in \mathcal{S} \times \mathcal{A}$.

To analyze the term in \eqref{eq: hat P V -P V tabular 1}, we let $\eta_{h}^{\tau}:=e^{\beta[r_h^m(s_h^\tau, a_h^\tau)+\bar{V}\left(s_{h+1}^{\tau}\right)]}-\left(\mathcal{P}_{h}^{\tau} e^{\beta[r_h^m+\bar{V}]}\right)(s_h^\tau, a_h^\tau).$ Conditioning on the filtration $\mathcal{F}_{h, 1}^{m}$, the term $\eta_{h}^{\tau}$ is a zero-mean and $\left|e^{\beta (H-h+1)}-1\right|$-sub-Gaussian random variable. By Lemma \ref{lemma: Concentration of Self-normalized Processes}, we use $Y=\lambda I$ and $X_{\tau}=\mathbb{1} \left\{(s, a)=\left(s_{h}^{\tau}, a_{h}^{\tau}\right)\right\}$ and thus with probability at least $1-\delta$ it holds for every $m \in[M]$  that
$$
\begin{aligned}
&\left(N_{h}^{m}(s, a)+\lambda\right)^{-1 / 2}\left|\sum_{\tau=\ell^m_Q}^{m-1} \mathbb{1} \left\{(s, a)=\left(s_{h}^{\tau}, a_{h}^{\tau}\right)\right\} \left(e^{\beta[r_h^m(s_h^\tau, a_h^\tau)+\bar{V}]\left(s_{h+1}^{\tau}\right)}-\left(\mathcal{P}_{h}^{\tau} e^{\beta[r_h^m+\bar{V}]}\right)(s_h^\tau, a_h^\tau)\right)\right| \\
&\leq \sqrt{\frac{(e^{\beta (H-h+1)}-1)^{2}}{2} \log \left(\frac{\left(N_{h}^{m}(s, a)+\lambda\right)^{1 / 2} \lambda^{-1 / 2}}{\delta}\right)} \\
&\leq \sqrt{\frac{(e^{\beta (H-h+1)}-1)^{2}}{2} \log \left(\frac{W}{\delta}\right)}
\end{aligned}
$$
where $W$ is the restart period.

For the term in \eqref{eq: hat P V -P V tabular 2}, by the definition of $B_{\mathcal{P},\mathcal{E}}$ and $N_h^m$, we have
\begin{align*}
&\left|\sum_{\tau=\ell^m_Q}^{m-1} \mathbb{1}\left\{(s, a)=\left(s_{h}^{\tau}, a_{h}^{\tau}\right)\right\}\left(\left(\mathcal{P}_{h}^{\tau} e^{\beta[r_h^m+\bar{V}]}\right)(s, a)-\left(\mathcal{P}_{h}^{m} e^{\beta[r_h^m+\bar{V}]}\right)(s, a)\right)\right| \\
=&\left|\sum_{\tau=\ell^m_Q}^{m-1} \mathbb{1}\left\{(s, a)=\left(s_{h}^{\tau}, a_{h}^{\tau}\right)\right\}\left(\left(\mathcal{P}_{h}^{\tau} \left(e^{\beta[r_h^m+\bar{V}]}-{1}\right)\right)(s, a)-\left(\mathcal{P}_{h}^{m} \left(e^{\beta[r_h^m+\bar{V}]}-{1}\right)\right)(s, a)\right)\right| \\
\leq & \left| \sum_{\tau=\ell^m_Q}^{m-1} \mathbb{1} \left\{(s, a)=\left(s_{h}^{\tau}, a_{h}^{\tau}\right)\right\}\right|  \left|e^{\beta (H-h+1)}-1\right| B_{\mathcal{P},\mathcal{E}} \\
\leq & \left(N_{h}^{m}(s, a)+\lambda\right)  \left|e^{\beta (H-h+1)}-1\right| B_{\mathcal{P},\mathcal{E}} .
\end{align*}
where the first equality is due to $\mathcal{P}_{h}^{m} {1}=\mathcal{P}_{h}^{\tau} {1}$ for all $\tau \in[\ell^m, m-1]$.
For the term in \eqref{eq: hat P V -P V tabular 3},  we have
\begin{align*}
\lambda\left| \frac{1}{\left|\mathcal{S} \right|}\sum_{s^\prime \in \mathcal{S}}e^{\beta\left[r_{h}^{m}(s, a)+\bar{V}\left(s^{\prime}\right)\right]}-\mathcal{P}_{h}^{m} e^{\beta\left[r_{h}^{m}+\bar{V}\right]}(s,a)\right|
\leq & \frac{\lambda}{\left|\mathcal{S} \right|}\sum_{s^\prime \in \mathcal{S}}\left|  e^{\beta\left[r_{h}^{m}(s, a)+\bar{V}\left(s^{\prime}\right)\right]}-\mathcal{P}_{h}^{m} e^{\beta\left[r_{h}^{m}+\bar{V}\right]}(s,a)\right|\\
\leq & {\lambda}  \left| e^{\beta (H-h+1)} -1\right|.
\end{align*}
By returning to \eqref{eq: hat P V -P V tabular} and setting $\lambda=1$, with probability at least $1-\delta$ it holds that
\begin{align*}
&\left|\sum_{s^{\prime} \in \mathcal{S}}\left(\widehat{\mathcal{P}}_{h}^{m}\left(s^{\prime} \mid s, a\right) e^{\beta\left[r_{h}^{m}(s, a)+\bar{V}\left(s^{\prime}\right)\right]}-\mathcal{P}_{h}^{m}\left(s^{\prime} \mid s, a\right) e^{\beta\left[r_{h}^{m}(s, a)+\bar{V}\left(s^{\prime}\right)\right]}\right)\right| \\
\leq & \left(N_{h}^{m}(s, a)+\lambda\right)^{-\frac{1}{2}}\left|e^{\beta (H-h+1)}-1\right| \sqrt{\frac{1}{2}\left(\log \left(\frac{W}{\delta}\right)\right)} +  \left|e^{\beta (H-h+1)}-1\right| B_{\mathcal{P},\mathcal{E}} +\left| e^{\beta (H-h+1)} -1\right|\\
\leq &C_1 \left(N_{h}^{m}(s, a)+\lambda\right)^{-\frac{1}{2}}\left|e^{\beta (H-h+1)}-1\right| \sqrt{ \left(\log \left(\frac{W}{\delta}\right)\right)} +  \left|e^{\beta (H-h+1)}-1\right| B_{\mathcal{P},\mathcal{E}} 
\end{align*}
for all $m \in [M]$ and for some constant $C_1>1$.

Furthermore, let $d\left(V, V^{\prime}\right)=\max _{s \in \mathcal{S}}\left|V(s)-V^{\prime}(s)\right|$ be a distance on $\mathcal{V}_{h+1} .$ For every $\epsilon$, an $\epsilon$-covering $\mathcal{V}^{\epsilon}_{h+1}$ of $\mathcal{V}_{h+1}$ with respect to distance $d(\cdot, \cdot)$ satisfies
$\left|\mathcal{V}_{h+1}^{\epsilon}\right| \leq\left(\frac{1}{\epsilon}\right)^{|\mathcal{S}|}.$ Then, for every $V \in \mathcal{V}_{h+1}$, there exists $V^{\prime} \in \mathcal{V}_{h+1}^{\epsilon}$ such that $\max _{s \in \mathcal{S}}\left|V(s)-V^{\prime}(s)\right| \leq \epsilon$, which further implies that 
\begin{align*}
\max_{s,a,s^\prime} \left|e^{\beta\left[r_{h}^{m}(s, a)+{V}\left(s^{\prime}\right)\right]}-e^{\beta\left[r_{h}^{m}(s, a)+{V}^\prime\left(s^{\prime}\right)\right]} \right|\leq & g_h(\beta)   \epsilon, 
\end{align*}
where 
\begin{align}\label{eq: g h beta}
 g_h(\beta)=\begin{cases} e^{\beta(H-h+1)} \beta,  & \text{ if }\beta>0,\\
 -\beta,&\text{ if }\beta<0.
 \end{cases}
\end{align}
Thus, by the triangle inequality and $\eqref{eq: hat P V -P V tabular}$, we have
$$
\begin{aligned}
&\left|\sum_{s^{\prime} \in \mathcal{S}}\left(\widehat{\mathcal{P}}_{h}^{m}\left(s^{\prime} \mid s, a\right) e^{\beta\left[r_{h}^{m}(s, a)+{V}\left(s^{\prime}\right)\right]}-\mathcal{P}_{h}^{m}\left(s^{\prime} \mid s, a\right) e^{\beta\left[r_{h}^{m}(s, a)+{V}\left(s^{\prime}\right)\right]}\right)\right| \\
\leq &\left|\sum_{s^{\prime} \in \mathcal{S}}\left(\widehat{\mathcal{P}}_{h}^{m}\left(s^{\prime} \mid s, a\right) e^{\beta\left[r_{h}^{m}(s, a)+{V}^\prime\left(s^{\prime}\right)\right]}-\mathcal{P}_{h}^{m}\left(s^{\prime} \mid s, a\right) e^{\beta\left[r_{h}^{m}(s, a)+{V}^\prime\left(s^{\prime}\right)\right]}\right)\right| +2  g_h(\beta)  \beta  \epsilon\\
\leq &C_1\left(N_{h}^{m}(s, a)+\lambda\right)^{-1 / 2} \left|e^{\beta (H-h+1)}-1\right| \sqrt{ \log \left(\frac{W}{\delta}\right)} +   \left|e^{\beta (H-h+1)}-1\right| B_{\mathcal{P},\mathcal{E}}+2 g_h(\beta)  \epsilon.
\end{aligned}
$$
Then, by choosing $\delta=(p / 2) /\left( \left|\mathcal{V}^{\epsilon}_{h+1}\right| H \left|\mathcal{S}\right| \left| \mathcal{A}\right|\right)$, $\epsilon=\frac{1}{4\sqrt{W}}$, and taking a union bound over $V \in \mathcal{V}_{h+1}^{\epsilon}$ and $(s, a, h) \in \mathcal{S} \times \mathcal{A}\times [H]$, it holds with probability at least $1-p / 2$ that
$$
\begin{aligned}
&\sup _{V \in \mathcal{V}_{h+1}}\left\{ \left|\sum_{s^{\prime} \in \mathcal{S}}\left(\widehat{\mathcal{P}}_{h}^{m}\left(s^{\prime} \mid s, a\right) e^{\beta\left[r_{h}^{m}(s, a)+{V}\left(s^{\prime}\right)\right]}-\mathcal{P}_{h}^{m}\left(s^{\prime} \mid s, a\right) e^{\beta\left[r_{h}^{m}(s, a)+{V}\left(s^{\prime}\right)\right]}\right)\right|\right\} \\
\leq &C_1 \left(N_{h}^{m}(s, a)+\lambda\right)^{-1 / 2} \left|e^{\beta (H-h+1)}-1\right| \sqrt{\left(\log \left(\frac{6W\left|\mathcal{V}^{\epsilon}_{h+1}\right| H \left|\mathcal{S}\right| \left| \mathcal{A}\right|}{p}\right)\right)} +    \left|e^{\beta (H-h+1)}-1\right| B_{\mathcal{P},\mathcal{E}}\\
&+2 g_h(\beta)  \epsilon\\
\leq &C_1 \left(N_{h}^{m}(s, a)+\lambda\right)^{-1 / 2}\left|e^{\beta (H-h+1)}-1\right| \sqrt{\left|\mathcal{S}\right|\left(\log \left(\frac{6W H \left|\mathcal{S}\right| \left| \mathcal{A}\right|}{p}\right)\right)} +  \left|e^{\beta (H-h+1)}-1\right| B_{\mathcal{P},\mathcal{E}}\\
&+ g_h(\beta) W^{-1/2} \\
\leq& \left(C_1 \left|e^{\beta (H-h+1)}-1\right| + g_h(\beta)  \right) \left(N_{h}^{m}(s, a)+\lambda\right)^{-1 / 2}\sqrt{\left|\mathcal{S}\right|\left(\log \left(\frac{6W H \left|\mathcal{S}\right| \left| \mathcal{A}\right|}{p}\right)\right)} +   \left|e^{\beta (H-h+1)}-1\right| B_{\mathcal{P},\mathcal{E}}\\
\leq& C_1 \left(\left|e^{\beta (H-h+1)}-1\right| + g_h(\beta)  \right) \left(N_{h}^{m}(s, a)+\lambda\right)^{-1 / 2}\sqrt{\left|\mathcal{S}\right|\left(\log \left(\frac{6W H \left|\mathcal{S}\right| \left| \mathcal{A}\right|}{p}\right)\right)} +   \left|e^{\beta (H-h+1)}-1\right| B_{\mathcal{P},\mathcal{E}}
\end{aligned}
$$
for every $(s,a, m, h)\in \mathcal{S}\times \mathcal{A} \times [M] \times [H]$. By our choice of $\Gamma_{h}^{m}$, with probability at least $1-p / 2$ it holds that
\begin{align*}
&\left|\sum_{s^{\prime} \in \mathcal{S}}\left(\widehat{\mathcal{P}}_{h}^{m}\left(s^{\prime} \mid s, a\right) e^{\beta\left[r_{h}^{m}(s, a)+\bar{V}\left(s^{\prime}\right)\right]}-\mathcal{P}_{h}^{m}\left(s^{\prime} \mid s, a\right) e^{\beta\left[r_{h}^{m}(s, a)+\bar{V}\left(s^{\prime}\right)\right]}\right)\right| \\
&\leq \Gamma_{h}^{m} +\left|e^{\beta (H-h+1)}-1\right| B_{\mathcal{P},\mathcal{E}} 
\end{align*}
for every $(s,a, m, h)\in \mathcal{S}\times \mathcal{A} \times [M] \times [H]$.
\end{proof}

\begin{lemma}\label{lemma: r estimation difference}
For every $(s,a, m, h)\in \mathcal{S}\times \mathcal{A} \times [M] \times [H]$ and $\bar{V} \in \overline{\mathcal{V}}_{h+1}$, we have
\begin{align*}
&\left|\sum_{s^{\prime} \in \mathcal{S}}\left({\widehat{\mathcal{P}}}_{h}^{m}\left(s^{\prime} \mid s, a\right) e^{\beta\left[\hat{r}_{h}^{m}(s, a)+\bar{V}\left(s^{\prime}\right)\right]}-\widehat{\mathcal{P}}_{h}^{m}\left(s^{\prime} \mid s, a\right) e^{\beta\left[r_{h}^{m}(s, a)+\bar{V}\left(s^{\prime}\right)\right]}\right)\right| \leq \Gamma_{h}^{m} + g_h(\beta) B_{r,\mathcal{E}}
\end{align*}
where $g_h(\beta)$ is defined in \eqref{eq: g h beta}.
\end{lemma}

\begin{proof}
Since 
\begin{align*}
|e^{\beta x} - e^{\beta y}| \leq \begin{cases} 
\beta e^{\beta u} |x-y|, & \text{ if } \beta>0, \\
-\beta  |x-y|, & \text{ if } \beta<0
\end{cases}
\end{align*}
for every $0\leq x\leq u$ and $0 \leq y\leq u$ where $u>0$ is some constant, it holds that 
\begin{align} \label{eq: bound on r estimation}
\nonumber &\left|\sum_{s^{\prime} \in \mathcal{S}}\left({\widehat{\mathcal{P}}}_{h}^{m}\left(s^{\prime} \mid s, a\right) e^{\beta\left[\hat{r}_{h}^{m}(s, a)+\bar{V}\left(s^{\prime}\right)\right]}-\widehat{\mathcal{P}}_{h}^{m}\left(s^{\prime} \mid s, a\right) e^{\beta\left[r_{h}^{m}(s, a)+\bar{V}\left(s^{\prime}\right)\right]}\right)\right|\\
\leq &  g_h(\beta)  \left|\hat{r}_{h}^{m}(s, a)-{r}_{h}^{m}(s, a) \right|.
\end{align}

Furthermore, by our estimation $\hat{r}_{h}^{m}(x, a)$, we have
\begin{align*}
&\left|\hat{r}_{h}^{m}(x, a)-r_{h}^{m}(x, a)\right|\\
=&\left|\hat{r}_{h}^{m}(x, a)-r_{h}^{m}{(x, a)}\right|\\
=&\left(n_{h}^{m}(x, a)+\lambda\right)^{-1}\left|\sum_{\tau=\ell^{m}}^{m-1} 1\left\{(x, a)=\left(x_{h}^{\tau}, a_{h}^{\tau}\right)\right\}\left(r_{ h}^{\tau}{\left(x_{h}^{\tau}, a_{h}^{\tau}\right)}-r_{h}^{m}{(x, a)}\right)-\lambda r_{h}^{m}{(x, a)}\right|\\
\leq & B_{r, \mathcal{E}}+\left(n_{h}^{m}(x, a)+\lambda\right)^{-1}\left|\lambda r_{h}^{m}{(x, a)}\right|\\
\leq & B_{r, \mathcal{E}}+\left(n_{h}^{m}(x, a)+\lambda\right)^{-1} \lambda\\
\leq  & B_{r, \mathcal{E}}+\left(n_{h}^{m}(x, a)+\lambda\right)^{-1 / 2} \lambda
\end{align*}
By substituting the above inequality into \eqref{eq: bound on r estimation} and setting $\lambda=1$, we obtain the desired results.
\end{proof}

\begin{lemma} \label{lemma: VI estimation difference combined}
For every $p \in(0,1]$, with probability $1-p/2$, we have
\begin{align*}
&\left|\sum_{s^{\prime} \in \mathcal{S}}\left(\widehat{\mathcal{P}}_{h}^{m}\left(s^{\prime} \mid s, a\right) e^{\beta\left[\hat{r}_{h}^{m}(s, a)+\bar{V}\left(s^{\prime}\right)\right]}-\mathcal{P}_{h}^{m}\left(s^{\prime} \mid s, a\right) e^{\beta\left[{r}_{h}^{m}(s, a)+\bar{V}\left(s^{\prime}\right)\right]}\right)\right|\\
\leq & 2\Gamma_{h}^{m}+ \left|e^{\beta (H-h+1)}-1\right| B_{\mathcal{P},\mathcal{E}} +  g_h(\beta) B_{r,\mathcal{E}}
\end{align*}
where $g_h(\beta)$ is defined in \eqref{eq: g h beta}, for every $(s,a, m, h)\in \mathcal{S}\times \mathcal{A} \times [M] \times [H]$ and $\bar{V} \in \overline{\mathcal{V}}_{h+1}$.
\end{lemma}
\begin{proof}
The proof follows from Lemma \ref{lemma: VI estimation difference}, Lemma \ref{lemma: r estimation difference} and Cauchy-Schwartz inequality. 
\end{proof}

\subsection{Value difference bounds}

\begin{lemma} \label{lemma: difference q1-q2}
Recall the definition of $\Gamma_{h}^{m}$ from Algorithm \ref{alg:algoirthm 1}. For all $(m, h, s, a) \in[M] \times[H] \times \mathcal{S} \times \mathcal{A}$, 
the following statement holds with probability at least $1-p / 2$:
\begin{itemize}
    \item If $\beta>0$:
\begin{align*}
- \left|e^{\beta (H-h+1)}-1\right| B_{\mathcal{P},\mathcal{E}} - g_h(\beta) B_{r,\mathcal{E}}  \leq & \left(q_{h, 1}^{m}-q_{h, 2}^{m}\right)(s, a) \\ 
\leq & 4\Gamma_{h}^{m}+ \left|e^{\beta (H-h+1)}-1\right| B_{\mathcal{P},\mathcal{E}} +  g_h(\beta) B_{r,\mathcal{E}}.
\end{align*}
\item If $\beta<0$:
\begin{align*}
- \left|e^{\beta (H-h+1)}-1\right| B_{\mathcal{P},\mathcal{E}} - g_h(\beta) B_{r,\mathcal{E}}  \leq & \left(q_{h, 2}^{m}-q_{h, 1}^{m}\right)(s, a) \\ 
\leq & 4\Gamma_{h}^{m}+ \left|e^{\beta (H-h+1)}-1\right| B_{\mathcal{P},\mathcal{E}} +  g_h(\beta) B_{r,\mathcal{E}}.
\end{align*}
\end{itemize}
(Note that $g_h(\beta)$ is defined in \eqref{eq: g h beta}).
\end{lemma}
\begin{proof}
We focus on the case of $\beta> 0$ since the proof for $\beta<0$ is similar.
We first fix a tuple $(m, h, s, a) \in[M] \times[H] \times \mathcal{S} \times \mathcal{A}$. 
By the definitions of $q_{h, 1}^{m,+}$ and $q_{h, 2}^{m}$,  one can compute
\begin{align*}
&\left|\left(q_{h, 1}^{m,+}-2\Gamma_{h}^{m}-q_{h, 2}^{m}\right)(s, a)\right|\\
&=\left|\left(w_{h}^{m}-q_{h, 2}^{m}\right)(s, a)\right|\\
&=\left|\sum_{s^{\prime} \in \mathcal{S}}\left(\widehat{\mathcal{P}}_{h}^{m}\left(s^{\prime} \mid s, a\right) e^{\beta\left[\hat{r}_{h}^{m}(s, a)+\bar{V}\left(s^{\prime}\right)\right]}-\mathcal{P}_{h}^{m}\left(s^{\prime} \mid s, a\right) e^{\beta\left[{r}_{h}^{m}(s, a)+\bar{V}\left(s^{\prime}\right)\right]}\right)\right| \\
&\leq 2\Gamma_{h}^{m}+ \left|e^{\beta (H-h+1)}-1\right| B_{\mathcal{P},\mathcal{E}} +  g_h(\beta) B_{r,\mathcal{E}}
\end{align*}
where the last step holds by Lemma \ref{lemma: VI estimation difference}. Then, we have
\begin{align*}
- \left|e^{\beta (H-h+1)}-1\right| B_{\mathcal{P},\mathcal{E}} -  g_h(\beta) B_{r,\mathcal{E}}  \leq & \left(q_{h, 1}^{m,+}-q_{h, 2}^{m}\right)(s, a) \\ 
\leq & 4\Gamma_{h}^{m}+ \left|e^{\beta (H-h+1)}-1\right| B_{\mathcal{P},\mathcal{E}} +  g_h(\beta) B_{r,\mathcal{E}}.
\end{align*}

Furthermore, if $q_{h, 1}^{m,+}\leq e^{\beta(H-h+1)}$, one can write
$$
q_{h, 1}^{m,+}-q_{h, 2}^{m}=q_{h, 1}^{m}-q_{h, 2}^{m}\geq - \left|e^{\beta (H-h+1)}-1\right| B_{\mathcal{P},\mathcal{E}} -  g_h(\beta) B_{r,\mathcal{E}}. 
$$
If  $q_{h, 1}^{m,+}\geq e^{\beta(H-h+1)}$, we have $q_{h, 1}^{m,+}-q_{h, 2}^{m}=e^{\beta(H-h+1)}-q_{h, 2}^{m}\geq 0$.
In addition, since $q_{h, 1}^{m,+} \geq q_{h, 1}^{m}$, it holds that $q_{h, 1}^{m}-q_{h, 2}^{m} \leq q_{h, 1}^{m,+}-q_{h, 2}^{m}$. This completes the proof. 

\end{proof}

\begin{lemma} \label{lemma: exp Qm-exp Qpi}
On the event of Lemma \ref{lemma: difference q1-q2}, for all $(m, h, s, a) \in[M] \times[H] \times \mathcal{S} \times \mathcal{A}$ and every policy $\pi$:
\begin{itemize}
    \item If $\beta>0$:
$$
e^{\beta \cdot Q_{h}^{m}(s, a)} - e^{\beta \cdot Q_{h}^{\pi, m}(s, a)}\geq -(H-h+1)\left[ \left|e^{\beta (H-h+1)}-1\right| B_{\mathcal{P},\mathcal{E}} +  g_h(\beta) B_{r,\mathcal{E}} \right].
$$
\item If $\beta<0$:
 $$
e^{\beta \cdot Q_{h}^{m}(s, a)} - e^{\beta \cdot Q_{h}^{\pi, m}(s, a)} \leq (H-h+1) \left[ \left|e^{\beta (H-h+1)}-1\right| B_{\mathcal{P},\mathcal{E}} +  g_h(\beta) B_{r,\mathcal{E}} \right].
$$
 \end{itemize}
\end{lemma}

\begin{proof}
We focus on the case of $\beta>0$ since the proof for $\beta<0$ is  similar. For the purpose of the proof, we set $Q_{H+1}^{\pi,m}(s, a)=Q_{H+1}^{*,m}(s, a)=0$ for all $(s, a) \in \mathcal{S} \times \mathcal{A}$. We fix a tuple $(m, s, a) \in$ $[M] \times \mathcal{S} \times \mathcal{A}$ and use strong induction on $h$. The base case for $h=H+1$ is satisfied since $e^{\beta \cdot Q_{H+1}^{m}(s, a)}=e^{\beta \cdot Q_{H+1}^{\pi, m}(s, a)}=1$ for all $m \in[M]$ by definition. Now, we fix an index $h \in[H]$ and assume that 
$$
e^{\beta \cdot Q_{h+1}^{m}(s, a)} -e^{\beta \cdot Q_{h+1}^{\pi, m}(s, a)}\geq  -(H-h)\left[ \left|e^{\beta (H-h)}-1\right| B_{\mathcal{P},\mathcal{E}} +   g_h(\beta) B_{r,\mathcal{E}} \right].
$$
Moreover, by the induction assumption, we have
\begin{align}
\label{eq: exp Vm > exp Vpi}
\nonumber e^{\beta \cdot V_{h+1}^{m}(s)}&=\max _{a^{\prime} \in \mathcal{A}} e^{\beta \cdot Q_{h+1}^{m}\left(s, a^{\prime}\right)}\\
\nonumber&\geq \max _{a^{\prime} \in \mathcal{A}} e^{\beta \cdot Q_{h+1}^{\pi, m}\left(s, a^{\prime}\right)}-(H-h)\left[ \left|e^{\beta (H-h)}-1\right| B_{\mathcal{P},\mathcal{E}} +   g_h(\beta) B_{r,\mathcal{E}} \right]\\
&\geq e^{\beta \cdot V_{h+1}^{\pi, m}(s)}-(H-h)\left[ \left|e^{\beta (H-h)}-1\right| B_{\mathcal{P},\mathcal{E}} +  g_h(\beta)B_{r,\mathcal{E}} \right].
\end{align}
By the definitions of $q_{h, 2}^{m}$ and $q_{h, 3}^{m,\pi}$, it follows from \eqref{eq: exp Vm > exp Vpi} that
$$q_{h, 2}^{m} - q_{h, 3}^{m, \pi}\geq -(H-h)\left[ \left|e^{\beta (H-h+1)}-1\right| B_{\mathcal{P},\mathcal{E}} +  g_h(\beta) B_{r,\mathcal{E}} \right].$$ 

In addition, on the event of Lemma \ref{lemma: difference q1-q2}, we also have
$$q_{h, 1}^{m}-q_{h, 2}^{m} \geq -\left[ \left|e^{\beta (H-h+1)}-1\right| B_{\mathcal{P},\mathcal{E}} +   g_h(\beta) B_{r,\mathcal{E}} \right].$$
Therefore, it follows that 
\begin{align*}
 \left( e^{\beta \cdot Q_{h}^{m}}- e^{\beta \cdot Q_{h}^{\pi, m}}\right)(s, a) =&\left(q_{h, 1}^{m}-q_{h, 3}^{m, \pi}\right)(s, a)\\
 =&\left(q_{h, 1}^{m}-q_{h, 2}^{m}\right)(s, a)+\left(q_{h, 2}^{m}-q_{h, 3}^{m, \pi}\right)(s, a)\\
 \geq & - (H-h+1)\left[ \left|e^{\beta (H-h+1)}-1\right| B_{\mathcal{P},\mathcal{E}} +  g_h(\beta) B_{r,\mathcal{E}} \right]
\end{align*}
which completes the induction.
\end{proof}

\begin{lemma} \label{lemma: bound on V difference}
 For all $(m, h, s) \in[M] \times[H] \times \mathcal{S}$, policy $\pi$ and $\delta \in(0,1]$, with probability at least $1-\delta / 2$:
 \begin{itemize}
\item  If $\beta>0$:
$$
e^{\beta \cdot V_{h}^{m}(s, a)} - e^{\beta \cdot V_{h}^{\pi, m}(s, a)}\geq -(H-h+1)\left[ \left|e^{\beta (H-h+1)}-1\right| B_{\mathcal{P},\mathcal{E}} +  g_h(\beta) B_{r,\mathcal{E}} \right].
$$
\item If $\beta<0$:
 $$
e^{\beta \cdot V_{h}^{m}(s, a)} - e^{\beta \cdot V_{h}^{\pi, m}(s, a)} \leq (H-h+1) \left[ \left|e^{\beta (H-h+1)}-1\right| B_{\mathcal{P},\mathcal{E}} +  g_h(\beta) B_{r,\mathcal{E}} \right].
$$
 \end{itemize}
\end{lemma}
Proof. The result follows from Lemma \ref{lemma: exp Qm-exp Qpi} and Equation \eqref{eq: exp Vm > exp Vpi}.

\subsection{Proof of Theorem \ref{thm: for alg 1}}
We first consider $\beta>0$. For $h \in[H]$, we define
\begin{subequations}
\begin{align}
\delta_{h}^{m} &:=e^{\beta V_{h}^{m}\left(s_{h}^{m}\right)}-e^{\beta V_{h}^{\pi^{m},m}\left(s_{h}^{m}\right)}, \label{eq: delta h m} \\
\zeta_{h+1}^{m} &:= q_{h,2}^m -  q_{h,3}^m -e^{\beta r_{h}^m\left(s_{h}^{m}, a_{h}^{m}\right)} \delta_{h+1}^{m} \label{eq: zeta h+1 m} \\
\nonumber &=\left[P_{h}^m\left(e^{\beta\left[r_{h}^m\left(s_{h}^{m}, a_{h}^{m}\right)+V_{h+1}^{m}\left(s^{\prime}\right)\right]}-e^{\beta\left[r_{h}^m\left(s_{h}^{m}, a_{h}^{m}\right)+V_{h+1}^{\pi^{m},m}\left(s^{\prime}\right)\right]}\right)\right]\left(s_{h}^{m}, a_{h}^{m}\right)-e^{\beta r_{h}^m\left(s_{h}^{m}, a_{h}^{m}\right)} \delta_{h+1}^{m},
\end{align}
\end{subequations}
where $\left[P_{h}^m f\right](s, a):=\mathbb{E}_{s^{\prime} \sim P_{h}^m(|\cdot| s, a)}\left[f\left(s^{\prime}\right)\right]$ for every $f: \mathcal{S} \rightarrow \mathbb{R}$ and $(s, a) \in \mathcal{S} \times \mathcal{A}$. 
 Then, for every $(m, h) \in[M] \times[H]$, we have
\begin{align}
\nonumber \delta_{h}^{m} \stackrel{(i)}{=}&\left(e^{\beta \cdot Q_{h}^{m}}-e^{\beta \cdot Q_{h}^{\pi^{m},m}}\right)\left(s_{h}^{m}, a_{h}^{m}\right) \\
\nonumber \stackrel{(i i)}{=}&   q_{h,1}^m\left(s_{h}^{m}, a_{h}^{m}\right) -q_{h,2}^m\left(s_{h}^{m}, a_{h}^{m}\right)+q_{h,2}^m\left(s_{h}^{m}, a_{h}^{m}\right)-q_{h,3}^m\left(s_{h}^{m}, a_{h}^{m}\right) \\
\nonumber \stackrel{(i i i)}{\leq} &4\Gamma_{h}^{m}+ \left|e^{\beta (H-h+1)}-1\right| B_{\mathcal{P},\mathcal{E}} +   g_h(\beta) B_{r,\mathcal{E}}+q_{h,2}^m\left(s_{h}^{m}, a_{h}^{m}\right)-q_{h,3}^m\left(s_{h}^{m}, a_{h}^{m}\right) \\
=& 4\Gamma_{h}^{m}+ \left|e^{\beta (H-h+1)}-1\right| B_{\mathcal{P},\mathcal{E}} +  g_h(\beta) B_{r,\mathcal{E}}+e^{\beta \cdot r_{h}^m\left(s_{h}^{m}, a_{h}^{m}\right)} \delta_{h+1}^{m}+\zeta_{h+1}^{m} \label{eq: recursion 1} .
\end{align}

In the above equation, step (i) holds by the construction of Algorithm \ref{alg:algoirthm 1} and the definition of $V_{h}^{\pi^{m}}$ in Equation \eqref{eq: Vh pi m}; step (ii) holds by Equations \eqref{eq: qh2} and \eqref{eq: qh3}; step (iii) holds on the event of Lemma \ref{lemma: difference q1-q2}; the last step follows from the definition of $\delta_h^m$ and $\zeta_h^m$ in Equations \ref{eq: delta h m} and \ref{eq: zeta h+1 m}.

Using the fact that $V_{H+1}^{m}(s)=V_{H+1}^{\pi^{m}}(s)=0$, we can expand the recursion in Equation \eqref{eq: recursion 1} to obtain
\begin{align*}
\delta_{1}^{m} \leq &   \sum_{h \in[H]} e^{\beta \sum_{i=1}^{h-1} r_{i}^m} \zeta_{h+1}^{m}+ \sum_{h \in[H]} e^{\beta \sum_{i=1}^{h-1} r_{i}^m}\left(4 \Gamma_{h}^{m} + \left|e^{\beta (H-h+1)}-1\right| B_{\mathcal{P},\mathcal{E}} +  g_h(\beta) B_{r,\mathcal{E}}\right) \\
\leq & \sum_{h \in[H]} e^{\beta \sum_{i=1}^{h-1} r_{i}^m} \zeta_{h+1}^{m}+ \sum_{h \in[H]} e^{\beta(h-1)} \left(4 \Gamma_{h}^{m} + \left|e^{\beta (H-h+1)}-1\right| B_{\mathcal{P},\mathcal{E}} +  g_h(\beta) B_{r,\mathcal{E}}\right).
\end{align*}
where the last step follows from $r_{h}^m(\cdot, \cdot) \in[0,1]$.
Summing the above display over $m \in[M]$ gives
\begin{align}
\nonumber &\sum_{m \in[M]} \delta_{1}^{m} \\
\nonumber \leq & \sum_{m \in[M]} \sum_{h \in[H]} e^{\beta \sum_{i=1}^{h-1} r_{i}^m} \zeta_{h+1}^{m}+ \sum_{m \in[M]} \sum_{h \in[H]} e^{\beta(h-1)}\left(4 \Gamma_{h}^{m} + \left|e^{\beta (H-h+1)}-1\right| B_{\mathcal{P},\mathcal{E}} +   g_h(\beta) B_{r,\mathcal{E}}\right)\\
\nonumber =&\sum_{\mathcal{E}=1}^{\ceil{\frac{M}{W}}} \sum_{m=(\mathcal{E}-1)W}^{\mathcal{E}W} \sum_{h \in[H]} \left(e^{\beta \sum_{i=1}^{h-1} r_{i}^m} \zeta_{h+1}^{m}+ e^{\beta(h-1)}\left(4 \Gamma_{h}^{m} + \left|e^{\beta (H-h+1)}-1\right| B_{\mathcal{P},\mathcal{E}} +  g_h(\beta) B_{r,\mathcal{E}}\right) \right)\\
=&\sum_{\mathcal{E}=1}^{\ceil{\frac{M}{W}}} \sum_{m=(\mathcal{E}-1)W}^{\mathcal{E}W} \sum_{h \in[H]} \left(e^{\beta \sum_{i=1}^{h-1} r_{i}^m} \zeta_{h+1}^{m}+ 4 e^{\beta(h-1)} \Gamma_{h}^{m} \right)+ WH\left(\left|e^{\beta H}-1\right| B_{\mathcal{P}} +  g_1(\beta) B_{r}\right) \label{eq: sum of delta 1} .
\end{align}

We aim to control the terms in \eqref{eq: sum of delta 1}. 
Since $\left\{e^{\beta \sum_{i=1}^{h-1} r_{i}^m}\zeta_{h+1}^{m}\right\}$ is a martingale difference sequence satisfying $\left|e^{\beta \sum_{i=1}^{h-1} r_{i}^m}\zeta_{h+1}^{m}\right| \leq 2 |e^{\beta H} -1 |$ for all $(m, h) \in[M] \times[H]$, by the Azuma-Hoeffding inequality, we have:
$$
\mathcal{P}\left(\sum_{m \in[M]} \sum_{h \in[H]} e^{\beta \sum_{i=1}^{h-1} r_{i}^m} \zeta_{h+1}^{m} \geq t\right) \leq \exp \left(-\frac{t^{2}}{8 H M\left(e^{\beta H}-1\right)^{2}}\right), \quad \forall t>0.
$$
Hence, with probability $1-\delta / 2$, it holds that
\begin{align} \label{eq: sum of martingale}
\sum_{k \in[K]} \sum_{h \in[H]} e^{\beta(h-1)} \zeta_{h+1}^{m} \leq \left(e^{\beta H}-1\right) \sqrt{2 H M \log (2 / \delta)} \leq 2\left(e^{\beta H}-1\right) \sqrt{2 H M \iota},
\end{align}
where $\iota=\log (6 H \left| \mathcal{S}\right| \left| \mathcal{A}\right| W / \delta)$. Furthermore, recall the definition of $\Gamma_{h}^{m}$, we can derive
$$
\begin{aligned}
&\sum_{m=(\mathcal{E}-1)W}^{\mathcal{E}W} \sum_{h \in[H]}  e^{\beta(h-1)} \Gamma_{h}^{m} \\
& \leq \sum_{m=(\mathcal{E}-1)W}^{\mathcal{E}W} \sum_{h \in[H]}  \left(C_1 \left|e^{\beta (H-h+1)}-1\right| +  g_h(\beta)\right) \sqrt{\left|\mathcal{S}\right|\iota} \sqrt{\frac{1}{ N_{h}^{m}\left(s_{h}^{m}, a_{h}^{m}  \right)+1}} \\
&\leq \left(C_1 \left|e^{\beta H}-1\right| +  g_1(\beta)\right) \sqrt{\left|\mathcal{S}\right|\iota} \sum_{m=(\mathcal{E}-1)W}^{\mathcal{E}W}  \sum_{h \in[H]} \sqrt{\frac{1}{N_{h}^{m}\left(s_{h}^{m}, a_{h}^{m}\right)+1}}\\
&\stackrel{(i)}{\leq} \left(C_1 \left|e^{\beta H}-1\right| + g_1(\beta)\right) \sqrt{\left|\mathcal{S}\right|\iota} \sum_{h \in[H]} \sqrt{W} \sqrt{\sum_{m=(\mathcal{E}-1)W}^{\mathcal{E}W}  \frac{1}{ N_{h}^{m}\left(s_{h}^{m}, a_{h}^{m}\right)+1}}\\
&\leq \left(C_1 \left|e^{\beta H}-1\right| +  g_1(\beta)\right) \sqrt{\left|\mathcal{S}\right|\iota} \sqrt{2 H^{2} |S| |A| W \iota}
\end{aligned}
$$
where step (i) follows the Cauchy-Schwarz inequality and the last step holds by the pigeonhole principle. Thus, it holds that
\begin{align}\label{eq: sum of bonus term}
&\sum_{m \in [M]} \sum_{h \in[H]}  e^{\beta(h-1)} \Gamma_{h}^{m} \leq \left(C_1 \left|e^{\beta H}-1\right| + e^{\beta H} \left|\beta\right|\right) \sqrt{2 H^{2} |S|^2 |A| \iota^2} \frac{M}{\sqrt{W}}.
\end{align}

Substituting  \eqref{eq: sum of martingale} and \eqref{eq: sum of bonus term} into \eqref{eq: sum of delta 1} yields that
\begin{align} \label{eq: sum of delta 1 m}
\sum_{m \in[M]} \delta_{1}^{m}\leq & 2 \left|e^{\beta H}-1\right| \sqrt{2 H M \iota}+\left(C_1 \left|e^{\beta H}-1\right| + g_1(\beta)\right) \sqrt{2 H^{2} |S|^2 |A| \iota^2} \frac{M}{\sqrt{W}}\\
\nonumber &+ WH\left(\left|e^{\beta H}-1\right| B_{\mathcal{P}} +   g_1(\beta) B_{r}\right)
\end{align}

For $\beta>0$, we have that $g_1(\beta)=e^{\beta H} \beta$  and the dynamic regret can be decomposed based on Lemma \ref{lemma: decompose of d regret}:
\begin{align} \label{eq: decompose of d regret in proof 1}
\nonumber &\operatorname{D-Regret}(M) \\
\nonumber \leq &\sum_{m \in[M]} \frac{1}{\beta}\left[e^{\beta \cdot V_{1}^{*,m}\left(s_{1}^{m}\right)}-e^{\beta \cdot V_{1}^{m}\left(s_{1}^{m}\right)}\right] + \sum_{m \in[M]} \frac{1}{\beta}\left[e^{\beta \cdot V_{1}^{m}\left(s_{1}^{m}\right)}-e^{\beta \cdot V_{1}^{\pi^{m},m}\left(s_{1}^{m}\right)}\right]\\
\nonumber \leq &  \frac{1}{\beta} \sum_{\mathcal{E}=1}^{\ceil{\frac{M}{W}}} \sum_{m=(\mathcal{E}-1)W}^{\mathcal{E}W} H \left(\left|e^{\beta H}-1\right| B_{\mathcal{P},\mathcal{E}} +  g_1(\beta) B_{r,\mathcal{E}}\right)+ \sum_{m \in[M]} \frac{1}{\beta}\left[e^{\beta \cdot V_{1}^{m}\left(s_{1}^{m}\right)}-e^{\beta \cdot V_{1}^{\pi^{m},m}\left(s_{1}^{m}\right)}\right]\\
\nonumber \leq &  \frac{1}{\beta} WH\left(\left|e^{\beta H}-1\right| B_{\mathcal{P}} +   g_1(\beta) B_{r}\right)+ \frac{1}{\beta} \sum_{m \in[M]} \delta_1^m\\
\nonumber \leq &  \frac{1}{\beta} \left(2 \left(e^{\beta H}-1\right) \sqrt{2 H M \iota}+\left(C_1 \left(e^{\beta H}-1\right) + e^{\beta H} \beta\right) \sqrt{2 H^{2} |S|^2 |A| \iota^2} \frac{M}{\sqrt{W}} \right.\\
\nonumber &\left.+ WH\left(\left(e^{\beta H}-1\right) B_{\mathcal{P}} +   e^{\beta H} \beta B_{r}\right)\right)\\
\nonumber \leq & 2 e^{\beta H}H  \sqrt{2 H M \iota}+e^{\beta H} \left(C_1H + 1 \right) \sqrt{2 H^{2} |S|^2 |A| \iota^2} \frac{M}{\sqrt{W}} + WHe^{\beta H} \left( H B_{\mathcal{P}} +    B_{r}\right)\\
\leq & 2 e^{\beta H}H  \sqrt{2 H M \iota}+ (C_1+1) e^{\beta H} H \sqrt{2 H^{2} |S|^2 |A| \iota^2} \frac{M}{\sqrt{W}} + WH^2 e^{\beta H} \left(B_{\mathcal{P}} +    B_{r}\right)
\end{align}
where the second inequality follows from Lemma \ref{lemma: bound on V difference}, the third inequality holds because of the definition of $B_{\mathcal{P}}$, $B_{{r}}$ and $\delta_1^m$, the forth inequality is due to \eqref{eq: sum of delta 1 m}, and the fifth inequality follows from $e^{\beta H}-1\leq \beta H e^{\beta H}$ for $\beta>0$.

Finally, by setting $W=M^{\frac{2}{3}} \left(B_{\mathcal{P}}+ B_{r}\right)^{-\frac{2}{3}}|S|^{\frac{2}{3}} |A|^{\frac{1}{3}}$, we conclude that
\begin{align*}
\operatorname{D-Regret(M)}\leq & \widetilde{\mathcal{O}} \left(e^{\beta H}  |S|^{\frac{2}{3}} |A|^{\frac{1}{3}} H^2 M^{\frac{2}{3}} \left(B_{\mathcal{P}}+ B_{r}\right)^{\frac{1}{3}}\right).
\end{align*}

The proof of $\beta<0$ follows a similar procedure and is therefore omitted.
\section{Proof of Theorem \ref{thm: for alg 2} }\label{app: proof of thm 2}
\subsection{Preliminaries}
We first lay out some additional notations to facilitate our proof. Let $N_{h}^{m}, G_{h}^{m}, V_{h}^{m}$ be $N_{h}, G_{h}, V_{h}$ at the beginning of episode $m$, before $t$ is updated. We also set $Q_{h}^{m}:=\frac{1}{\beta} G_{h}^{m}$. Let $\widehat{P}_{h}^{m}(\cdot \mid s, a)$ denote the delta function centered at $s_{h+1}^{m}$ for all $(m, h, s, a) \in[M] \times[H] \times \mathcal{S} \times \mathcal{A}$. This means that $\mathbb{E}_{s^{\prime} \sim \widehat{P}_{h}^{m}(\cdot \mid s, a)}\left[f\left(s^{\prime}\right)\right]=f\left(s_{h+1}^{m}\right)$ for every $f: \mathcal{S} \rightarrow \mathbb{R}$. Denote by $n_{h}^{m}:=N_{h}^{m}\left(s_{h}^{m}, a_{h}^{m}\right)$. Recall from Algorithm \ref{alg:algoirthm 2} that the learning rate is defined as
$$
\alpha_{t}:=\frac{H+1}{H+t}
$$
for $t \in \mathbb{Z}$. We also define
\begin{align}\label{eq: definition of alpha t i}
\alpha_{t}^{0}:=\prod_{j=1}^{t}\left(1-\alpha_{j}\right), \quad \alpha_{t}^{i}:=\alpha_{i} \prod_{j=i+1}^{t}\left(1-\alpha_{j}\right)
\end{align}

for integers $i, t \geq 1$. We set $\alpha_{t}^{0}=1$ and $\sum_{i \in[t]} \alpha_{t}^{i}=0$ if $t=0$, and $\alpha_{t}^{i}=\alpha_{i}$ if $t<i+1$.

The epoch is defined as an interval that starts at the first episode after a restart and ends at the first
time when the restart is triggered. In Algorithm \ref{alg:algoirthm 2}, the restart mechanism divides $M$ episodes into $\ceil{\frac{M}{W}}$ epochs.

Define the shorthand notation $\iota:=\log (|\mathcal{S}| |\mathcal{A}| MH / \delta)$ for $\delta \in(0,1]$. We fix a tuple $(m, h, s, a) \in[M] \times[H] \times$ $\mathcal{S} \times \mathcal{A}$ with $m_{i}^{\mathcal{E}} \leq M$ being the episode in which $(s, a)$ is visited the $i$-th time at step $h$ in epoch $\mathcal{E}$.  Let us define
$$
\begin{aligned}
&q_{h, 1}^{m,+}(s, a):=\alpha_{t}^{0} e^{\beta(H-h+1)}+ \begin{cases}\sum_{i \in[t]} \alpha_{t}^{i}\left[e^{\beta\left[r^{m_{i}^{\mathcal{E}} }_{h}(s, a)+V_{h+1}^{m_{i}^{\mathcal{E}} }\left(s_{h+1}^{m_{i}^{\mathcal{E}} }\right)\right]}+\Gamma_{h, i}\right], & \text { if } \beta>0, \\
\sum_{i \in[t]} \alpha_{t}^{i}\left[e^{\beta\left[r^{m_{i}^{\mathcal{E}} }_{h}(s, a)+V_{h+1}^{m_{i}^{\mathcal{E}} }\left(s_{h+1}^{m_{i}^{\mathcal{E}} }\right)\right]}-\Gamma_{h, i}\right], & \text { if } \beta<0,\end{cases} \\
&q_{h, 1}^{m}(s, a):= \begin{cases}\min \left\{q_{h, 1}^{m,+}(s, a), e^{\beta(H-h+1)}\right\}, & \text { if } \beta>0, \\
\max \left\{q_{h, 1}^{m,+}(s, a), e^{\beta(H-h+1)}\right\}, & \text { if } \beta<0,\end{cases}
\end{aligned}
$$
and
$$
\begin{aligned}
&q_{h, 2}^{m, \circ}(s, a):=\alpha_{t}^{0} e^{\beta(H-h+1)}+\sum_{i \in[t]} \alpha_{t}^{i}\left[e^{\beta\left[r^{m_{i}^{\mathcal{E}}}_{h}(s, a)+V_{h+1}^{*,m_{i}^{\mathcal{E}}}\left(s_{h+1}^{m_{i}^{\mathcal{E}}}\right)\right]}\right] \\
&q_{h, 2}^{m,+}(s, a):=\alpha_{t}^{0} e^{\beta(H-h+1)}+ \begin{cases}\sum_{i \in[t]} \alpha_{t}^{i}\left[e^{\beta\left[r^{m_{i}^{\mathcal{E}}}_{h}(s, a)+V_{h+1}^{*,m_{i}^{\mathcal{E}}}\left(s_{h+1}^{m_{i}^{\mathcal{E}}}\right)\right]}+\Gamma_{h, i}\right], & \text { if } \beta>0 \\
\sum_{i \in[t]} \alpha_{t}^{i}\left[e^{\beta\left[r^{m_{i}^{\mathcal{E}}}_{h}(s, a)+V_{h+1}^{*,m_{i}^{\mathcal{E}}}\left(s_{h+1}^{m_{i}^{\mathcal{E}}}\right)\right]}-\Gamma_{h, i}\right], & \text { if } \beta<0\end{cases} \\
&q_{h, 2}^{m}(s, a):= \begin{cases}\min \left\{q_{h, 2}^{m,+}(s, a), e^{\beta(H-h+1)}\right\}, & \text { if } \beta>0 \\
\max \left\{q_{h, 2}^{m,+}(s, a), e^{\beta(H-h+1)}\right\}, & \text { if } \beta<0\end{cases}
\end{aligned}
$$
and
$$
q_{h, 3}^{m}(s, a):=\alpha_{t}^{0} e^{\beta \cdot Q_{h}^{*,m}(s, a)}+\sum_{i \in[t]} \alpha_{t}^{i}\left[\mathbb{E}_{s^{\prime} \sim P^m_{h}(\cdot \mid s, a)} e^{\beta\left[r^m_{h}(s, a)+V_{h+1}^{*,m}\left(s^{\prime}\right)\right]}\right].
$$

By the definition of $q_{h, 2}^{m, \circ}$, $q_{h, 2}^{m, +}$ and $q_{h, 2}^{m}$, it can be seen that $q_{h, 2}^{m, \circ}\leq q_{h, 2}^{m}$ if $\beta >0$, and $q_{h, 2}^{m, \circ}\geq q_{h, 2}^{m}$ if $\beta <0$. In addition, by definition, we have $\left(e^{\beta \cdot Q_{h}^{m}}-e^{\beta \cdot Q_{h}^{*,m}}\right)(s, a)=\left(q_{h, 1}^{m}-q_{h, 3}^{m}\right)(s, a)$.

\subsection{Value difference bounds}
\begin{lemma} \label{lemma: V 1 star - V 2 star}
For every triple $(s,a,h)$ and episodes $m_1,m_2$ in the epoch $\mathcal{E}$, it holds that $\left|V_{h}^{*,m_1}(s)-V_{h}^{*,m_2}(s)\right| \leq B_{r,\mathcal{E}}+HB_{\mathcal{P},\mathcal{E}}$.
\end{lemma}
\begin{proof}
Let $a_1=\argmax_{a} Q_{h}^{*,m_1}(s,a)$ and $a_2=\argmax_{a} Q_{h}^{*,m_2}(s,a)$, it holds that
\begin{align*}
 V_{h}^{*,m_1}(s)= Q_{h}^{*,m_1}(s,a_1)\geq  Q_{h}^{*,m_1}(s,a_2)\geq & Q_{h}^{*,m_2}(s,a_2)-B_{r,\mathcal{E}}-HB_{\mathcal{P},\mathcal{E}}\\
 =&V_{h}^{*,m_2}(s)-B_{r,\mathcal{E}}-HB_{\mathcal{P},\mathcal{E}}
\end{align*}
where the second inequality follows from \cite[Lemma 1]{mao2020model}. Similarly, we have
\begin{align*}
V_{h}^{*,m_2}(s)\geq V_{h}^{*,m_1}(s)-B_{r,\mathcal{E}}-HB_{\mathcal{P},\mathcal{E}}.
\end{align*}
This completes the proof.
\end{proof}

\begin{lemma} \label{lemma: thm 2 concentration}
For every $(m, h, s, a) \in[M] \times[H] \times \mathcal{S} \times \mathcal{A}$ and $m_{1}, \ldots, m_{t}<m$ with $t=N_{h}^{m}(s, a)$, we have
$$
\begin{aligned}
&\left| \sum_{i \in[t]} \alpha_{t}^{i}\left[e^{\beta\left[r^{m_{i}^{\mathcal{E}}}_{h}(s, a)+V_{h+1}^{*,m_{i}^{\mathcal{E}}}\left(s_{h+1}^{m_{i}^{\mathcal{E}}}\right)\right]}-\mathbb{E}_{s^{\prime} \sim P^m_{h}(\cdot \mid s, a)}\left[e^{\beta\left[r^m_{h}(s, a)+V_{h+1}^{*,m}\left(s^{\prime}\right)\right]}\right] \right] \right| \\
\leq & \Gamma_{h, t} + 2g_h(\beta) B_{r,\mathcal{E}} +  \left( g_h(\beta)H + \left|e^{\beta(H-h+1) }-1 \right| \right) B_{\mathcal{P},\mathcal{E}} 
\end{aligned}
$$
with probability at least $1-\delta$, and
$$
\sum_{i \in[t]} \alpha_{t}^{i} \Gamma_{h, i} \in\left[\Gamma_{h, t}, 2 \Gamma_{h, t}\right],
$$
where $\Gamma_{h, t}$ is defined in \eqref{eq: bonus term thm 2}.
\end{lemma}

\begin{proof}
 For every $(m, h, s, a) \in$ $[M] \times[H] \times \mathcal{S} \times \mathcal{A}$, we have the following decomposition:
\begin{subequations} 
\begin{align}
\nonumber &e^{\beta\left[r^{m_{i}^{\mathcal{E}}}_{h}(s, a)+V_{h+1}^{*,m_{i}^{\mathcal{E}}}\left(s_{h+1}^{m_{i}^{\mathcal{E}}}\right)\right]}-\mathbb{E}_{s^{\prime} \sim P^m_{h}(\cdot \mid s, a)}\left[e^{\beta\left[r^m_{h}(s, a)+V_{h+1}^{*,m}\left(s^{\prime}\right)\right]}\right] \\
=& e^{\beta\left[r^{m_{i}^{\mathcal{E}}}_{h}(s, a)+V_{h+1}^{*,m_{i}^{\mathcal{E}}}\left(s_{h+1}^{m_{i}^{\mathcal{E}}}\right)\right]}-e^{\beta\left[r^{m}_{h}(s, a)+V_{h+1}^{*,m_{i}^{\mathcal{E}}}\left(s_{h+1}^{m_{i}^{\mathcal{E}}}\right)\right]} \label{eq: thm2 r drift} \\
&+e^{\beta\left[r^m_{h}(s, a)+V_{h+1}^{*,m_{i}^{\mathcal{E}}}\left(s_{h+1}^{m_{i}^{\mathcal{E}}}\right)\right]}-e^{\beta\left[r^m_{h}(s, a)+V_{h+1}^{*,m}\left(s_{h+1}^{m_{i}^{\mathcal{E}}}\right)\right]}\label{eq: thm2 V drift}  \\
&+e^{\beta\left[r^m_{h}(s, a)+V_{h+1}^{*,m}\left(s_{h+1}^{m_{i}^{\mathcal{E}}}\right)\right]} -\mathbb{E}_{s^{\prime} \sim P^{m_{i}^{\mathcal{E}}}_{h}(\cdot \mid s, a)}\left[e^{\beta\left[r^m_{h}(s, a)+V_{h+1}^{*,m}\left(s^{\prime}\right)\right]}\right] \label{eq: thm2 P drift 1}\\
&+\mathbb{E}_{s^{\prime} \sim P^{m_{i}^{\mathcal{E}}}_{h}(\cdot \mid s, a)}\left[e^{\beta\left[r^m_{h}(s, a)+V_{h+1}^{*,m}\left(s^{\prime}\right)\right]}\right] -\mathbb{E}_{s^{\prime} \sim P^m_{h}(\cdot \mid s, a)}\left[e^{\beta\left[r^m_{h}(s, a)+V_{h+1}^{*,m}\left(s^{\prime}\right)\right]}\right] \label{eq: thm2 P drift 2}.
\end{align}
\end{subequations}

For the terms in \eqref{eq: thm2 r drift}, it holds that 
\begin{align}
\nonumber \left|e^{\beta\left[r^{m_{i}^{\mathcal{E}}}_{h}(s, a)+V_{h+1}^{*,m_{i}^{\mathcal{E}}}\left(s_{h+1}^{m_{i}^{\mathcal{E}}}\right)\right]}-e^{\beta\left[r^{m}_{h}(s, a)+V_{h+1}^{*,m_{i}^{\mathcal{E}}}\left(s_{h+1}^{m_{i}^{\mathcal{E}}}\right)\right]} \right|
\leq &  g_h(\beta)  \left|{r}_{h}^{m_{i}^{\mathcal{E}}}(s, a)-{r}_{h}^{m}(s, a) \right|\\
\leq &  g_h(\beta)  {B}_{r,\epsilon}, \label{eq: decompose 1}
\end{align}
where the first inequality follows from the Lipschitz continuity of $e^{\beta x}$ with respect to $x$ and the second inequality is due to the definition of the local variation budget ${B}_{r,\epsilon}$.

For the terms in \eqref{eq: thm2 V drift}, it holds that 
\begin{align}
\nonumber \left| e^{\beta\left[r^m_{h}(s, a)+V_{h+1}^{*,m_{i}^{\mathcal{E}}}\left(s_{h+1}^{m_{i}^{\mathcal{E}}}\right)\right]}-e^{\beta\left[r^m_{h}(s, a)+V_{h+1}^{*,m}\left(s_{h+1}^{m_{i}^{\mathcal{E}}}\right)\right]}\right|
\leq &  g_h(\beta)  \left|V_{h+1}^{*,m_{i}^{\mathcal{E}}}\left(s_{h+1}^{m_{i}^{\mathcal{E}}}\right)-V_{h+1}^{*,m}(s_{h+1}^{m_{i}^{\mathcal{E}}}) \right|\\
\leq& g_h(\beta) \left( B_{r,\mathcal{E}}+HB_{\mathcal{P},\mathcal{E}} \right)  \label{eq: decompose 2}
\end{align}
where the second inequality follows from Lemma \ref{lemma: V 1 star - V 2 star}.

For the terms in \eqref{eq: thm2 P drift 2}, we have
\begin{align}
\nonumber & \left| \mathbb{E}_{s^{\prime} \sim P^{m_{i}^{\mathcal{E}}}_{h}(\cdot \mid s, a)}\left[e^{\beta\left[r^m_{h}(s, a)+V_{h+1}^{*,m}\left(s^{\prime}\right)\right]}\right] -\mathbb{E}_{s^{\prime} \sim P^m_{h}(\cdot \mid s, a)}\left[e^{\beta\left[r^m_{h}(s, a)+V_{h+1}^{*,m}\left(s^{\prime}\right)\right]}\right]\right|\\
\nonumber &=\left|\mathbb{E}_{s^{\prime} \sim P^{m_{i}^{\mathcal{E}}}_{h}(\cdot \mid s, a)}\left[e^{\beta\left[r^m_{h}(s, a)+V_{h+1}^{*,m}\left(s^{\prime}\right)\right]}-1\right] -\mathbb{E}_{s^{\prime} \sim P^m_{h}(\cdot \mid s, a)}\left[e^{\beta\left[r^m_{h}(s, a)+V_{h+1}^{*,m}\left(s^{\prime}\right)\right]}-1\right]\right|\\
&\leq \left|e^{\beta(H-h+1) }-1 \right| B_{\mathcal{P},\mathcal{E}}   \label{eq: decompose 3}
\end{align}
where the first step follows from $\mathcal{P}_{h}^{m} {1}(s,a)=\mathcal{P}_{h}^{\tau} {1}(s,a)$ for all $\tau \in[\ell^m, m-1]$ and the last step holds by the definition of $B_{\mathcal{P},\mathcal{E}}$.

We now analyze the terms in \eqref{eq: thm2 P drift 1}. For every $(m, h, s, a) \in[M] \times[H] \times \mathcal{S} \times \mathcal{A}$, we define
\begin{align*}
\nonumber \psi(i, m, h, s, a) \coloneqq e^{\beta\left[r^m_{h}(s, a)+V_{h+1}^{*,m}\left(s_{h+1}^{m_{i}^{\mathcal{E}}}\right)\right]} -\mathbb{E}_{s^{\prime} \sim P^{m_{i}^{\mathcal{E}}}_{h}(\cdot \mid s, a)}\left[e^{\beta\left[r^m_{h}(s, a)+V_{h+1}^{*,m}\left(s^{\prime}\right)\right]}\right].
\end{align*}

For a fix tuple $(m,h,s,a)\in[M] \times[H] \times \mathcal{S} \times \mathcal{A}$,  $\left\{\psi(i, m, h, s, a)\right\}_{i \in[t]}$ with $t=N_h^m(s,a)$ is a martingale difference sequence. By the Azuma-Hoeffding inequality,  with probability at least $1-\delta /(HM |\mathcal{S}| |\mathcal{A}|)$, it holds that
$$
\begin{aligned}
&\left|\sum_{i \in[t]} \alpha_{t}^{i}   \cdot \psi(i, m, h, s, a)\right| \leq \frac{C_2}{2}\left|e^{\beta(H-h+1)}-1\right| \sqrt{\iota \sum_{i \in[t]}\left(\alpha_{t}^{i}\right)^{2}} \leq C_2\left|e^{\beta(H-h+1)}-1\right| \sqrt{\frac{H \iota}{t}}
\end{aligned}
$$
where $C_2>0$ is some universal constant, the first step holds since $r_{h}(s, a)+V_{h+1}^{*}\left(s^{\prime}\right) \in[0, H-h+1]$ for $s^{\prime} \in \mathcal{S}$, and the last step follows from the second property in Lemma \ref{lemma: alpha t property}. Then, applying the union bound over $(m, h, s, a) \in[M] \times [H] \times \mathcal{S} \times \mathcal{A}$, we have that the following holds for all $(m, h, s, a) \in[M] \times[H] \times \mathcal{S} \times \mathcal{A}$ with probability at least $1-\delta$ :
\begin{align}
\left|\sum_{i \in[t]} \alpha_{t}^{i} \cdot \psi(i, m, h, s, a)\right| \leq C_2 \left|e^{\beta(H-h+1)}-1\right| \sqrt{\frac{H \iota}{t}},  \label{eq: decompose 4}
\end{align}

where $t=N_{h}^{m}(s, a)$. 

Finally, by combining Equations \eqref{eq: decompose 1}-\eqref{eq: decompose 4} and noticing that $\sum_{i\in[t]} \alpha_t^i=1$ from the forth property in Lemma \ref{lemma: alpha t property}, we have 
$$
\begin{aligned}
&\left| \sum_{i \in[t]} \alpha_{t}^{i}\left[e^{\beta\left[r^{m_{i}^{\mathcal{E}}}_{h}(s, a)+V_{h+1}^{*,m_{i}^{\mathcal{E}}}\left(s_{h+1}^{m_{i}^{\mathcal{E}}}\right)\right]}-\mathbb{E}_{s^{\prime} \sim P^m_{h}(\cdot \mid s, a)}\left[e^{\beta\left[r^m_{h}(s, a)+V_{h+1}^{*,m}\left(s^{\prime}\right)\right]}\right] \right] \right| \\
\leq & C_2\left|e^{\beta(H-h+1)}-1\right| \sqrt{\frac{H \iota}{t}} + 2g_h(\beta) B_{r,\mathcal{E}} +  \left( g_h(\beta)H + \left|e^{\beta(H-h+1) }-1 \right| \right) B_{\mathcal{P},\mathcal{E}} 
\end{aligned}
$$
For bounds on $\sum_{i \in[t]} \alpha_{t}^{i} \Gamma_{h, i}$, we recall the definition of $\left\{\Gamma_{h, t}\right\}$ in \eqref{eq: bonus term thm 2} and compute
$$
\begin{aligned}
\sum_{i \in[t]} \alpha_{t}^{i} \Gamma_{h, i} &=C_2 \left|e^{\beta(H-h+1)}-1\right| \sum_{i \in[t]} \alpha_{t}^{i} \sqrt{\frac{H \iota}{i}} \\
& \in \left[C_2\left|e^{\beta(H-h+1)}-1\right| \sqrt{\frac{H \iota}{t}}, 2 C_2\left|e^{\beta(H-h+1)}-1\right| \sqrt{\frac{H \iota}{t}}\right]
\end{aligned}
$$
where the last step holds by the first property in Lemma \ref{lemma: alpha t property}.
\end{proof}

\begin{lemma} \label{lemma: thm 2 q2-q3}
For all $(m, h, s, a)$ and $\delta \in(0,1]$, the following statements hold with probability at least $1-\delta$: 
\begin{itemize}[leftmargin=*]
\item If $\beta>0$:
\begin{align*}
&-2 e^{\beta(H-h+1)} \beta B_{r,\mathcal{E}} - \left(e^{\beta(H-h+1)} \beta H + \left(e^{\beta(H-h+1) }-1 \right) \right) B_{\mathcal{P},\mathcal{E}}  \leq q_{h,2}^m(s,a) -q_{h,3}^m(s,a)\\ &\leq  \alpha_t^0 \left( e^{\beta(H-h+1)}-1\right) +2\sum_{i\in[t]} \alpha_t^i \Gamma_{h,i} +  2e^{\beta(H-h+1)} \beta B_{r,\mathcal{E}} +\left( e^{\beta(H-h+1)} \beta H + \left(e^{\beta(H-h+1) }-1 \right) \right) B_{\mathcal{P},\mathcal{E}}.
\end{align*}
\item  If $\beta<0$:
\begin{align*}
& 2\beta B_{r,\mathcal{E}} - \left(-\beta H + \left(1-e^{\beta(H-h+1) } \right) \right) B_{\mathcal{P},\mathcal{E}}  \leq q_{h,3}^m(s,a) -q_{h,2}^m(s,a)\\ &\leq  \alpha_t^0 \left(1-e^{\beta(H-h+1)}\right) +2\sum_{i\in[t]} \alpha_t^i \Gamma_{h,i} -  2  \beta B_{r,\mathcal{E}} +\left( -\beta H + \left(1 - e^{\beta(H-h+1) } \right) \right) B_{\mathcal{P},\mathcal{E}}.
\end{align*}
\end{itemize}
\end{lemma}
\begin{proof}
We focus on the case where $\beta>0$ and the case for $\beta<0$ can be proved similarly. 
By the definition of $q_{h,2}^{m,+}$ and  $q_{h,3}^{m}$, it holds that
\begin{align*}
    q_{h,2}^{m,+}-q_{h,3}^{m}=&\alpha_t^0 \left( e^{\beta(H-h+1)}-e^{\beta Q_h^{*,m} (s,a)}\right)\\
    &+\sum_{i \in [t]}\alpha_t^i  \left[e^{\beta\left[r^{m_{i}^{\mathcal{E}}}_{h}(s, a)+V_{h+1}^{*,m_{i}^{\mathcal{E}}}\left(s_{h+1}^{m_{i}^{\mathcal{E}}}\right)\right]}+\Gamma_{h, i} -\mathbb{E}_{s^{\prime} \sim P^m_{h}(\cdot \mid s, a)} e^{\beta\left[r^m_{h}(s, a)+V_{h+1}^{*,m}\left(s^{\prime}\right)\right]} \right].
\end{align*}
Due to $e^{\beta(H-h+1)}\geq e^{\beta Q_h^{*,m} (s,a)}\geq 1$ and Lemma \ref{lemma: thm 2 concentration}, we have
\begin{align*}
    q_{h,2}^{m,+}-q_{h,3}^{m}\geq &- 2g_h(\beta) B_{r,\mathcal{E}} - \left( g_h(\beta)H + \left|e^{\beta(H-h+1) }-1 \right| \right) B_{\mathcal{P},\mathcal{E}} 
\end{align*}
and
\begin{align*}
    q_{h,2}^{m,+}-q_{h,3}^{m}\leq & \alpha_t^0 \left( e^{\beta(H-h+1)}-1\right) +2\sum_{i\in[t]} \alpha_t^i \Gamma_{h,i} \\
    &+  2g_h(\beta) B_{r,\mathcal{E}} +\left( g_h(\beta)H + \left|e^{\beta(H-h+1) }-1 \right| \right) B_{\mathcal{P},\mathcal{E}}.
\end{align*}
Furthermore, if $q_{h,2}^{m,+}\leq e^{\beta (H-h+1)}$, then we have $q_{h,2}^{m}=q_{h,2}^{m,+}$. On the other hand, if $q_{h,2}^{m,+}\geq e^{\beta (H-h+1)}$, then $q_{h,2}^{m}= e^{\beta (H-h+1)} \leq q_{h,2}^{m,+}$. Thus, it holds that $0 \leq q_{h,2}^{m} - q_{h,3}^{m} \leq  q_{h,2}^{m,+} - q_{h,3}^{m}$. This completes the proof.
\end{proof}

The next two lemmas compare the iterate $e^{\beta \cdot Q_{h}^{m}}$ (and $e^{\beta \cdot V_{h}^{m}}$ ) with the optimal exponential value function $e^{\beta \cdot Q_{h}^{*,m}}$ (and $e^{\beta \cdot V_{h}^{*,m}})$.

\begin{lemma}\label{lemma: Qm-Qstar lower bound}
For all $(m, h, s, a)$ and  $\delta \in(0,1]$, it holds with probability at least $1-\delta$: 
\begin{itemize}
    \item If $\beta>0$:
\begin{align*}
&(e^{\beta Q_{h}^m } - e^{\beta Q_{h}^{*,m} }) (s,a)\\
\geq& -(H-h+1) \left(2 e^{\beta(H-h+1)} \beta B_{r,\mathcal{E}} + \left(e^{\beta(H-h+1)} \beta H + \left(e^{\beta(H-h+1) }-1 \right) \right) B_{\mathcal{P},\mathcal{E}} \right).
\end{align*}
\item If $\beta<0$:
\begin{align*}
&(e^{\beta Q_{h}^m } - e^{\beta Q_{h}^{*,m} }) (s,a)\leq (H-h+1) \left(-2 \beta B_{r,\mathcal{E}} + \left(- \beta H + \left(1-e^{\beta(H-h+1) } \right) \right) B_{\mathcal{P},\mathcal{E}} \right).  
\end{align*}
\end{itemize}
\end{lemma}

\begin{proof}
We focus only on the case where $\beta>0$ since the proof for $\beta<0$ is similar. 
For the purpose of the proof, we set $Q_{H+1}^{m}(s, a)=Q_{H+1}^{*}(s, a)=0$ for all $(m, s, a) \in[M] \times \mathcal{S} \times \mathcal{A}$. We fix a pair $(s, a) \in \mathcal{S} \times \mathcal{A}$ and use strong induction on $m$ and $h$. Without loss of generality, we assume that there exists a pair $(m, h)$ such that $(s, a)=\left(s_{h}^{m}, a_{h}^{m}\right)$ (that is, $(s, a)$ has been visited at some point in Algorithm \ref{alg:algoirthm 2}), since otherwise $e^{\beta \cdot Q_{h}^{m}(s, a)}=e^{\beta(H-h+1)} \geq e^{\beta \cdot Q_{h}^{*}(s, a)}$ for all $(m, h) \in[M] \times[H]$ and we are done.

The base case for $m=1$ and $h=H+1$ is satisfied since $e^{\beta \cdot Q_{H+1}^{m^{\prime}}(s, a)}=e^{\beta \cdot Q_{H+1}^{*,m}(s, a)}$ for $m^{\prime} \in[M]$ by definition. We fix a pair $(m, h) \in[M] \times[H]$ and assume that 
\begin{align*}
e^{\beta \cdot Q_{h+1}^{m_{i}^{\mathcal{E}}}(s, a)}-e^{\beta \cdot Q_{h+1}^{*,m_{i}^{\mathcal{E}}}(s, a)} \geq -(H-h)\left(2 e^{\beta(H-h)} \beta B_{r,\mathcal{E}} + \left(e^{\beta(H-h)} \beta H + \left(e^{\beta(H-h) }-1 \right) \right) B_{\mathcal{P},\mathcal{E}} \right)
\end{align*}
for each $m_{1}^{\mathcal{E}}, \ldots, m_{t}^{\mathcal{E}}$ (here $t=N_{h}^{m}(s, a)$ ). We have for $i \in[t]$ that
\begin{align}
\nonumber e^{\beta \cdot V_{h+1}^{m_{i}^{\mathcal{E}}}(s)}=&\max _{a^{\prime} \in \mathcal{A}} e^{\beta \cdot Q_{h+1}^{m_{i}^{\mathcal{E}}}\left(s, a^{\prime}\right)} -(H-h)\left(2 e^{\beta(H-h)} \beta B_{r,\mathcal{E}} + \left(e^{\beta(H-h)} \beta H + \left(e^{\beta(H-h) }-1 \right) \right) B_{\mathcal{P},\mathcal{E}} \right)\\
\nonumber \geq & \max _{a^{\prime} \in \mathcal{A}} e^{\beta \cdot Q_{h+1}^{*,m_{i}^{\mathcal{E}}}\left(s, a^{\prime}\right)}-(H-h)\left(2 e^{\beta(H-h)} \beta B_{r,\mathcal{E}} + \left(e^{\beta(H-h)} \beta H + \left(e^{\beta(H-h) }-1 \right) \right) B_{\mathcal{P},\mathcal{E}} \right)\\
=&e^{\beta \cdot V_{h+1}^{*,m_{i}^{\mathcal{E}}}(s)}-(H-h)\left(2 e^{\beta(H-h)} \beta B_{r,\mathcal{E}} + \left(e^{\beta(H-h)} \beta H + \left(e^{\beta(H-h) }-1 \right) \right) B_{\mathcal{P},\mathcal{E}} \right) \label{eq: exp Vm -exp V*}
\end{align}
where the first equality holds by the update procedure in Algorithm \ref{alg:algoirthm 2}. Then, it holds that
\begin{align*}
(q_{h,1}^{m,+}-q_{h,2}^{m})(s,a) &\geq (q_{h,1}^{m,+}-q_{h,2}^{m,+})(s,a)\\
& \geq  \sum_{i \in[t]} \alpha_{t}^{i}\left[e^{\beta\left[r^{m_{i}^{\mathcal{E}}}_{h}(s, a)+V_{h+1}^{m_{i}^{\mathcal{E}}}\left(s_{h+1}^{m_{i}^{\mathcal{E}}}\right)\right]}-e^{\beta\left[r^{m_{i}^{\mathcal{E}}}_{h}(s, a)+V_{h+1}^{*,m_{i}^{\mathcal{E}}}\left(s_{h+1}^{m_{i}^{\mathcal{E}}}\right)\right]}\right] \\
&= \sum_{i \in[t]} \alpha_{t}^{i} e^{\beta r^{m_{i}^{\mathcal{E}}}_{h}(s, a)}\left[e^{\beta\left[V_{h+1}^{m_{i}^{\mathcal{E}}}\left(s_{h+1}^{m_{i}^{\mathcal{E}}}\right)\right]}-e^{\beta\left[V_{h+1}^{*,m_{i}^{\mathcal{E}}}\left(s_{h+1}^{m_{i}^{\mathcal{E}}}\right)\right]}\right] \\
& \geq -(H-h)\sum_{i \in[t]} \alpha_{t}^{i} e^{\beta}\left(2 e^{\beta(H-h)} \beta B_{r,\mathcal{E}} + \left(e^{\beta(H-h)} \beta H + \left(e^{\beta(H-h) }-1 \right) \right) B_{\mathcal{P},\mathcal{E}} \right)\\
& \geq -(H-h)\sum_{i \in[t]} \alpha_{t}^{i} \left(2 e^{\beta(H-h+1)} \beta B_{r,\mathcal{E}} + \left(e^{\beta(H-h+1)} \beta H + \left(e^{\beta(H-h+1) }-1 \right) \right) B_{\mathcal{P},\mathcal{E}} \right)\\
& \geq -(H-h) \left(2 e^{\beta(H-h+1)} \beta B_{r,\mathcal{E}} + \left(e^{\beta(H-h+1)} \beta H + \left(e^{\beta(H-h+1) }-1 \right) \right) B_{\mathcal{P},\mathcal{E}} \right)
\end{align*}
where the first inequality follows from the definitions of $q_{h,1}^{m,+}$, $q_{h,2}^{m,+}$, the second inequality holds by the induction hypothesis, the third inequality follows from $e^{\beta}>1$ for $\beta>0$, and the last inequality holds by $\sum_{i \in[t]} \alpha_{t}^{i} \leq 1$ from Lemma \ref{lemma: alpha t property}. Furthermore, when $q_{h,1}^{m}=e^{\beta(H-h+1)}\leq q_{h,1}^{m,+}$, we have $q_{h,1}^{m}-q_{h,2}^{m} \geq 0$ since $q_{h,2}^{m}  \leq e^{\beta(H-h+1)}$ by definition. Thus, we can conclude that 
\begin{align} \label{eq: thm 2 q1-q2}
(q_{h,1}^{m}-q_{h,2}^{m})(s,a) &\geq -(H-h) \left(2 e^{\beta(H-h+1)} \beta B_{r,\mathcal{E}} + \left(e^{\beta(H-h+1)} \beta H + \left(e^{\beta(H-h+1) }-1 \right) \right) B_{\mathcal{P},\mathcal{E}} \right)
\end{align}
In addition, from Lemma \ref{lemma: thm 2 q2-q3} , we also have
\begin{align}\label{eq: thm 2 q2-q3}
(q_{h,2}^m-q_{h,3}^m)(s,a)\geq -2 e^{\beta(H-h+1)} \beta B_{r,\mathcal{E}} - \left(e^{\beta(H-h+1)} \beta H + \left(e^{\beta(H-h+1) }-1 \right) \right) B_{\mathcal{P},\mathcal{E}}    
\end{align}
Finally, by combining \eqref{eq: thm 2 q1-q2} and \eqref{eq: thm 2 q2-q3}, we obtain
\begin{align*}
&(e^{\beta Q_{h}^m } - e^{\beta Q_{h}^{*,m} }) (s,a)\\
=&(q_{h,1}^m -q_{h,2}^m)(s,a)+(q_{h,2}^m-q_{h,3}^m)(s,a)\\
\geq& -(H-h+1) \left(2 e^{\beta(H-h+1)} \beta B_{r,\mathcal{E}} + \left(e^{\beta(H-h+1)} \beta H + \left(e^{\beta(H-h+1) }-1 \right) \right) B_{\mathcal{P},\mathcal{E}} \right).
\end{align*}
The induction is completed.
\end{proof}

\begin{lemma}\label{lemma: Qm-Qstar upper bound}
For all $(m, h, s, a)$ and $\delta \in(0,1]$, it holds with probability at least $1-\delta$: 
\begin{itemize}
    \item  If $\beta>0$:
\begin{align*}
&(e^{\beta Q_{h}^m } - e^{\beta Q_{h}^{*,m} }) (s,a)\\
\leq& \sum_{i \in[t]} \alpha_{t}^{i}  e^{\beta r^{m_{i}^{\mathcal{E}}}_{h}(s, a)} \left[e^{\beta\left[V_{h+1}^{m_{i}^{\mathcal{E}}}\left(s_{h+1}^{m_{i}^{\mathcal{E}}}\right)\right]}-e^{\beta\left[V_{h+1}^{*,m_{i}^{\mathcal{E}}}\left(s_{h+1}^{m_{i}^{\mathcal{E}}}\right)\right]}\right] + 3\sum_{i \in[t]} \alpha_{t}^{i} \Gamma_{h, i}\\
&+\alpha_t^0 \left( e^{\beta(H-h+1)}-1\right) + 2e^{\beta(H-h+1)} \beta B_{r,\mathcal{E}} +\left( e^{\beta(H-h+1)} \beta H + \left(e^{\beta(H-h+1) }-1 \right) \right) B_{\mathcal{P},\mathcal{E}}.
\end{align*}
\item If $\beta<0$:
\begin{align*}
&(e^{\beta Q_{h}^m } - e^{\beta Q_{h}^{*,m} }) (s,a)\\
\geq & \sum_{i \in[t]} \alpha_{t}^{i}  e^{\beta r^{m_{i}^{\mathcal{E}}}_{h}(s, a)} \left[e^{\beta\left[V_{h+1}^{*,m_{i}^{\mathcal{E}}}\left(s_{h+1}^{m_{i}^{\mathcal{E}}}\right)\right]}-e^{\beta\left[V_{h+1}^{m_{i}^{\mathcal{E}}}\left(s_{h+1}^{m_{i}^{\mathcal{E}}}\right)\right]}\right] - 3\sum_{i \in[t]} \alpha_{t}^{i} \Gamma_{h, i}\\
&-\alpha_t^0 \left(1- e^{\beta(H-h+1)}\right) + 2 \beta B_{r,\mathcal{E}} -\left(  -\beta H + \left(1-e^{\beta(H-h+1) } \right) \right) B_{\mathcal{P},\mathcal{E}}.
\end{align*}
\end{itemize}
\end{lemma}

\begin{proof}
We focus on the case where $\beta>0$ since the case for $\beta<0$ can be proved similarly. 
By the definition of $q_{h,1}^m$ and $q_{h,2}^m$, we have
\begin{align*}
    \left(q_{h,1}^m-q_{h,2}^m\right)(s,a)\leq& \left(q_{h,1}^{m,+}-q_{h,2}^{m,\circ}\right)(s,a)\\
    \leq & \sum_{i \in[t]} \alpha_{t}^{i}\left[e^{\beta\left[r^{m_{i}^{\mathcal{E}}}_{h}(s, a)+V_{h+1}^{m_{i}^{\mathcal{E}}}\left(s_{h+1}^{m_{i}^{\mathcal{E}}}\right)\right]}-e^{\beta\left[r^{m_{i}^{\mathcal{E}}}_{h}(s, a)+V_{h+1}^{*,m_{i}^{\mathcal{E}}}\left(s_{h+1}^{m_{i}^{\mathcal{E}}}\right)\right]}\right] + \sum_{i \in[t]} \alpha_{t}^{i} \Gamma_{h, i}\\
    = & \sum_{i \in[t]} \alpha_{t}^{i}  e^{\beta r^{m_{i}^{\mathcal{E}}}_{h}(s, a)} \left[e^{\beta\left[V_{h+1}^{m_{i}^{\mathcal{E}}}\left(s_{h+1}^{m_{i}^{\mathcal{E}}}\right)\right]}-e^{\beta\left[V_{h+1}^{*,m_{i}^{\mathcal{E}}}\left(s_{h+1}^{m_{i}^{\mathcal{E}}}\right)\right]}\right] + \sum_{i \in[t]} \alpha_{t}^{i} \Gamma_{h, i}
\end{align*}
where the first inequality follows from $q_{h,1}^m\leq q_{h,1}^{m,+}$ and $q_{h,2}^m\geq q_{h,2}^{m,\circ}$, and the second inequality holds by the definition of $q_{h,1}^{m,+}$ and $q_{h,2}^{m,\circ}$. Then, by Lemma \ref{lemma: thm 2 q2-q3}, we obtain
\begin{align*}
&(e^{\beta Q_{h}^m } - e^{\beta Q_{h}^{*,m} }) (s,a)\\
=&(q_{h,1}^m -q_{h,2}^m)(s,a)+(q_{h,2}^m-q_{h,3}^m)(s,a)\\
\leq & \sum_{i \in[t]} \alpha_{t}^{i}  e^{\beta r^{m_{i}^{\mathcal{E}}}_{h}(s, a)} \left[e^{\beta\left[V_{h+1}^{m_{i}^{\mathcal{E}}}\left(s_{h+1}^{m_{i}^{\mathcal{E}}}\right)\right]}-e^{\beta\left[V_{h+1}^{*,m_{i}^{\mathcal{E}}}\left(s_{h+1}^{m_{i}^{\mathcal{E}}}\right)\right]}\right] + 3\sum_{i \in[t]} \alpha_{t}^{i} \Gamma_{h, i}\\
&+\alpha_t^0 \left( e^{\beta(H-h+1)}-1\right) + 2e^{\beta(H-h+1)} \beta B_{r,\mathcal{E}} +\left( e^{\beta(H-h+1)} \beta H + \left(e^{\beta(H-h+1) }-1 \right) \right) B_{\mathcal{P},\mathcal{E}}.
\end{align*}
This completes the proof.
\end{proof}

\subsection{Proof of Theorem \ref{thm: for alg 2}}
For now, we consider the case for $\beta>0$. We define the following quantities to ease the notations for the proof:
$$
\begin{aligned}
\delta_{h}^{m} &:=e^{\beta \cdot V_{h}^{m}\left(s_{h}^{m}\right)}-e^{\beta \cdot V_{h}^{\pi^m}\left(s_{h}^{m}\right)}, \\
\phi_{h}^{m} &:=e^{\beta \cdot V_{h}^{m}\left(s_{h}^{m}\right)}-e^{\beta \cdot V_{h}^{*,m}\left(s_{h}^{m}\right)}, \\
\xi_{h+1}^{m} &:=\left[\left(\mathcal{P}^m_{h}-\widehat{\mathcal{P}}_{h}^{m}\right)\left(e^{\beta \cdot V_{h+1}^{*,m}}-e^{\beta \cdot V_{h+1}^{\pi^{m}}}\right)\right]\left(s_{h}^{m}, a_{h}^{m}\right)
\end{aligned}
$$
For each fixed $(m, h) \in[M] \times[H]$, we let $t=N_{h}^{m}\left(s_{h}^{m}, a_{h}^{m}\right)$. Then, it holds that
\begin{align}
\nonumber \delta_{h}^{m} \stackrel{(i)}{=}& e^{\beta \cdot Q_{h}^{m}\left(s_{h}^{m}, a_{h}^{m}\right)}-e^{\beta \cdot Q_{h}^{\pi^m,m}\left(s_{h}^{m}, a_{h}^{m}\right)}\\
\nonumber =&\left[e^{\beta \cdot Q_{h}^{m}\left(s_{h}^{m}, a_{h}^{m}\right)}-e^{\beta \cdot Q_{h}^{*,m}\left(s_{h}^{m}, a_{h}^{m}\right)}\right]+\left[e^{\beta \cdot Q_{h}^{*,m}\left(s_{h}^{m}, a_{h}^{m}\right)}-e^{\beta \cdot Q_{h}^{\pi^{m}}\left(s_{h}^{m}, a_{h}^{m}\right)}\right]\\
\nonumber \stackrel{(i i)}{=}&\left[e^{\beta \cdot Q_{h}^{m}\left(s_{h}^{m}, a_{h}^{m}\right)}-e^{\beta \cdot Q_{h}^{*,m}\left(s_{h}^{m}, a_{h}^{m}\right)}\right]+e^{\beta \cdot r^m_{h}\left(s_{h}^{m}, a_{h}^{m}\right)}\left[\mathcal{P}^m_{h}\left(e^{\beta \cdot V_{h+1}^{*,m}}-e^{\beta \cdot V_{h+1}^{\pi^{m},m}}\right)\right]\left(s_{h}^{m}, a_{h}^{m}\right)\\
\nonumber \stackrel{(iii)}{\leq}& \left[e^{\beta \cdot Q_{h}^{m}\left(s_{h}^{m}, a_{h}^{m}\right)}-e^{\beta \cdot Q_{h}^{*,m}\left(s_{h}^{m}, a_{h}^{m}\right)}\right]+e^{\beta}\left[\mathcal{P}^m_{h}\left(e^{\beta \cdot V_{h+1}^{*,m}}-e^{\left.\beta \cdot V_{h+1}^{\pi^{m},m}\right)}\right)\left(s_{h}^{m}, a_{h}^{m}\right)\right. \\
\nonumber =&\left[e^{\beta \cdot Q_{h}^{m}\left(s_{h}^{m}, a_{h}^{m}\right)}-e^{\beta \cdot Q_{h}^{*,m}\left(s_{h}^{m}, a_{h}^{m}\right)}\right]+e^{\beta}\left(\delta_{h+1}^{m}-\phi_{h+1}^{m}+\xi_{h+1}^{m}\right)\\
\nonumber \stackrel{(iv)}{\leq}& \alpha_{t}^{0}\left(e^{\beta(H-h+1)}-1\right)+3\sum_{i \in[t]} \alpha_{t}^{i} \Gamma_{h, i}+\sum_{i \in[t]} \alpha_{t}^{i} \cdot e^{\beta \cdot r^{m_{i}^{\mathcal{E}}}_{h}\left(s_{h}^{m}, a_{h}^{m}\right)} \left[e^{\beta \cdot V_{h+1}^{m_{i}^{\mathcal{E}}}\left(s_{h+1}^{m_{i}^{\mathcal{E}}}\right)}-e^{\beta \cdot V_{h+1}^{*}\left(s_{h+1}^{m_{i}^{\mathcal{E}}}\right)}\right]\\
\nonumber &+ 2e^{\beta(H-h+1)} \beta B_{r,\mathcal{E}} +\left( e^{\beta(H-h+1)} \beta H + \left(e^{\beta(H-h+1) }-1 \right) \right) B_{\mathcal{P},\mathcal{E}}\\
\nonumber &+e^{\beta}\left(\delta_{h+1}^{m}-\phi_{h+1}^{m}+\xi_{h+1}^{m}\right)\\
=&\alpha_{t}^{0}\left(e^{\beta(H-h+1)}-1\right)+\sum_{i \in[t]} \alpha_{t}^{i} \cdot e^{\beta \cdot r^{m_{i}^{\mathcal{E}}}_{h}\left(s_{h}^{m}, a_{h}^{m}\right)} \phi_{h+1}^{m_{i}^{\mathcal{E}}} + e^{\beta} \left(\delta_{h+1}^{m}-\phi_{h+1}^{m}+\xi_{h+1}^{m}\right) \label{eq: thm 2 first row in decompose}\\
&+3\sum_{i \in[t]} \alpha_{t}^{i} \Gamma_{h, i}+ 2e^{\beta(H-h+1)} \beta B_{r,\mathcal{E}} +\left( e^{\beta(H-h+1)} \beta H + \left(e^{\beta(H-h+1) }-1 \right) \right) B_{\mathcal{P},\mathcal{E}}\label{eq: thm 2 second row in decompose}
\end{align}
where step $(i)$ holds since $V_{h}^{m}\left(s_{h}^{m}\right)=\max _{a^{\prime} \in \mathcal{A}} Q_{h}^{m}\left(s_{h}^{m}, a^{\prime}\right)=Q_{h}^{m}\left(s_{h}^{m}, a_{h}^{m}\right)$ and $V_{h}^{\pi^{m},m}\left(s_{h}^{m}\right)=Q_{h}^{\pi^{m},m}\left(s_{h}^{m}, \pi_{h}^{m}\left(s_{h}^{m}\right)\right)=$ $Q_{h}^{\pi^{m},m}\left(s_{h}^{m}, a_{h}^{m}\right)$; step (ii) holds by the exponential Bellman equation \eqref{eq: exp bellman equation}; step (iii) holds since $V_{h+1}^{*,m} \geq$ $V_{h+1}^{\pi^{m},m}$ implies $e^{\beta \cdot V_{h+1}^{*,m}} \geq e^{\beta \cdot V_{h+1}^{\pi^{m},m}}$ given that $\beta>0$;
step (iv) holds on the event of Lemma \ref{lemma: Qm-Qstar upper bound}.

We bound each term in \eqref{eq: thm 2 first row in decompose} and \eqref{eq: thm 2 second row in decompose} one by one. First, we have
$$
\begin{aligned}
\sum_{m \in [M]} \alpha_{n_{h}^{m}}^{0}\left(e^{\beta(H-h+1)}-1\right) &=\left(e^{\beta(H-h+1)}-1\right) \sum_{m \in [M]} \mathbb{1}\left\{n_{h}^{m}=0\right\} \\
& \leq\left(e^{\beta(H-h+1)}-1\right) |\mathcal{S}| |\mathcal{A}| .
\end{aligned}
$$
To bound the second term in \eqref{eq: thm 2 first row in decompose}, we first define 
\begin{align*}
\hat{\phi}_{h+1}^{m_{i}^{\mathcal{E}}\left(s_{h}^{m}, a_{h}^{m}\right)}  \coloneqq {\phi}_{h+1}^{m_{i}^{\mathcal{E}}\left(s_{h}^{m}, a_{h}^{m}\right)}  +(H-h)\left(2 e^{\beta(H-h)} \beta B_{r,\mathcal{E}} + \left(e^{\beta(H-h)} \beta H + \left(e^{\beta(H-h) }-1 \right) \right) B_{\mathcal{P},\mathcal{E}} \right)
\end{align*}
which is non-negative from Lemma \ref{lemma: Qm-Qstar lower bound} and \eqref{eq: exp Vm -exp V*}:
\begin{align*}
\sum_{m \in [M]}\left(\sum_{i \in[t]} \alpha_{t}^{i} \cdot e^{\beta  \cdot r^{m_{i}^{\mathcal{E}}}_{h}\left(s_{h}^{m}, a_{h}^{m}\right)} \phi_{h+1}^{m_{i}^{\mathcal{E}}}\right)=&\sum_{m \in [M]}\left(\sum_{i \in\left[n_{h}^{m}\right]} \alpha_{n_{h}^{m}}^{i} \cdot e^{\beta  \cdot r^{m_{i}^{\mathcal{E}}}_{h}\left(s_{h}^{m}, a_{h}^{m}\right)} \phi_{h+1}^{m_{i}^{\mathcal{E}}\left(s_{h}^{m}, a_{h}^{m}\right)}\right)\\
\leq & e^{\beta } \sum_{m \in [M]}\left(\sum_{i \in\left[n_{h}^{m}\right]} \alpha_{n_{h}^{m}}^{i}  \hat{\phi}_{h+1}^{m_{i}^{\mathcal{E}}\left(s_{h}^{m}, a_{h}^{m}\right)}\right)
\end{align*}

where $m_{i}^{\mathcal{E}}\left(s_{h}^{m}, a_{h}^{m}\right)$ denotes the episode in which $\left(s_{h}^{m}, a_{h}^{m}\right)$ was taken at step $h$ for the $i$-th time in the epoch $\mathcal{E}$. We re-group the above summation by changing the order of the summation. For every $\hat{m}^{\mathcal{E}}$ in the epoch $\mathcal{E}$, the term $\phi_{h+1}^{\hat{m}^{\mathcal{E}}}$ appears in the
summand with $m>\hat{m}^{\mathcal{E}}$ if and only if $\left(s_{h}^{m}, a_{h}^{m}\right)=\left(s_{h}^{m^{\prime}}, a_{h}^{m^{\prime}}\right)$ and the episode $m$ is in the epoch $\mathcal{E}$. 
Since the inverse of the mapping $i \rightarrow m_i^{\mathcal{E}}(s_h^m,a_h^m)$ is 
$\hat{m}^{\mathcal{E}} \rightarrow n_h^{\hat{m}^{\mathcal{E}}}$, we can continue the above display as
\begin{align*}
e^{\beta } \sum_{m \in [M]}\left(\sum_{i \in\left[n_{h}^{m}\right]} \alpha_{n_{h}^{m}}^{i} \hat{\phi}_{h+1}^{m_{i}^{\mathcal{E}}\left(s_{h}^{m} a_{h}^{m}\right)}\right)& \leq  e^{\beta } \sum_{\mathcal{E}=1}^{\ceil{\frac{M}{W}}} \sum_{m=(\mathcal{E}-1)W}^{\mathcal{E}W} \left(\sum_{i \in\left[n_{h}^{m}\right]} \alpha_{n_{h}^{m}}^{i}  \hat{\phi}_{h+1}^{m_{i}^{\mathcal{E}}\left(s_{h}^{m} a_{h}^{m}\right)}\right) \\
& \leq e^{\beta} \sum_{\mathcal{E}=1}^{\ceil{\frac{M}{W}}} \sum_{m^{\prime}=(\mathcal{E}-1)W}^{\mathcal{E}W}  \hat{\phi}_{h+1}^{m^{\prime}}\left(\sum_{t \geq n_{h}^{m^{\prime}}+1} \alpha_{t}^{n_{h}^{m^{\prime}}}\right) \\
& \leq e^{\beta } \left(1+\frac{1}{H}\right)  \sum_{\mathcal{E}=1}^{\ceil{\frac{M}{W}}} \sum_{m^{\prime}=(\mathcal{E}-1)W}^{\mathcal{E}W}  \hat{\phi}_{h+1}^{m^{\prime}}
\end{align*}
where the last step follows the third property in Lemma \ref{lemma: alpha t property}. Collecting the above results and substituting them into \eqref{eq: thm 2 first row in decompose}-\eqref{eq: thm 2 second row in decompose}, we have
$$
\begin{aligned}
\sum_{m \in [M]} \delta_{h}^{m} \leq &\left(e^{\beta(H-h+1)}-1\right) |\mathcal{S}| |\mathcal{A}|+\left(1+\frac{1}{H}\right)  e^{\beta } \sum_{\mathcal{E}=1}^{\ceil{\frac{M}{W}}} \sum_{m=(\mathcal{E}-1)W}^{\mathcal{E}W}   \hat{\phi}_{h+1}^{m} \\
&+\sum_{\mathcal{E}=1}^{\ceil{\frac{M}{W}}} \sum_{m=(\mathcal{E}-1)W}^{\mathcal{E}W}  e^{\beta}\left(\delta_{h+1}^{m}-\phi_{h+1}^{m}+\xi_{h+1}^{m}\right)+3\sum_{\mathcal{E}=1}^{\ceil{\frac{M}{W}}} \sum_{m=(\mathcal{E}-1)W}^{\mathcal{E}W} \sum_{i \in[t]} \alpha_{t}^{i} \Gamma_{h, i}\\
&+ \sum_{\mathcal{E}=1}^{\ceil{\frac{M}{W}}} \sum_{m=(\mathcal{E}-1)W}^{\mathcal{E}W} \left(2e^{\beta(H-h+1)} \beta B_{r,\mathcal{E}} +\left( e^{\beta(H-h+1)} \beta H + \left(e^{\beta(H-h+1) }-1 \right) \right) B_{\mathcal{P},\mathcal{E}} \right)\\
\leq &\left(e^{\beta(H-h+1)}-1\right) |\mathcal{S}| |\mathcal{A}|+\left(1+\frac{1}{H}\right) \sum_{\mathcal{E}=1}^{\ceil{\frac{M}{W}}} \sum_{m=(\mathcal{E}-1)W}^{\mathcal{E}W}  e^{\beta  } \delta_{h+1}^{m} \\
&+\sum_{\mathcal{E}=1}^{\ceil{\frac{M}{W}}} \sum_{m=(\mathcal{E}-1)W}^{\mathcal{E}W} \left(3\sum_{i \in[t]} \alpha_{t}^{i} \Gamma_{h, i}+e^{\beta  } \xi_{h+1}^{m}\right)\\
 &+3(H-h) \sum_{\mathcal{E}=1}^{\ceil{\frac{M}{W}}} \sum_{m=(\mathcal{E}-1)W}^{\mathcal{E}W} \left(2e^{\beta(H-h+1)} \beta B_{r,\mathcal{E}} +\left( e^{\beta(H-h+1)} \beta H + \left(e^{\beta(H-h+1) }-1 \right) \right) B_{\mathcal{P},\mathcal{E}}\right)\\
 \leq &\left(e^{\beta(H-h+1)}-1\right) |\mathcal{S}| |\mathcal{A}|+\left(1+\frac{1}{H}\right) \sum_{m\in [M]} e^{\beta   } \delta_{h+1}^{m} \\
 &+3\sum_{\mathcal{E}=1}^{\ceil{\frac{M}{W}}} \sum_{m=(\mathcal{E}-1)W}^{\mathcal{E}W} \sum_{i \in[t]} \alpha_{t}^{i} \Gamma_{h, i}  + \sum_{m\in [M]} e^{\beta } \xi_{h+1}^{m}\\
 &+  3(H-h) \left(2 e^{\beta(H-h+1)} \beta W B_{r} +\left( e^{\beta(H-h+1)} \beta H + \left(e^{\beta(H-h+1) }-1 \right) \right) W B_{\mathcal{P}} \right)
\end{aligned}
$$
where the second step holds since $\delta_{h+1}^{m} \geq \phi_{h+1}^{m}$ (due to the fact that $\beta>0$ and $V_{h+1}^{*,m} \geq V_{h+1}^{\pi^{m},m}$ ) and the definition of $\hat{\phi}_{h+1}^m$;  the last step follows from the definition of $B_{r}$ and $B_{\mathcal{P}}$.
Now, we unroll the quantity $\sum_{m \in [M]} \delta_{h}^{m}$ recursively in the form of Equation (36), and get

\begin{align} \label{eq: be plugged}
&\sum_{m \in [M]} \delta_{1}^{m} \\
\nonumber \leq & \sum_{h \in[H]}\left[\left(1+\frac{1}{H}\right) e^{\beta}\right]^{h-1}\left[\left(e^{\beta(H-h+1)}-1\right) |\mathcal{S}| |\mathcal{A}|+3\sum_{\mathcal{E}=1}^{\ceil{\frac{M}{W}}} \sum_{m=(\mathcal{E}-1)W}^{\mathcal{E}W} \sum_{i \in[t]} \alpha_{t}^{i} \Gamma_{h, i} +\sum_{m \in [M]}\left(e^{\beta} \xi_{h+1}^{m}\right) \right. \\
\nonumber & \left. + 3(H-h) \left(2 e^{\beta(H-h+1)} \beta W B_{r} +\left( e^{\beta(H-h+1)} \beta H + \left(e^{\beta(H-h+1) }-1 \right) \right) W B_{\mathcal{P}} \right)\right] \\
\nonumber \leq & \sum_{h \in[H]}\left(1+\frac{1}{H}\right)^{h-1}\left[\left(e^{\beta H}-1\right) |\mathcal{S}| |\mathcal{A}|+3\sum_{\mathcal{E}=1}^{\ceil{\frac{M}{W}}} \sum_{m=(\mathcal{E}-1)W}^{\mathcal{E}W} e^{\beta (h-1)} \sum_{i \in[t]} \alpha_{t}^{i} \Gamma_{h, i}+\sum_{m \in [M]} e^{\beta h} \xi_{h+1}^{m} \right. \\
\nonumber & \left. + 3(H-h) \left(2 e^{\beta H} \beta W B_{r} +\left( e^{\beta H} \beta H + \left(e^{\beta H }-1 \right) \right) W B_{\mathcal{P}} \right)\right] \\
\nonumber \leq & e\left[\left(e^{\beta H}-1\right) H |\mathcal{S}| |\mathcal{A}|+3e\sum_{\mathcal{E}=1}^{\ceil{\frac{M}{W}}} \sum_{m=(\mathcal{E}-1)W}^{\mathcal{E}W} \sum_{h\in [H]}  e^{\beta (h-1)} \sum_{i \in[t]} \alpha_{t}^{i} \Gamma_{h, i} \right]+\sum_{h \in[H]} \sum_{m \in [M]}\left(1+\frac{1}{H}\right)^{h-1} e^{\beta h} \xi_{h+1}^{m}\\
\nonumber &  + 3eH^2  \left(2 e^{\beta H} \beta W B_{r} +\left( e^{\beta H} \beta H + \left(e^{\beta H }-1 \right) \right) W B_{\mathcal{P}} \right) 
\end{align}

where the first step uses the fact that $\delta_{H+1}^{m}=0$ for $m \in [M]$; the last step holds since $(1+1 / H)^{h} \leq$ $(1+1 / H)^{H} \leq e$ for all $h \in[H]$. 
Furthermore, the definition of $\Gamma_{h,i}$ and Lemma \ref{lemma: alpha t property} imply that
\begin{align*} 
 \sum_{i \in[t]} \alpha_{t}^{i} \Gamma_{h, i}\leq  C_2 \left(e^{\beta (H-h+1) -1} \right) \sqrt{\frac{H \iota}{t}}. 
\end{align*}
for some constant $C_2>0$. By the pigeonhole principle, for any $h \in[H]$ we have
\begin{align}
\nonumber \sum_{\mathcal{E}=1}^{\ceil{\frac{M}{W}}} \sum_{m=(\mathcal{E}-1)W}^{\mathcal{E}W} \sum_{h\in [H]}  e^{\beta(h-1)}\sum_{i \in[t]} \alpha_{t}^{i} \Gamma_{h, i} & \leq C_2 \left(e^{\beta H}-1\right) \sum_{\mathcal{E}=1}^{\ceil{\frac{M}{W}}} \sum_{m=(\mathcal{E}-1)W}^{\mathcal{E}W} \sqrt{\frac{H \iota}{n_{h}^{m}}} \\
\nonumber & \leq C_2 \left(e^{\beta H}-1\right) \sum_{\mathcal{E}=1}^{\ceil{\frac{M}{W}}} \sqrt{W} \sqrt{\sum_{m=(\mathcal{E}-1)W}^{\mathcal{E}W} \frac{H \iota}{n_{h}^{m}}} \\
& \leq C_2 \left(e^{\beta H}-1\right) M \sqrt{H |\mathcal{S}| |\mathcal{A}| \iota / W} \label{eq: plug 1}
\end{align}

where the second step follows from the Cauchy-Schwarz inequality, the third step holds since $\sum_{(s, a) \in \mathcal{S} \times \mathcal{A}} N_{h}^{m}(s, a)=W$ and the right-hand side of the second step is maximized when $N_{h}^{m}(s, a)=W /(|\mathcal{S}| |\mathcal{A}|)$ for all $(s, a) \in \mathcal{S} \times \mathcal{A}$. Finally, the Azuma-Hoeffding inequality and the fact that $\left|\left(1+\frac{1}{H}\right)^{h-1} e^{\beta h} \xi_{h+1}^{m}\right| \leq e\left(e^{\beta H}-1\right)$ for $h \in[H]$ together imply that with probability at least $1-\delta$, we have
\begin{align} \label{eq: plug 2}
\left|\sum_{h \in[H]} \sum_{m \in [M]}\left(1+\frac{1}{H}\right)^{h-1} e^{\beta h} \xi_{h+1}^{m}\right| \leq C_3\left(e^{\beta H}-1\right) \sqrt{H M \iota} 
\end{align}
for some constant $C_3>0$. Plugging Equations \eqref{eq: plug 1} and \eqref{eq: plug 2} into \eqref{eq: be plugged}, we have
\begin{align}
\nonumber \sum_{m \in [M]} \delta_{1}^{m} \leq & \mathcal{O}\left( \left(e^{\beta H}-1\right) M \sqrt{H |\mathcal{S}| |\mathcal{A}| \iota / W} +\left(e^{\beta H}-1\right) \sqrt{H M \iota} \right. \\
& \left.+ H^2  \left(2 e^{\beta H} \beta W B_{r} +\left( e^{\beta H} \beta H + \left(e^{\beta H }-1 \right) \right) W B_{\mathcal{P}} \right)  \right)    \label{eq: sum of delta 1 m thm 2}
\end{align}
when $M$ is large enough. Invoking Lemma \ref{lemma: decompose of d regret} yields that

\begin{align} \label{eq: decompose of d regret in proof 2}
\nonumber &\operatorname{D-Regret}(M) \\
\nonumber \leq &\sum_{m \in[M]} \frac{1}{\beta}\left[e^{\beta \cdot V_{1}^{*,m}\left(s_{1}^{m}\right)}-e^{\beta \cdot V_{1}^{m}\left(s_{1}^{m}\right)}\right] + \sum_{m \in[M]} \frac{1}{\beta}\left[e^{\beta \cdot V_{1}^{m}\left(s_{1}^{m}\right)}-e^{\beta \cdot V_{1}^{\pi^{m},m}\left(s_{1}^{m}\right)}\right]\\
\nonumber \leq &  \frac{1}{\beta} \sum_{\mathcal{E}=1}^{\ceil{\frac{M}{W}}} \sum_{m=(\mathcal{E}-1)W}^{\mathcal{E}W} H\left(2 e^{\beta H} \beta B_{r,\mathcal{E}} + \left(e^{\beta H} \beta H + \left(e^{\beta H }-1 \right) \right) B_{\mathcal{P},\mathcal{E}} \right)\\
\nonumber &+ \sum_{m \in[M]} \frac{1}{\beta}\left[e^{\beta \cdot V_{1}^{m}\left(s_{1}^{m}\right)}-e^{\beta \cdot V_{1}^{\pi^{m},m}\left(s_{1}^{m}\right)}\right]\\
\nonumber \leq &  \frac{1}{\beta} WH\left(2 e^{\beta H} \beta B_{r} + \left(e^{\beta H} \beta H + \left(e^{\beta H }-1 \right) \right) B_{\mathcal{P}} \right)+ \frac{1}{\beta} \sum_{m \in[M]} \delta_1^m\\
\nonumber \leq &  \frac{1}{\beta} WH\left(2 e^{\beta H} \beta B_{r} + \left(e^{\beta H} \beta H + \left(e^{\beta H }-1 \right) \right) B_{\mathcal{P}} \right)\\
\nonumber &+ \frac{1}{\beta} \mathcal{O}\left( \left(e^{\beta H}-1\right) M \sqrt{H |\mathcal{S}| |\mathcal{A}| \iota / W} +\left(e^{\beta H}-1\right) \sqrt{H M \iota} \right. \\
\nonumber & \left.+ \hspace{0.7cm} H^2  \left(2 e^{\beta H} \beta W B_{r} +\left( e^{\beta H} \beta H + \left(e^{\beta H }-1 \right) \right) W B_{\mathcal{P}} \right)  \right) \\
 \leq  &  \mathcal{O}\left( e^{\beta H} H M \sqrt{H |\mathcal{S}| |\mathcal{A}| \iota / W} +e^{\beta H} H \sqrt{H M \iota} + H^2  e^{\beta H}  W \left(B_{r} + H B_{\mathcal{P}} \right)  \right) \\
  \leq  &  \widetilde{\mathcal{O}}\left( e^{\beta H} M \sqrt{H^3 |\mathcal{S}| |\mathcal{A}|  / W} +e^{\beta H}  \sqrt{H^3 M } + H^3  e^{\beta H}  W \left( B_{r} +  B_{\mathcal{P}} \right)  \right) 
\end{align}
where the second step holds by \eqref{eq: exp Vm -exp V*}, the third inequality holds because of the definition of $B_{\mathcal{P}}$, $B_{{r}}$ and $\delta_1^m$, the forth inequality is due to \eqref{eq: sum of delta 1 m thm 2}, and the fifth inequality follows from $e^{\beta H}-1\leq \beta H e^{\beta H}$ for $\beta>0$. Finally, by setting $W=M^{\frac{2}{3}} H^{-\frac{3}{4}} \left(B_{\mathcal{P}}+ B_{r}\right)^{-\frac{2}{3}}|S|^{\frac{1}{3}} |A|^{\frac{1}{3}}$, we conclude that
\begin{align*}
\operatorname{D-Regret(M)}\leq & \widetilde{\mathcal{O}} \left(e^{\beta H}  |S|^{\frac{1}{3}} |A|^{\frac{1}{3}} H^{\frac{9}{4}} M^{\frac{2}{3}} \left(B_{\mathcal{P}}+ B_{r}\right)^{\frac{1}{3}}\right).
\end{align*}

The proof is  similar for the case of $\beta<0$, and one only needs to exchange the role of $V_{h}^{m}$, $V_{h}^{\pi^{m},m}$ and $V_{h}^{*,m}$ in the definitions of $\delta_{h}^{m}, \phi_{h}^{m}, \xi_{h}^{m}$:
$$
\begin{aligned}
\delta_{h}^{m} &:=e^{\beta \cdot V_{h}^{\pi^m}\left(s_{h}^{m}\right)}-e^{\beta \cdot V_{h}^{m}\left(s_{h}^{m}\right)}, \\
\phi_{h}^{m} &:=e^{\beta \cdot V_{h}^{*,m}\left(s_{h}^{m}\right)}-e^{\beta \cdot V_{h}^{m}\left(s_{h}^{m}\right)}, \\
\xi_{h+1}^{m} &:=\left[\left(\mathcal{P}^m_{h}-\widehat{\mathcal{P}}_{h}^{m}\right)\left(e^{\beta \cdot V_{h+1}^{\pi^{m}}}-e^{\beta \cdot V_{h+1}^{*,m}}\right)\right]\left(s_{h}^{m}, a_{h}^{m}\right)
\end{aligned}
$$
to derive the counterparts of \eqref{eq: thm 2 first row in decompose} and \eqref{eq: thm 2 second row in decompose}, and complete the remaining analysis.

\section{Proof of Theorem \ref{thm: for alg 3}}


\subsection{Multi-scale ALG Initialization} \label{sec: MALG}
\begin{algorithm}[ht]
   \caption{Multi-scale ALG Initialization (MALG-initialization)}
   \label{alg:algoirthm 4}
\begin{algorithmic}[1]
   \STATE {\bfseries Inputs:} ALG and its associated $\rho(\cdot)$, n;
   \FOR{$\tau=0,\ldots, 2^n-1$} 
   \FOR{k=n,n-1,\ldots,0}
    \STATE If $\tau$ is a multiple of $2^k$, with probability $\frac{\rho(2^n)}{\rho(2^k)}$, schedule a new instance $alg$ of ALG that starts at $alg.s=\tau+1$ and ends at $alg.e=\tau+2^k$
    \ENDFOR
   \ENDFOR
\end{algorithmic}
\end{algorithm}

\subsection{An illustrative example} \label{appe: MALG example}
For better illustration, we give an example with $n=4$. This example has also been shown in \cite{wei2021non} and we present here for completeness. By Algorithm \ref{alg:algoirthm 4}, one possible realization of the MALG initialization is shown in Figure \ref{fig:MALG} with one order-4 instance (red), zero order-3 instance, two order-2 instances (green), two order-1 instances (purple) and five order-0 instances (blue). 
The bolder part of the segment indicates the period of time when the instances are active, while the thinner part indicates the inactive period. At any point of time, the active instance is always the one with the shortest length.
The dashed arrow marked with \circled{1} indicates that ALG is executed as of the two sides of the arrow are concatenated. On the other hand, the two blue instances on the two sides of the dashed line marked with \circled{2} are two different order-0 instances, so the second one should start from scratch even though they are consecutive.

\begin{figure}
    \centering
    \includegraphics[width=13cm]{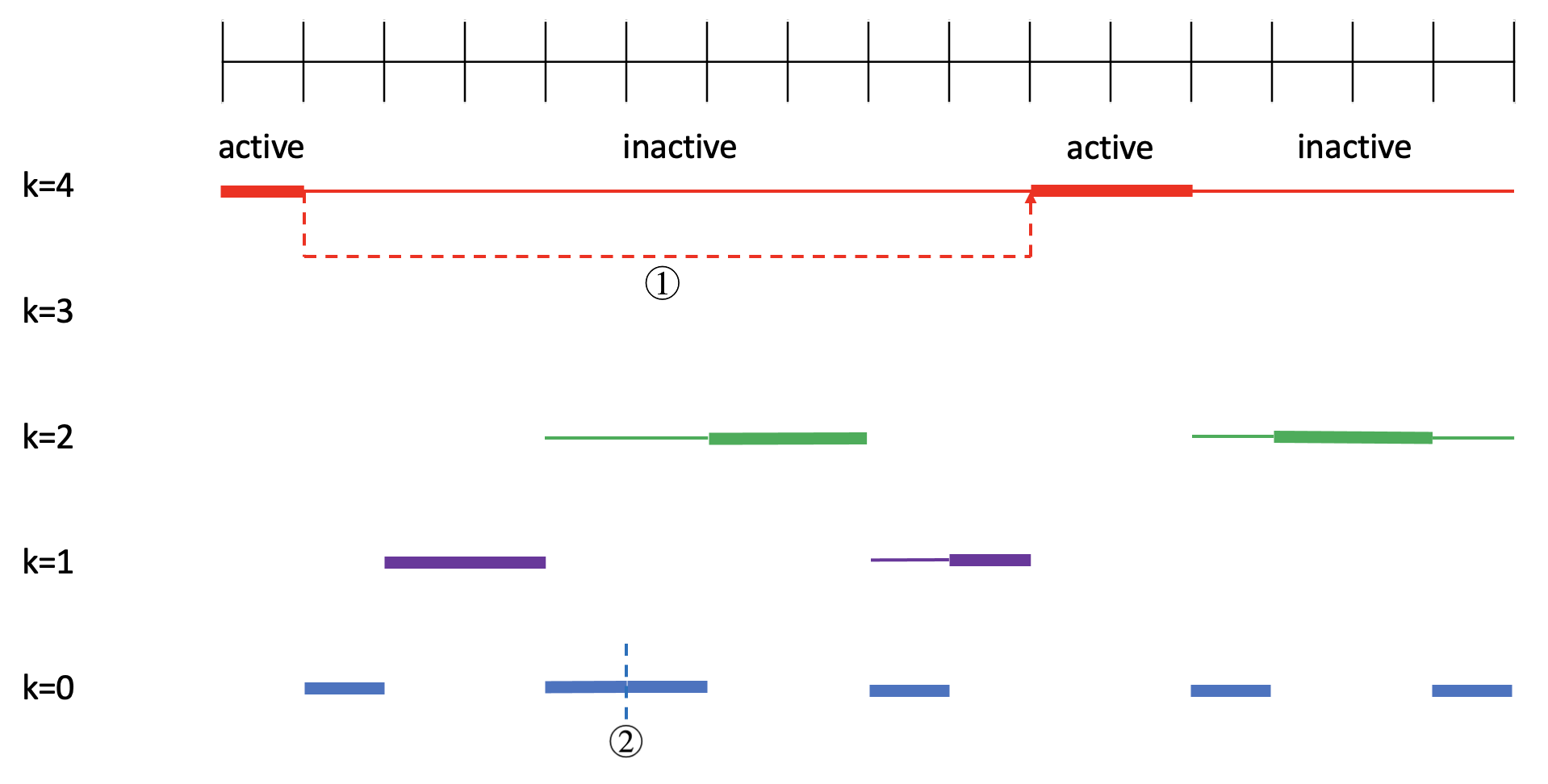}
    \caption{An illustrate example of MALG with $n=4$.}
    \label{fig:MALG}
\end{figure}

\subsection{Preliminaries} \label{app: adaptive prelim}
Similar to \cite{wei2021non}, our approach takes a base algorithm that tackles the risk-sensitive RL problem when the environment is (near-)stationary, and turns it into another algorithm that can deal with non-stationary environments. The base algorithm is assumed to satisfy the following requirement:

\begin{assumption} \label{ass: non-stationary detection}
$ALG$ outputs an auxiliary quantity $e^{\beta V_1^m(s_1)} \in[0,e^{\beta H}]$ at the beginning of each round $m$. There exist a non-stationarity measure $\Delta$ and a non-increasing function $\rho:[M] \rightarrow \mathbb{R}$ such that running $ALG$ satisfies the following: for all $m \in[M]$, as long as $\Delta_{[1, m]} \leq \rho(m)$, without knowing $\Delta_{[1, m]}$ ALG ensures with probability at least $1-\frac{\delta}{M}$: if $\beta>0$, it holds that
$$
e^{\beta V_1^m(s_1)} \geq \min _{\tau \in[1, m]} e^{\beta V_1^{*,\tau}(s_1)}-\Delta_{[1, m]} \quad \text { and } \quad \frac{1}{m} \sum_{\tau=1}^{m}\left(e^{\beta V_1^\tau(s_1)}-e^{\beta \sum_{h=1}^H r_h^\tau}\right) \leq \rho(m)+\Delta_{[1, m]},
$$
and if $\beta<0$, it holds that
$$
\max_{\tau \in[1, m]} e^{\beta V_1^{*,\tau}(s_1)} \geq e^{\beta V_1^m(s_1)} -\Delta_{[1, m]} \quad \text { and } \quad \frac{1}{m} \sum_{\tau=1}^{m}\left(e^{\beta \sum_{h=1}^H r_h^\tau}-e^{\beta V_1^\tau(s_1)}\right) \leq \rho(m)+\Delta_{[1, m]},
$$
Furthermore, we assume that $\rho(m) \geq \frac{1}{\sqrt{m}}$ and $C(m)=m\rho(m)$ is a non-decreasing function.
\end{assumption}

Under Assumption \ref{ass: non-stationary detection}, the multi-scale nature of MALG allows the learner's regret to also enjoy a multi-scale structure, as shown in the next lemma:

\begin{lemma} \label{lemma: lemma 3 in wei2021}
Let $\widehat{n}=\log_{2} M+1$ and $\widehat{\rho}(m) = 6 \widehat{n} \log (M / \delta) \rho(m)$. {MALG} with input $n \leq \log _{2} M$ guarantees the following: for every instance $alg$ that MALG maintains and every $m \in[alg.s, alg.e]$, as long as $\Delta_{[\text {alg.s }, t]} \leq \rho\left(m^{\prime}\right)$ where $m^{\prime}=m- alg.s +1$, we have with probability at least $1-\frac{\delta}{M}$ : if $\beta>0$, it holds that
\begin{align*}
{g}_{m} \geq \min _{\tau \in [alg.s,m]} e^{\beta V_1^{*,\tau}(s_1)}-\Delta_{[\text {alg.s }, t]}, \quad 
\frac{1}{m^{\prime}} \sum_{\tau= { alg.s }}^{m}\left({g}_{\tau}-e^{\beta \sum_{h=1}^H r_h^\tau}\right) \leq \widehat{\rho}\left(m^{\prime}\right)+\widehat{n} \Delta_{[alg.s, m]},
\end{align*}
and if $\beta<0$, it holds that
\begin{align*}
\max _{\tau \in [alg.s,m]} e^{\beta V_1^{*,\tau}(s_1)}  \geq {g}_{m} -\Delta_{[\text {alg.s }, t]}, \quad 
\frac{1}{m^{\prime}} \sum_{\tau= { alg.s }}^{m}\left(e^{\beta \sum_{h=1}^H r_h^\tau} -{g}_{\tau} \right) \leq \widehat{\rho}\left(m^{\prime}\right)+\widehat{n} \Delta_{[alg.s, m]},
\end{align*}
where ${g}_{m}$ is the UCB-based optimistic estimator $e^{\beta V_1^m(s_1)}$ for the unique active instance alg at the episode $m$, and the number of instances started within $[alg.s, m]$ is upper bounded by $6 \widehat{n} \log (M / \delta) \frac{C\left(m^{\prime}\right)}{C(1)}$.
\end{lemma}

\begin{proof}
The proof is similar to that of Lemma 3 in \cite{wei2021non} with the standard value functions replaced by the exponential value functions and is thus omitted.
\end{proof}
Lemma \ref{lemma: lemma 3 in wei2021} states that even if there are multiple instances interleaving in a complicated way, the regret for a specific interval is still almost the same as running ALG alone on this interval, due to the carefully chosen probability  $\frac{\rho(2^n)}{\rho(2^k)}$ in Algorithm \ref{alg:algoirthm 4}.
Built on Lemma \ref{lemma: lemma 3 in wei2021}, the regret on a single block $[m_n,E_n]$, where $E_n$ is either $m_n+2^n-1$ or something smaller in the case where a restart is triggered, is bounded in the following lemma:

\begin{lemma}
For Algorithm \ref{alg:algoirthm 3} with ALG satisfying Assumption \ref{ass: non-stationary detection} and
on every block $\mathcal{J}=\left[m_{n}, E_{n}\right]$ where $E_{n} \leq m_{n}+2^{n}-1$, it holds with high probability that: 
\begin{align*}
\begin{cases}
\sum_{\tau \in \mathcal{J}}\left(e^{\beta V_1^{*,\tau}(s_1)}-R_{\tau}\right) \leq  \widetilde{\mathcal{O}}\left(\sum_{i=1}^{\ell} C\left(\left|\mathcal{I}_{i}^{\prime}\right|\right)+\sum_{m=0}^{n} \frac{\rho\left(2^{m}\right)}{\rho\left(2^{n}\right)} C\left(2^{m}\right)\right), \text{ if } \beta>0, \\
\sum_{\tau \in \mathcal{J}}\left(R_{\tau}-e^{\beta V_1^{*,\tau}(s_1)}\right) \leq  \widetilde{\mathcal{O}}\left(\sum_{i=1}^{\ell} C\left(\left|\mathcal{I}_{i}^{\prime}\right|\right)+\sum_{m=0}^{n} \frac{\rho\left(2^{m}\right)}{\rho\left(2^{n}\right)} C\left(2^{m}\right)\right), \text{ if } \beta<0,
\end{cases}
\end{align*}
where $\left\{\mathcal{I}_{1}^{\prime}, \ldots, \mathcal{I}_{\ell}^{\prime}\right\}$ is any partition of $\mathcal{J}$ such that $\Delta_{\mathcal{I}_{i}^{\prime}} \leq \rho\left(\left|\mathcal{I}_{i}^{\prime}\right|\right)$ for all $i$.
\end{lemma}
\begin{proof}
The proof is similar to that of Lemma 4 in \cite{wei2021non} with the standard value functions replaced by the exponential value functions and is thus omitted.
\end{proof}

Built on the dynamic regret over a block, we can further bound the dynamic regret over a single-epoch. The epoch is defined as an interval that starts at the first episode after a restart and ends at the first time when the restart is triggered.

\begin{lemma} \label{lemma: adaptive epoch bound}
Assume that $C(m)$ takes the form of $C(m)=c_1 m^{\frac{1}{2}}$ for some constant $c_1$. Then, for Algorithm \ref{alg:algoirthm 3} with ALG satisfying Assumption \ref{ass: non-stationary detection} and on every epoch $\mathcal{E}$, it holds with high probability that:
\begin{align*}
\begin{cases}
   \sum_{\tau \in \mathcal{E}} \left(e^{\beta V_1^{*,\tau}(s_1)}-R_{\tau}\right) \leq \widetilde{\mathcal{O}} \left( c_1^{\frac{2}{3}} \Delta_{\mathcal{E}}^{\frac{1}{3}} |\mathcal{E}|^{\frac{2}{3}}+c_1|\mathcal{E}|^{\frac{1}{2}}\right), \text{ if } \beta>0,\\
\sum_{\tau \in \mathcal{E}} \left(R_{\tau}-e^{\beta V_1^{*,\tau}(s_1)}\right) \leq \widetilde{\mathcal{O}} \left( c_1^{\frac{2}{3}} \Delta_{\mathcal{E}}^{\frac{1}{3}} |\mathcal{E}|^{\frac{2}{3}}+c_1|\mathcal{E}|^{\frac{1}{2}}\right), \text{ if } \beta<0,
\end{cases}
\end{align*}
\end{lemma}
\begin{proof}
The proof is similar to that of Lemma 22 in \cite{wei2021non} with the standard value functions replaced by the exponential value functions and is thus omitted.
\end{proof}

Finally, we have the following bound on the number of epoch:

\begin{lemma}[Lemma 24 in \cite{wei2021non}] \label{lemma: adaptive number of epoch}
Assume that $C(m)$ takes the form of $C(m)= c_1 m^{\frac{1}{2}}$ for some constant $c_1$. Then, with high probability, the number of epoch is upper-bounded by $1+2(c_1^{-\frac{1}{3}} \Delta^{\frac{2}{3}} M^{\frac{1}{3}})$.
\end{lemma}

\subsection{Proof of Theorem \ref{thm: for alg 3}}\label{app: proof of thm3}
We first focus on the case for $\beta>0$.
Let $\mathcal{E}_1,\ldots, \mathcal{E}_N$ be epochs in $[1,M]$. If Assumption \ref{ass: non-stationary detection} holds, by Lemma \ref{lemma: adaptive epoch bound}, the dynamic regret of the exponential value functions over $M$ episodes is upper-bounded by
\begin{align} \label{eq: adaptive upper bound}
\nonumber \sum_{m=1}^M \left(e^{\beta V_1^{*,m}(s_1)}-R_{m}\right) \leq & \widetilde{\mathcal{O}} \left( \sum_{i=1}^N \left( c_1^{\frac{2}{3}} \Delta_{\mathcal{E}_i}^{\frac{1}{3}} |\mathcal{E}_i|^{\frac{2}{3}}+c_1|\mathcal{E}_i|^{\frac{1}{2}}\right)\right)\\
\nonumber \leq &\widetilde{\mathcal{O}} \left( c_1^{\frac{2}{3}} \Delta^{\frac{1}{3}}M^{\frac{2}{3}}+c_1N^{\frac{1}{2}} M^{\frac{1}{2}}\right)\\
\leq & \widetilde{\mathcal{O}} \left( c_1^{\frac{2}{3}} \Delta^{\frac{1}{3}}M^{\frac{2}{3}}\right).
\end{align}
where the second inequality follows from H\"{o}lder's inequality and the facts that $\sum_{i=1}^N \Delta_{\mathcal{E}_i} \leq \Delta$ and $\sum_{i=1}^N |\mathcal{E}_i| \leq M$, the last step holds by the bound on $N$ from Lemma \ref{lemma: adaptive number of epoch}.

Now, it remains to show that the base algorithms RSVI and RSQ satisfy Assumption \ref{ass: non-stationary detection} and provide the concrete form of $\Delta(m)$, $\rho(m)$, $c_1$ and $c_2$.

\begin{itemize}
\item RSVI as the base algorithm: it has been shown in Lemma \ref{lemma: bound on V difference} and \eqref{eq: sum of delta 1 m} in the proof of Theorem \ref{thm: for alg 1} that RSVI satisfies 
Assumption \ref{ass: non-stationary detection} with the following choices:
\begin{align*}
\nonumber&\Delta(m)=H\left( \left|e^{\beta H}-1\right| B_{\mathcal{P},m} +  g_1(\beta) B_{r,m} \right),\\
\nonumber& \rho(m)= {\mathcal{O}} \left( \left( \left|e^{\beta H}-1\right| + g_1(\beta)\right) \sqrt{H^{2} |S|^2 |A| \iota^2 /m} \right),\\
& c_1= \left( \left|e^{\beta H}-1\right| + g_1(\beta)\right) \sqrt{H^{2} |S|^2 |A| \iota^2 }.
\end{align*}
Then, by plugging in the form of $\Delta$ and $c_1$ in \eqref{eq: adaptive upper bound}, and using $e^{\beta H}-1\leq \beta H e^{\beta H}$ for $\beta>0$, we have 
\begin{align*}
\sum_{m=1}^M \left(e^{\beta V_1^{*,m}(s_1)}-R_{m}\right)\leq & \widetilde{\mathcal{O}} \left( \beta e^{\beta H} H^{2} |\mathcal{S}|^{\frac{2}{3}} |\mathcal{A}|^{\frac{1}{3}} B^{\frac{1}{3}}M^{\frac{2}{3}}\right).
\end{align*}
Invoking the above inequality with Lemma \ref{lemma: decompose of d regret} and applying Azuma's inequality to bound $\sum_{m=1}^M (R_m-e^{\beta V^{\pi^m,m}_1})$ yield that:
\begin{align*}
    \operatorname{D-Regret(M)}\leq & \widetilde{\mathcal{O}} \left(  e^{\beta H} H^{2} |\mathcal{S}|^{\frac{2}{3}} |\mathcal{A}|^{\frac{1}{3}} B^{\frac{1}{3}}M^{\frac{2}{3}}\right).
\end{align*}

\item RSQ as the base algorithm: it has also been shown in Lemma \ref{lemma: Qm-Qstar lower bound} and \eqref{eq: sum of delta 1 m thm 2} in the proof of Theorem \ref{thm: for alg 2} that RSQ satisfies Assumption \ref{ass: non-stationary detection} with the following choices:
\begin{align*}
&\Delta(m)=H\left(2 g_1(\beta) B_{r,m} + \left(g_1(\beta) H + \left|e^{\beta H }-1 \right| \right) B_{\mathcal{P},m} \right)\\
& \rho(m)= {\mathcal{O}} \left( \left|e^{\beta H}-1\right|  \sqrt{H |\mathcal{S}| |\mathcal{A}| \iota /m}  \right),\\
&c_1= {\mathcal{O}} \left( \left|e^{\beta H}-1\right|  \sqrt{H |\mathcal{S}| |\mathcal{A}| \iota } \right).  
\end{align*}

Then, by plugging in the form of $\Delta$ and $c_1$ in \eqref{eq: adaptive upper bound}, and using $e^{\beta H}-1\leq \beta H e^{\beta H}$ for $\beta>0$, we have 
\begin{align*}
\sum_{m=1}^M \left(e^{\beta V_1^{*,m}(s_1)}-R_{m}\right)\leq & \widetilde{\mathcal{O}} \left( \beta e^{\beta H} H^{\frac{5}{3}} |\mathcal{S}|^{\frac{1}{3}} |\mathcal{A}|^{\frac{1}{3}} B^{\frac{1}{3}}M^{\frac{2}{3}}\right).
\end{align*}
Invoking the above inequality with Lemma \ref{lemma: decompose of d regret} and applying Azuma's inequality to bound $\sum_{m=1}^M (R_m-e^{\beta V^{\pi^m,m}_1})$ yield that:
\begin{align*}
    \operatorname{D-Regret(M)}\leq & \widetilde{\mathcal{O}} \left(  e^{\beta H} H^{\frac{5}{3}} |\mathcal{S}|^{\frac{1}{3}} |\mathcal{A}|^{\frac{1}{3}} B^{\frac{1}{3}}M^{\frac{2}{3}}\right).
\end{align*}
\end{itemize}

For the case of $\beta<0$, note that from Lemma \ref{lemma: decompose of d regret}, the dynamic regret can be bounded and decomposed as follows:
$$
\operatorname{D-Regret}(M) \leq \frac{e^{-\beta H}}{(-\beta)} \sum_{m \in[M]} \left[e^{\beta \cdot V_{1}^{{m}}\left(s_{1}^{m}\right)}-e^{\beta \cdot V_{1}^{*,m}\left(s_{1}^{m}\right)}\right]+ \frac{e^{-\beta H}}{(-\beta)} \sum_{m \in[M]} \left[e^{\beta \cdot V_{1}^{\pi^{m},m}\left(s_{1}^{m}\right)}-e^{\beta \cdot V_{1}^{m}\left(s_{1}^{m}\right)}\right].
$$
Then, following a procedure similar to the one used for the case $\beta>0$ and noticing that $g_1(\beta)H=-\beta H \geq 1-e^{\beta H} $ for $\beta<0$, we obtain the desired result.
\section{Proof of Theorem \ref{thm: low bound}}
\subsection{Case \texorpdfstring{$\beta>0$}{Lg}}
Consider a stochastic $k$-arm and $M$ horizons bandit environment $\nu$, where the reward for pulling arm $j\in\{1,2,\ldots,k\}$ is given by the scaled Bernoulli random variable $Ber(p_j)$
\begin{align*}
    X_j=\begin{cases}
    H, &\text{with probability } p_j,\\
    0, &\text{with probability } 1-p_j
    \end{cases}
\end{align*}
where $H\geq 1$ specifies the range of the reward. We let the arm $i$ be the unique optimal arm and all the other $k-1$ arms have the same $p_j$, that is, $p_1=p_2=\cdots=p_{i-1}=p_{i+1}=\cdots=p_k=p$ and  $p_i=p+\Delta$ for some constants $p>0$ and $\Delta>0$.
Define $X_j^m$ to be the outcome of arm $j$ (if pulled) in round $m$, and $Y^m$ to be the outcome of arm actually pulled in round $m$.

\begin{lemma}\label{lemma: Regret decomposition with entropic risk measure}
For the Bernoulli bandit $\nu$ described above, if $p=e^{-\beta H}$, $\Delta\leq e^{-\beta H}$ and $H\geq \frac{\log{2}}{\beta}$, then for every policy $\pi$, the regret with the entropic risk measure in $\nu$ satisfies
\begin{align*}
\operatorname{Regret}(M)\coloneqq& \sum_{m=1}^M \frac{1}{\beta} \left(\log \left[\EE[ \exp\left(\beta X_1^m \right) ] \right]- \log \left[\EE[ \exp\left(\beta Y^m \right) ] \right] \right)\\
\geq & \sum_{j\in[k]/\ i}  { \EE\left[T_j(M)\right]}  \frac{\Delta (e^{\beta H}-1)}{4 \beta}
\end{align*}
\end{lemma}
\begin{proof}
By the definition of $\operatorname{Regret}(M)$, we have
\begin{align} \label{eq: lower bound D regret before expectation}
\nonumber \operatorname{Regret(M)}=& \sum_{m=1}^M \frac{1}{\beta} \left(\log \left[\EE[ \exp\left(\beta X_1^m \right) ] \right]- \log \left[\EE[ \exp\left(\beta Y^m \right) ] \right] \right)\\
=& \sum_{j\in[k]/\ i}  \frac{T_i(M)}{\beta} \left(\log \left[\EE[ \exp\left(\beta X_1 \right) ] \right]- \log \left[\EE[ \exp\left(\beta X_i \right) ] \right] \right)
\end{align}
where the last step holds because of the independence among $\{X_1^m\}_{m=1}^M$ and the independence among $\{Y^m\}_{m=1}^M$. Taking the expectation over $M$ on both sides of \eqref{eq: lower bound D regret before expectation}, we have
\begin{align*}
    \EE\left[ \operatorname{Regret(M)}\right]=& \sum_{j\in[k]/\ i}  \frac{ \EE\left[T_i(M)\right]}{\beta} \left(\log \left[\EE[ \exp\left(\beta X_i \right) ] \right]- \log \left[\EE[ \exp\left(\beta X_j \right) ] \right] \right)\\
    =& \sum_{j\in[k]/\ i}  \frac{ \EE\left[T_j(M)\right]}{\beta} \log \left(  \frac{(p+\Delta) e^{\beta H}+(1-p-\Delta)}{p e^{\beta H}+(1-p)} \right) \\
     =& \sum_{j\in[k]/\ i}  \frac{ \EE\left[T_j(M)\right]}{\beta} \log \left( 1+ \frac{\Delta (e^{\beta H}-1)}{p e^{\beta H}+(1-p)} \right) \\
    =& \sum_{j\in[k]/\ i}  \frac{ \EE\left[T_j(M)\right]}{\beta} \log \left( 1+ \frac{\Delta (e^{\beta H}-1)  }{2-e^{-\beta H}} \right) \\
     \geq & \sum_{j\in[k]/\ i}  \frac{ \EE\left[T_j(M)\right]}{\beta} \log \left( 1+ \frac{\Delta (e^{\beta H}-1)}{2} \right) \\
     \geq & \sum_{j\in[k]/\ i}  { \EE\left[T_j(M)\right]}  \frac{\Delta (e^{\beta H}-1)}{4 \beta}
\end{align*}
where the forth equality holds since $p=e^{-\beta H}$, the first inequality follows from $e^{\beta H}\geq 2$, and the second inequality holds since $\Delta\leq e^{-\beta H}$ and $\log(1+x)\geq \frac{x}{2}$ for $x \in [0,1]$.
\end{proof}

\begin{lemma}\label{lemma: low bound bandit}
Let $k>1$. For every policy $\pi$ and sufficiently large $M$ and $H$, there exists a $k$-arm bandit instance such that
\begin{align*}
 \EE_{\Vec{p}}\left[ \operatorname{Regret(M)}\right]>& \frac{e^{\beta H/2}-1}{\beta}\frac{\sqrt{Mk}}{64e}.
\end{align*}
\end{lemma}
\begin{proof}
Fix a policy $\pi$. Let $\Delta \in[0,e^{-\beta H}]$ be some constant to be chosen later. We start with a Bernoulli bandit where the reward of each arm is a scaled Bernoulli random variable Ber$(p_i)$ with $\Vec{p}\coloneqq (p_1,\ldots, p_k)=(\Delta+p,p,\ldots,p)$. This environment and the policy $\pi$ give rise to the probability measure $\mathbb{P}_{\Vec{p}}$ on the canonical bandit model (Section 4.6 in \cite{lattimore2020bandit}) induced by the $M$-round interconnection of $\pi$ and $\nu$. Expectation under $\mathbb{P}_{\Vec{p}}$  will be denoted as $\EE_{\Vec{p}}$.
To choose the second environment, let
\begin{align*}
    i=\argmin_{j>1} \EE_{\Vec{p}}\left[T_j(M) \right].
\end{align*}
Since $\sum_{j=1}^k \EE_{\Vec{p}}\left[T_j(M) \right]=M$, it holds that 
\begin{align}\label{eq: upper bound of T i}
\EE_{\Vec{p}}\left[T_i(M) \right]\leq \frac{M}{k-1}    
\end{align}
The second bandit is also a Bernoulli bandit where the reward of each arm is a scaled Bernoulli random variable Ber$(p_i^\prime)$ with $\Vec{p}^\prime\coloneqq (p_1^\prime,\ldots, p_k^\prime)=(\Delta+p,p,\ldots,2 \Delta+p,p,\ldots,p)$, where specifically $p_i^\prime=2\Delta+p$. Therefore, $p_j=p_j^\prime$ except at index $i$ and the optimal arm in $\nu_{\Vec{p}}$ is the first arm, while in $\nu_{\Vec{p}^\prime}$ arm $i$ is optimal. Then,  Lemma \ref{lemma: Regret decomposition with entropic risk measure} and a simple calculation lead to
\begin{align*}
 &\EE_{\Vec{p}}\left[ \operatorname{Regret(M)}\right]\geq \mathbb{P}_{\Vec{p}}(T_1(M)\leq \frac{M}{2})\frac{M\Delta (e^{\beta H}-1)}{8 \beta},\\
 & \EE_{\Vec{p}^\prime}\left[ \operatorname{Regret(M)}\right]>\mathbb{P}_{\Vec{p}^\prime}(T_1(M)> \frac{M}{2})\frac{M\Delta (e^{\beta H}-1)}{8 \beta }.
\end{align*}
Then, applying the Bretagnolle-Huber inequality in Lemma \ref{lemma: Bretagnolle-Huber} leads to
\begin{align*}
&\EE_{\Vec{p}}\left[ \operatorname{Regret(M)}\right]+\EE_{\Vec{p}^\prime}\left[ \operatorname{Regret(M)}\right]\\
> &\frac{M\Delta (e^{\beta H}-1)}{8\beta } \left(\mathbb{P}_{\Vec{p}}(T_1(M)\leq \frac{M}{2})+\mathbb{P}_{\Vec{p}^\prime}(T_1(M)> \frac{M}{2}) \right)\\
\geq &\frac{M\Delta (e^{\beta H}-1)}{8 \beta} \exp\left({-D_{\text{KL}}(\mathbb{P}_{\Vec{p}} \mid \mathbb{P}_{\Vec{p}^\prime} )  }\right)
\end{align*}

It remains to upper-bound $D_{\text{KL}}(\mathbb{P}_{\Vec{p}} \mid \mathbb{P}_{\Vec{p}^\prime} )$. For this, we use Lemma \ref{lemma: Divergence decomposition}:
\begin{align}\label{eq: KL Pv}
{D}_{\mathrm{KL}}\left(\mathbb{P}_{\nu}\mid \mathbb{P}_{\nu^{\prime}}\right)=& \mathbb{E}_{\mathbb{P}_{\Vec{p}}}\left[T_{i}(M)\right] \mathrm{D}_{\mathrm{KL}}\left(\text{Ber}(p_i) \mid \text{Ber}(p_{i}^{\prime})\right)\\
\nonumber=& \mathbb{E}_{\mathbb{P}_{\Vec{p}}}\left[T_{i}(M)\right] \mathrm{D}_{\mathrm{KL}}\left(p \mid 2\Delta+p )\right)\\
\nonumber\leq &  \mathbb{E}_{\mathbb{P}_{\Vec{p}}}\left[T_{i}(M)\right]\cdot \frac{4\Delta^2}{(2\Delta+p)(1-2\Delta-p)}\\
\nonumber\leq & \frac{M}{k-1}\cdot \frac{4\Delta^2}{(2\Delta+p)(1-2\Delta-p)}\\
\nonumber\leq & \frac{16 M \Delta^2}{kp}\\
\nonumber\leq & \frac{16 e^{\beta H} M \Delta^2}{k}
\end{align}
where the first inequality follows from Lemma \ref{lemma: KL of ber upper bound}, the second inequality holds by \eqref{eq: upper bound of T i}, the third step follows from $1-2\Delta-p \geq \frac{1}{2}$ and $k\geq 3$, and the last step holds by $p=e^{-\beta H}$.

Substituting this into the previous expression, we find that
\begin{align*}
 \EE_{\Vec{p}}\left[ \operatorname{Regret(M)}\right]+\EE_{\Vec{p}^\prime}\left[ \operatorname{Regret(M)}\right]>& \frac{M\Delta (e^{\beta H}-1)}{8 \beta} \exp\left({-\frac{16 e^{\beta H} M \Delta^2}{k}  }\right)   \\
 >& \frac{e^{\beta H/2}-1}{\beta}\frac{\sqrt{Mk}}{32e}
\end{align*}
where the second inequality holds by choosing $\Delta=\sqrt{k/(16M e^{\beta H})} \leq e^{-\beta H}$ with $M$ sufficiently large. This result is completed by using $2\max(a,b)\geq a+b$.
\end{proof}

\begin{lemma}\label{lemma: low bound of regret beta>0}
For every policy $\pi$ and sufficiently large $M$ and $H$, there exists a MDP instance with horizon $H$, $S\geq 3$ states and $A$ actions such that
\begin{align*}
 \EE\left[ \operatorname{Regret(M)}\right]>&  \frac{e^{\beta H/2}-1}{\beta}\frac{\sqrt{MSA}}{64e}.
\end{align*}
\end{lemma}
\begin{proof}
Note that the $M$-round $k$-arm bandit model described in Lemma \ref{lemma: low bound bandit} is a special case of an $M$-episode $(H+2)$-horizon MDP with $S$ states and $\frac{S-1}{2}$ actions where $S\geq 3$ is odd.  Let $s_1$ be the initial state, and all other states be absorbing regardless of actions taken. At the initial state $s_1$, we may choose to take action $a_1$, $a_2, \ldots, a_{\frac{S-1}{2}}$. If $a_j$ is taken at state $s_1$, then we transition to state $s_{1+2(j-1)+1}$ with probability $p_j$ and to state  $s_{1+2(j-1)+2}$ with probability $1-p_j$. The reward function satisfies $r_h(s_{1+2(j-1)+1},a)=1$, $r_h(s_{1+2(j-1)+2},a)=0$ and $r_h(s_1,a)=0$ for all $h\in[H+2]$, $a\in \mathcal{A}$ and $j=1,\ldots, \frac{S-1}{2}$.
\end{proof}

Based on Lemma \ref{lemma: low bound of regret beta>0}, let us now incorporate the non-stationarity of the MDP and derive a lower bound for the dynamic regret $\operatorname{D-Regret(M)}$. We will construct the non-stationary environment as a switching-MDP. For each segment of length $M_0$, the environment is held constant, and the regret lower bound for each segment is $\mathcal{O}\left(\frac{e^{\beta H/2}-1}{\beta} \sqrt{SA M_0}\right)$. At the beginning of each new segment, we uniformly sample a new action at random at the state $s_1$ from the action space $\mathcal{A}$ to be the optimal action at the state $s_1$ for the new segment. In this case, the learning algorithm cannot use the information it learned during its previous interactions with the environment, even if it knows the switching structure of the environment. Therefore, the algorithm needs to learn a new (static) MDP in each segment, which leads to a dynamic regret lower bound of
$$\mathcal{O}\left(\frac{e^{\beta H/2}-1}{\beta} L \sqrt{S A M_0}\right)=\mathcal{O}\left(\frac{e^{\beta H/2}-1}{\beta} \sqrt{S A M L}\right),$$ 
where $L$ is the number of segments. Every time that the optimal action at the state $s_1$ varies, it will cause a variation of magnitude $2 \Delta=\sqrt{SA/(4M_0 e^{\beta H})}$ in the transition kernel. The constraint of the overall variation budget requires that 
$$2 \Delta L= \sqrt{\frac{S A}{4 M_0 e^{\beta H}}} L = \sqrt{\frac{S A L^3}{4 M e^{\beta H}}}  \leq B,$$ 
which in turn requires $L \leq 4^{\frac{1}{3}} B^{\frac{2}{3}}  M^{\frac{1}{3}} e^{\frac{ \beta H}{3}}  S^{-\frac{1}{3}} A^{-\frac{1}{3}}$. Finally, by assigning the largest possible value to $L$ subject to the variation budget, we obtain a dynamic regret lower bound of 
$$\mathcal{O}\left( \frac{e^{\frac{2\beta H}{3}}-1}{\beta}  S^{\frac{1}{3}} A^{\frac{1}{3}} B^{\frac{1}{3}} M^{\frac{2}{3}}\right).$$ 
This completes the proof of Theorem \ref{thm: low bound} for the case $\beta>0$.

\subsection{Case \texorpdfstring{$\beta<0$}{Lg}}
The proof of the base $\beta<0$ is similar to that of the case $\beta>0$. For $\beta<0$, consider  a stochastic $k$-arm and $M$ horizons bandit environment $\nu$, where the reward for pulling arm $j\in\{1,2,\ldots,k\}$ is given by the scaled Bernoulli random variable $Ber(1-p_j)$
\begin{align*}
    X_j=\begin{cases}
    0, &\text{with probability } p_j,\\
    H, &\text{with probability } 1-p_j
    \end{cases}
\end{align*}
where $H\geq 1$ specifies the range of the reward. We let the arm $i$ be the unique optimal arm and all the other $k-1$ arms have the same $p_j$, that is, $p_1=p_2=\cdots=p_{i-1}=p_{i+1}=\cdots=p_k=p$ and  $p_i=p+\Delta$ for some constants $p>0$ and $\Delta<0$.
Define $X_j^m$ to be the outcome of arm $j$ (if pulled) in round $m$, and $Y^m$ to be the outcome of arm actually pulled in round $m$.

\begin{lemma} \label{lemma: Regret decomposition with entropic risk measure beta < 0}
For the Bernoulli bandit $\nu$ described above, if $p=e^{\beta H}$ and $\Delta\geq -e^{\beta H}$, then for every policy $\pi$, the regret with the entropic risk measure in $\nu$ satisfies
\begin{align*}
\operatorname{Regret}(M)\coloneqq& \sum_{m=1}^M \frac{1}{\beta} \left(\log \left[\EE[ \exp\left(\beta X_1^m \right) ] \right]- \log \left[\EE[ \exp\left(\beta Y^m \right) ] \right] \right)\\
\geq & \sum_{j\in[k]/\ i}  { \EE\left[T_j(M)\right]}  \frac{\Delta (e^{-\beta H}-1)}{2 \beta}
\end{align*}
\end{lemma}
\begin{proof}
Taking the expectation over $M$ on both sides of \eqref{eq: lower bound D regret before expectation}, we have
\begin{align*}
    \EE\left[ \operatorname{Regret(M)}\right]=& \sum_{j\in[k]/\ i}  \frac{ \EE\left[T_i(M)\right]}{\beta} \left(\log \left[\EE[ \exp\left(\beta X_i \right) ] \right]- \log \left[\EE[ \exp\left(\beta X_j \right) ] \right] \right)\\
    =& \sum_{j\in[k]/\ i}  \frac{ \EE\left[T_j(M)\right]}{\beta} \log \left(  \frac{(1-p-\Delta) e^{\beta H}+(p+\Delta)}{(1-p) e^{\beta H}+p} \right) \\
     =& \sum_{j\in[k]/\ i}  \frac{ \EE\left[T_j(M)\right]}{\beta} \log \left( 1+ \frac{\Delta (1-e^{\beta H})}{(1-p) e^{\beta H}+p} \right) \\
    \geq& \sum_{j\in[k]/\ i}  \frac{ \EE\left[T_j(M)\right]}{\beta} \log \left( 1+ \frac{\Delta (1-e^{\beta H})  }{2e^{\beta H}} \right) \\
     \geq & \sum_{j\in[k]/\ i}  { \EE\left[T_j(M)\right]}  \frac{\Delta (e^{-\beta H}-1)}{2 \beta}
\end{align*}
where the first inequality holds since $p=e^{\beta H}$, the second inequality holds since $\Delta\leq e^{-\beta H}$ and $\log(1+x)\leq x$ for $x >-1$.
\end{proof}

\begin{lemma}
Let $k>1$. For every policy $\pi$ and sufficiently large $M$ and $H$, there exists a $k$-arm bandit instance such that
\begin{align*}
 \EE_{\Vec{p}}\left[ \operatorname{Regret(M)}\right]>& \frac{e^{-\beta H/2}-1}{-\beta}\frac{\sqrt{Mk}}{64e}.
\end{align*}
\end{lemma}

\begin{proof}
The proof is similar to that of Lemma \ref{lemma: low bound bandit} by replacing Lemma \ref{lemma: Regret decomposition with entropic risk measure} with Lemma \ref{lemma: Regret decomposition with entropic risk measure beta < 0}, replacing \eqref{eq: KL Pv} by
\begin{align}
\nonumber {D}_{\mathrm{KL}}\left(\mathbb{P}_{\nu}\mid \mathbb{P}_{\nu^{\prime}}\right)=& \mathbb{E}_{\mathbb{P}_{\Vec{p}}}\left[T_{i}(M)\right] \mathrm{D}_{\mathrm{KL}}\left(\text{Ber}(1-p_i) \mid \text{Ber}(1-p_{i}^{\prime})\right)
\end{align}
and by choosing $\Delta=-\sqrt{k/(16M e^{-\beta H})} \geq -e^{\beta H}$.
\end{proof}

The rest of the proof is similar to that for the case $\beta>0$ and is thus omitted.

\section{Auxiliary lemmas}
\begin{lemma} \label{lemma: decompose of d regret}
For $\beta>0$, the dynamic regret is bounded by
$$
\operatorname{D-Regret}(M) \leq  \sum_{m \in[M]} \frac{1}{\beta}\left[e^{\beta \cdot V_{1}^{*}\left(s_{1}^{m}\right)}-e^{\beta \cdot V_{1}^{m}\left(s_{1}^{m}\right)}\right] + \sum_{m \in[M]} \frac{1}{\beta}\left[e^{\beta \cdot V_{1}^{m}\left(s_{1}^{m}\right)}-e^{\beta \cdot V_{1}^{\pi^{m},m}\left(s_{1}^{m}\right)}\right],
$$
and for $\beta<0$, the dynamic regret is bounded by
$$
\operatorname{D-Regret}(M) \leq \sum_{m \in[M]} \frac{e^{-\beta H}}{(-\beta)}\left[e^{\beta \cdot V_{1}^{{m}}\left(s_{1}^{m}\right)}-e^{\beta \cdot V_{1}^{*,m}\left(s_{1}^{m}\right)}\right]+\sum_{m \in[M]} \frac{e^{-\beta H}}{(-\beta)}\left[e^{\beta \cdot V_{1}^{\pi^{m},m}\left(s_{1}^{m}\right)}-e^{\beta \cdot V_{1}^{m}\left(s_{1}^{m}\right)}\right] .
$$
\end{lemma}

\begin{proof}
For $\beta>0$, we have
$$
\begin{aligned}
&\operatorname{D-Regret}(M)\\
&=\sum_{m \in[M]}\left(V_{1}^{*,m}-V_{1}^{\pi^{m},m}\right)\left(s_{1}^{m}\right) \\
& = \sum_{m \in[M]}\left(V_{1}^{*,m}-V_{1}^{{m}}\right)\left(s_{1}^{m}\right)+ \sum_{m \in[M]}\left(V_{1}^{m}-V_{1}^{\pi^{m}}\right)\left(s_{1}^{m}\right) \\
&=\sum_{m \in[M]}\left[\frac{1}{\beta} \log \left\{e^{\beta \cdot V_{1}^{*,m}\left(s_{1}^{m}\right)}\right\}-\frac{1}{\beta} \log \left\{e^{\beta \cdot V_{1}^{m}\left(s_{1}^{m}\right)}\right\}\right]+\sum_{m \in[M]}\left[\frac{1}{\beta} \log \left\{e^{\beta \cdot V_{1}^{m}\left(s_{1}^{m}\right)}\right\}-\frac{1}{\beta} \log \left\{e^{\beta \cdot V_{1}^{\pi^{m}}\left(s_{1}^{m}\right)}\right\}\right] \\
& \leq \sum_{m \in[M]} \frac{1}{\beta}\left[e^{\beta \cdot V_{1}^{*,m}\left(s_{1}^{m}\right)}-e^{\beta \cdot V_{1}^{m}\left(s_{1}^{m}\right)}\right] + \sum_{m \in[M]} \frac{1}{\beta}\left[e^{\beta \cdot V_{1}^{m}\left(s_{1}^{m}\right)}-e^{\beta \cdot V_{1}^{\pi^{m},m}\left(s_{1}^{m}\right)}\right] 
\end{aligned}
$$
where the last step holds by the 1-Lipschitzness of the function $f(x)=\log x$ for $x \geq 1$.

For $\beta<0$, we similarly have
$$
\begin{aligned}
&\operatorname{D-Regret}(M)\\
&=\sum_{m \in[M]}\left(V_{1}^{*,m}-V_{1}^{\pi^{m},m}\right)\left(s_{1}^{m}\right) \\
& = \sum_{m \in[M]}\left(V_{1}^{*,m}-V_{1}^{{m}}\right)\left(s_{1}^{m}\right)+ \sum_{m \in[M]}\left(V_{1}^{m}-V_{1}^{\pi^{m}}\right)\left(s_{1}^{m}\right) \\
&=\sum_{m \in[M]}\left[\frac{1}{-\beta} \log \left\{e^{\beta \cdot V_{1}^{m}\left(s_{1}^{m}\right)}\right\}-\frac{1}{-\beta} \log \left\{e^{\beta \cdot V_{1}^{*,m}\left(s_{1}^{m}\right)}\right\}\right]\\
&\hspace{0.6cm} +\sum_{m \in[M]}\left[\frac{1}{-\beta} \log \left\{e^{\beta \cdot V_{1}^{\pi^m, m}\left(s_{1}^{m}\right)}\right\}-\frac{1}{-\beta} \log \left\{e^{\beta \cdot V_{1}^{{m}}\left(s_{1}^{m}\right)}\right\}\right] \\
&\leq \sum_{m \in[M]} \frac{e^{-\beta H}}{(-\beta)}\left[e^{\beta \cdot V_{1}^{{m}}\left(s_{1}^{m}\right)}-e^{\beta \cdot V_{1}^{*,m}\left(s_{1}^{m}\right)}\right]+\sum_{m \in[M]} \frac{e^{-\beta H}}{(-\beta)}\left[e^{\beta \cdot V_{1}^{\pi^{m},m}\left(s_{1}^{m}\right)}-e^{\beta \cdot V_{1}^{m}\left(s_{1}^{m}\right)}\right] 
\end{aligned}
$$
where the last step holds by the $\left(e^{-\beta H}\right)$-Lipschitzness of the function $f(x)=\log x$ for $x \geq e^{\beta H}$.
\end{proof}

\begin{lemma}[Theorem 1 in \cite{abbasi2011improved}] \label{lemma: Concentration of Self-normalized Processes}
Let $\left\{\mathcal{F}_{t}\right\}_{t=0}^{\infty}$ be a filtration and $\left\{\eta_{t}\right\}_{t=1}^{\infty}$ be a $\mathbb{R}$-valued stochastic process such that $\eta_{t}$ is $\mathcal{F}_{t}$-measurable for every $t \geq 0$. Assume that for every $t \geq 0$, conditioning on $\mathcal{F}_{t}, \eta_{t}$ is a zero-mean and $\sigma$-subGaussian random variable with the variance proxy $\sigma^{2}>0$, i.e., $\mathbb{E}\left[e^{\lambda \eta_{t}} \mid \mathcal{F}_{t}\right] \leq e^{\lambda^{2} \sigma^{2} / 2}$ for every $\lambda \in \mathbb{R}$. Let $\left\{X_{t}\right\}_{t=1}^{\infty}$ be an $\mathbb{R}^{d}$-valued stochastic process such that $X_{t}$ is $\mathcal{F}_{t}$-measurable for every $t \geq 0$. Let $Y \in \mathbb{R}^{d \times d}$ be a deterministic and positive-definite matrix. For every $t \geq 0$, we define
$$
\bar{Y}_{t}:=Y+\sum_{\tau=1}^{t} X_{\tau} X_{\tau}^{\top} \text { and } S_{t}=\sum_{\tau=1}^{t} \eta_{\tau} X_{\tau} .
$$
Then, for every fixed $\delta \in(0,1)$, it holds with probability at least $1-\delta$ that
$$
\left\|S_{t}\right\|_{\left(\bar{Y}_{t}\right)^{-1}}^{2} \leq 2 \sigma^{2} \log \left(\frac{\operatorname{det}\left(\bar{Y}_{t}\right)^{1 / 2} \operatorname{det}(Y)^{-1 / 2}}{\delta}\right)
$$
for every $t \geq 0$.
\end{lemma}

\begin{lemma}[Fact 1 in \cite{fei2021risk}]\label{lemma: alpha t property}
The following properties hold for $\alpha_{t}^{i}$ defined in \eqref{eq: definition of alpha t i}:
\begin{enumerate}
    \item $\frac{1}{\sqrt{t}} \leq \sum_{i \in[t]} \frac{\alpha_{t}^{i}}{\sqrt{i}} \leq \frac{2}{\sqrt{t}}$ for every integer $t \geq 1$.
    \item $\max _{i \in[t]} \alpha_{t}^{i} \leq \frac{2 H}{t}$ and $\sum_{i \in[t]}\left(\alpha_{t}^{i}\right)^{2} \leq \frac{2 H}{t}$ for every integer $t \geq 1$.
    \item $\sum_{t=i}^{\infty} \alpha_{t}^{i}=1+\frac{1}{H}$ for every integer $i \geq 1$.
    \item $\sum_{i \in[t]} \alpha_{t}^{i}=1$ and $\alpha_{t}^{0}=0$ for every integer $t \geq 1$, and $\sum_{i \in[t]} \alpha_{t}^{i}=0$ and $\alpha_{t}^{0}=1$ for $t=0$.
\end{enumerate}
\end{lemma}


\begin{lemma}[Lemma 14.2 in \cite{lattimore2020bandit}] \label{lemma: Bretagnolle-Huber}
Let $P, Q$ be probability measures on the same measurable space $(\Omega, \mathcal{F})$, and let $A\in\mathcal{F}$ be an arbitrary event. Then,
\begin{align*}
    P(A)+Q(A^c)\geq \frac{1}{2} \exp \left(-D_{\text{KL}}(P \mid Q)\right),
\end{align*}
where $D_{\text{KL}}$ denotes the KL divergence and $A^c=\Omega /\ A$ is the complement of $A$. 
\end{lemma}

\begin{lemma}[Lemma 14 in \cite{fei2020risk}] \label{lemma: KL of ber upper bound}
Let $p, p^{\prime} \in(0,1)$ be such that $p>p^{\prime} .$ We have $D_{\mathrm{KL}}\left(\operatorname{Ber}\left(p^{\prime}\right) \| \operatorname{Ber}(p)\right) \leq \frac{\left(p-p^{\prime}\right)^{2}}{p(1-p)}$
\end{lemma}

\begin{lemma}[Divergence decomposition, Lemma 15.1 in \cite{lattimore2020bandit}] \label{lemma: Divergence decomposition} 
Let $\nu=\left(P_{1}, \ldots, P_{k}\right)$ be the reward distributions associated with one $k$-armed bandit, and let $\nu^{\prime}=\left(P_{1}^{\prime}, \ldots, P_{k}^{\prime}\right)$ be the reward distributions associated with another $k$-armed bandit. Fix some policy $\pi$ and let $\mathbb{P}_{\nu}=\mathbb{P}_{\nu \pi}$ and $\mathbb{P}_{\nu^{\prime}}=\mathbb{P}_{\nu^{\prime} \pi}$ be the probability measures on the canonical bandit model (Section 4.6 in \cite{lattimore2020bandit}) induced by the $M$-round interconnection of $\pi$ and $\nu$ (respectively, $\pi$ and $\nu^{\prime}$ ). Then,
$$
\mathrm{D}_{\mathrm{KL}}\left(\mathbb{P}_{\nu}, \mathbb{P}_{\nu^{\prime}}\right)=\sum_{i=1}^{k} \mathbb{E}_{\nu}\left[T_{i}(M)\right] \mathrm{D}_{\mathrm{KL}}\left(P_{i}, P_{i}^{\prime}\right)
$$
\end{lemma}

\end{document}